\documentclass{article} 
\usepackage[final]{colm2026_conference}

\usepackage{microtype}
\microtypesetup{protrusion=false}
\usepackage{hyperref}
\usepackage{url}
\usepackage{booktabs}

\usepackage{lineno}

\definecolor{darkblue}{rgb}{0, 0, 0.5}
\hypersetup{colorlinks=true, citecolor=darkblue, linkcolor=darkblue, urlcolor=darkblue}

\usepackage{wrapfig}
\usepackage{caption}
\usepackage{hyperref}

\usepackage{amsmath,amssymb,mathtools}
\usepackage{amsthm}

\usepackage{graphicx}
\usepackage[dvipsnames]{xcolor}
\usepackage{tikz}

\usepackage{hyperref}
\usepackage{url}
\usepackage{booktabs}
\usepackage{multirow}
\usepackage{array}

\usepackage[ruled,vlined,linesnumbered,noend]{algorithm2e}
\SetKwComment{Comment}{$\triangleright$ }{}
\SetKwInput{KwInput}{Input}
\SetKwInput{KwOutput}{Output}
\SetKw{Proj}{Proj}
\SetKw{Soft}{SoftThresh}
\SetKw{ReLU}{ReLU}

\usepackage{diagbox}
\usepackage{graphicx}
\usepackage[normalem]{ulem}
\usepackage{enumitem}
\usepackage{placeins}
\usepackage{caption}
\usepackage{adjustbox}
\usepackage{subcaption}

\usepackage{booktabs,tabularx,makecell,threeparttable}
\usepackage{multirow}
\usepackage[table]{xcolor} 

\usepackage{amsfonts}
\usepackage{url}
\usepackage{hyperref}

\usepackage{caption}
\usepackage{subcaption}
\usepackage{float}
\usepackage{placeins}
\usepackage{fontawesome5} 

\usepackage{amsmath,amsfonts,bm}

\def\eqref#1{equation~\ref{#1}}

\def\1{\bm{1}}

\DeclareMathAlphabet{\mathsfit}{\encodingdefault}{\sfdefault}{m}{sl}
\SetMathAlphabet{\mathsfit}{bold}{\encodingdefault}{\sfdefault}{bx}{n}

\newcommand{\R}{\mathbb{R}}

\definecolor{todored}{HTML}{E63946}
\definecolor{darkred}{RGB}{150,0,0}
\definecolor{darkgreen}{RGB}{0,150,0}
\definecolor{darkblue}{RGB}{0,0,200}

\hypersetup{
    colorlinks=true,
    linkcolor=darkred,
    citecolor=darkgreen,
    urlcolor=darkblue
}

\newcolumntype{L}{>{\raggedright\arraybackslash}X}

\newtheorem{theorem}{Theorem}[section]
\newtheorem{proposition}[theorem]{Proposition}

\title{Covariance-Aware Transformers for Quadratic Programming and Decision Making}

\author{
Kutay Tire\textsuperscript{*,1} \qquad
Yufan Zhang\textsuperscript{*,2} \qquad
Ege Onur Taga\textsuperscript{2} \qquad
Samet Oymak\textsuperscript{2} \\[12pt]
\textsuperscript{1}The University of Texas at Austin \\
\texttt{kutaytire@utexas.edu} \\[4pt]
\textsuperscript{2}University of Michigan, Ann Arbor \\
\texttt{\{yufanzh,egetaga,oymak\}@umich.edu} \\[3pt]
{\small \textsuperscript{*}Equal contribution}
}

\date{} 

\begin{document}

\ifcolmsubmission
\linenumbers
\fi

\maketitle

\begin{abstract}
We explore the use of transformers for solving quadratic programs and how this capability benefits decision-making problems that involve covariance matrices. We first show that the linear attention mechanism can provably solve unconstrained QPs by tokenizing the matrix variables (e.g.~$A$ of the objective $\frac{1}{2}x^\top Ax+b^\top x$) row-by-row and emulating gradient descent iterations. Furthermore, by incorporating MLPs, a transformer block can solve (i) $\ell_1$-penalized QPs by emulating iterative soft-thresholding and (ii) $\ell_1$-constrained QPs when equipped with an additional feedback loop. Our theory motivates us to introduce \emph{Time2Decide}: a generic method that enhances a time series foundation model (TSFM) by explicitly feeding the covariance matrix between the variates. We empirically find that \emph{Time2Decide} uniformly outperforms the base TSFM model for the classical portfolio optimization problem that admits an $\ell_1$-constrained QP formulation. Remarkably, \emph{Time2Decide} also outperforms the classical ``Predict-then-Optimize (PtO)'' procedure, where we first forecast the returns and then explicitly solve a constrained QP, in suitable settings. Our results demonstrate that transformers benefit from explicit use of second-order statistics, and this can enable them to effectively solve complex decision-making problems, like portfolio construction, in one forward pass.
\end{abstract}

\section{Introduction}
\label{sec:intro}

Quadratic programming (QP) is a central tool in modern optimization, underlying applications in control, machine learning, and finance. Its canonical form is
\[
\min_x \; \frac{1}{2}x^\top A x + b^\top x \quad \text{subject to}\quad Cx \preceq d,
\]
where $A,b$ define the objective and $Cx\preceq d$ encodes constraints. Beyond this template, many practically important variants, including $\ell_1$-regularized and $\ell_1$-constrained formulations, remain convex and capture structural requirements such as sparsity, budgets, and transaction-cost constraints.

We study the following concrete question: \emph{To what extent can transformer architectures act as general-purpose QP solvers, and how does this capability translate into improved decision making when second-order statistics (e.g., covariance matrices) matter?} This focus is motivated by two observations: First, there is growing evidence that transformers can emulate iterative algorithms and optimization procedures through their attention-and-residual structure (e.g., theories of in-context learning showing gradient-descent-like behavior). Second, many high-stakes decision problems are naturally expressed as QPs whose inputs include covariances. However, standard transformer pipelines, specifically time-series foundation models, do not explicitly incorporate such second-order information, even when it is essential for the downstream task.

\begin{figure*}[t]
  \centering
  \includegraphics[width=0.98\textwidth]{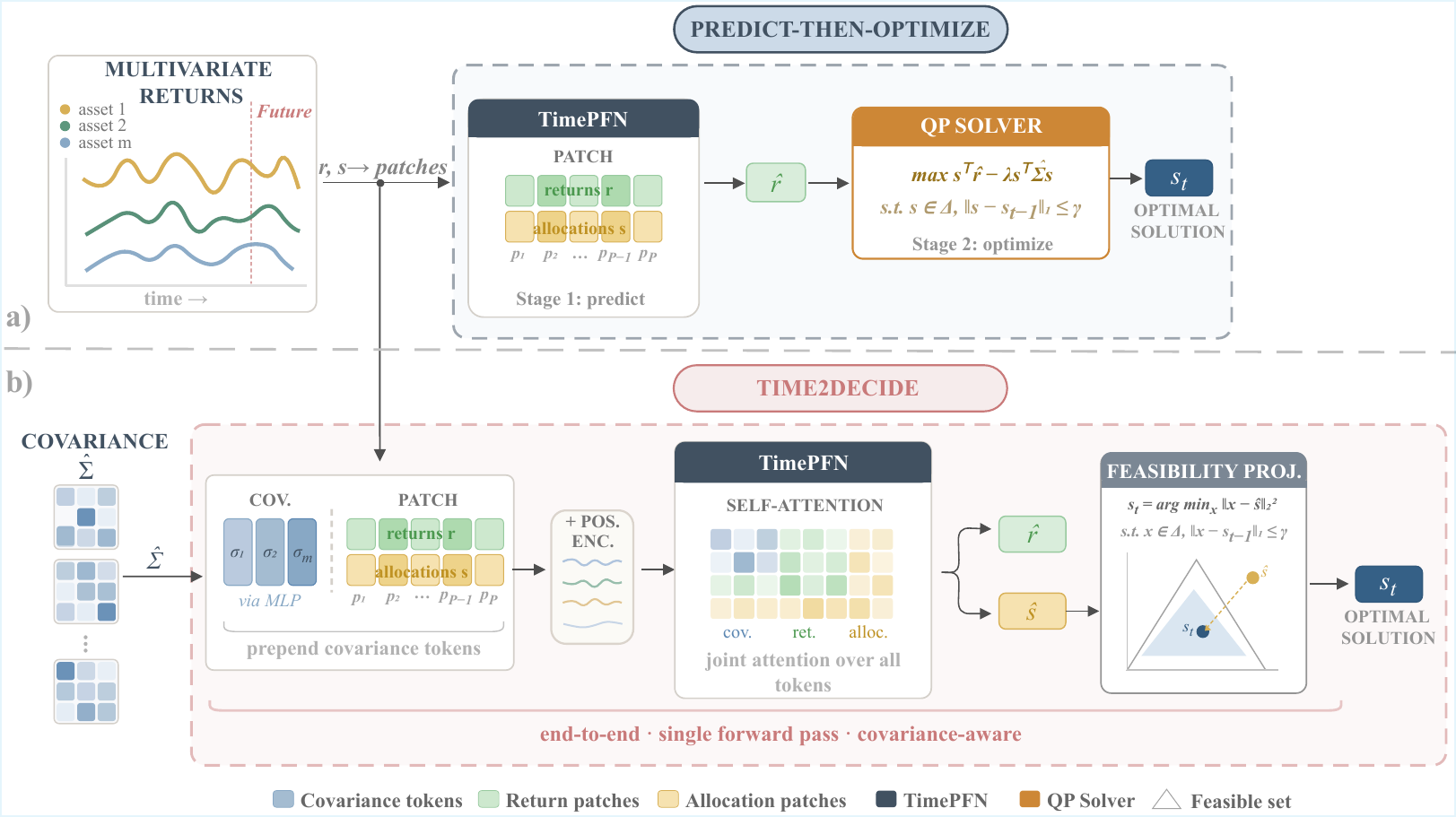}
    \caption{\textbf{Overview of the portfolio decision-making pipelines at decision time $t$.} 
        \textbf{(a) Predict-then-Optimize (PtO):}   historical multivariate returns are tokenized into patches and fed to
        TimePFN to predict next-step returns~$\hat{r}$. A QP
        solver then takes these predictions to compute the
        allocation~$s_t$ under simplex and $\ell_1$ rebalancing constraints
        \textbf{(b) Time2Decide:} unifies both stages in a single
        forward pass. We first derive learned \emph{covariance tokens} $(\sigma_1,\ldots,\sigma_m)$ from
        $\hat{\Sigma}$ by applying an MLP and prepend them to the return and
        allocation patch tokens. TimePFN with joint self-attention processes this full
        sequence to directly output both~$\hat{r}$ and~$\hat{s}$. The
        final feasibility projection enforces the portfolio constraints to
        yield~$s_t$, enabling an end-to-end, covariance-aware decision
        without a separate optimization step.}
    \label{fig:pipeline}
\end{figure*}

A standard example is portfolio optimization. Classical mean–variance allocation solves a QP whose parameters depend on forecasted returns and an (estimated) covariance matrix. A fundamental and widely used baseline is the \emph{two-stage} (or \emph{predict-then-optimize}) pipeline: first predict the unknown parameters (e.g., returns), then solve the resulting QP with a conventional optimizer \citep[e.g.,][]{markowitz1952portfolio,elmachtoub2022smart}. While well established, this decoupling can be statistically suboptimal: when forecasts are noisy, the plug-in solution can overreact and yield inferior decisions. Indeed, one can formalize an explicit scenario where the predict-then-optimize (PtO) rule is \emph{strictly dominated} by a Bayes-optimal end-to-end policy that shrinks noisy estimates before optimizing (Prop.~\ref{prop:pto_suboptimal}). 

As our central contribution, we develop a comprehensive theory on how and when transformers can solve regularized QPs. Our constructions are highly compute-efficient and capture the individual contributions of the attention mechanism, MLP, and looping. Building on this theory, we introduce \emph{Time2Decide}, a generic method to endow time-series foundation models (TSFMs) with transformer-based optimization that is \emph{covariance-aware}. Concretely, Time2Decide augments the TSFM input with tokens that explicitly represent the covariance structure, enabling the model to learn an end-to-end mapping from historical context to decisions in one forward pass. The overview of our setup is provided in Figure~\ref{fig:pipeline}. This directly targets the limitation highlighted above: by training to predict the eventual decision variable, the model can learn the proper amount of attenuation/calibration in the presence of estimation noise, rather than committing to a brittle plug-in rule. Empirically, Time2Decide consistently improves over (i) the base TSFM and (ii) supervised fine-tuning of the TSFM \emph{without} covariance information, and it can even outperform the classical two-stage PtO procedure that forecasts returns and then solves the QP explicitly. In particular, while PtO is highly competitive, Time2Decide surpasses it in regimes where end-to-end calibration matters most, aligning with the strict separation exhibited by Proposition~\ref{prop:pto_suboptimal} and illustrating how explicitly admitting second-order statistics can enable transformers to solve two-step decision problems \emph{within a single learned model}. Our contributions are:
\begin{itemize}\vspace{-4pt}
    \item \textbf{Quadratic programming theory.} We provide explicit constructions showing how transformer components can emulate classical first-order methods for multiple QP classes, including unconstrained, linearly constrained, $\ell_1$-regularized, and $\ell_1$-constrained programs.\vspace{-3pt}
    \item \textbf{Experimental validation.} We demonstrate that transformers can indeed learn the mapping from tokenized QP instances to optimal solutions, with linear attention variants offering strong empirical performance.\vspace{-3pt}
    \item \textbf{Our Method: Time2Decide.} We integrate QP-solving capabilities into TSFMs via covariance-augmented tokenization, enabling end-to-end forecasting and decision-making. Focusing on portfolio optimization, our proposed method Time2Decide outperforms fine-tuning a TSFM thanks to its special covariance tokens. Importantly, it also outperforms the two-stage Predict-then-Optimize (PtO) baseline under realistic scenarios. Specifically, when (i) the returns are not too predictable or (ii) the portfolio transition constraints are not too tight. We complement these through theoretical insights (Proposition~\ref{prop:pto_suboptimal}) where end-to-end decision-making can \emph{provably} and \emph{strictly} dominate two-stage optimization.
\end{itemize}\vspace{-4pt}

Together, our results position transformers as both (i) theoretically grounded algorithm emulators for regularized QPs and (ii) practical end-to-end decision makers when second-order/covariance information is essential.

\section{Related Work}
\label{sec:rel_work}

\paragraph*{Theoretical Analysis of In-Context Learning} 
A growing body of work shows that transformers can emulate gradient descent during in-context learning \citep{akyrek2023what,von2023transformers,dai2023gpt}, implement preconditioned variants \citep{ahn2024transformers}, and learn linear models in-context \citep{zhang2023trained,zhang2024trained,mahankali2024one}, with several studies analyzing generalization \citep{li2023transformers,yang2024context,li2024fine,bai2023transformers,gozeten2025test,shen2024training}. However, these results focus on supervised settings with input–label pairs. In contrast, we show that the same attention-and-residual structure can emulate first-order optimization methods when tokens encode problem data rather than labeled examples.

\paragraph{Neural Networks as Optimizers.} Prior work has used neural networks to assist or replace classical solvers: MLPs have been used to approximate nonlinear objectives \citep{villarrubia2018artificial}, differentiable QP layers enable end-to-end gradient computation \citep{amos2017optnet,magoon2024differentiation}, neural heuristics accelerate mixed-integer programming \citep{nair2020solving}, and graph neural networks predict LP feasibility and solutions \citep{chen2022representing}. Our work differs by feeding QP instances directly into transformers as token sequences and proving that, with appropriate tokenization, self-attention can emulate iterative first-order solvers handling linear and $\ell_1$ constraints.

\paragraph{Multivariate Time-Series Transformers.} Early neural approaches to multivariate forecasting relied on encoder–decoder RNNs and temporal convolutional networks; however, their receptive fields or recurrence limited long-range context.  Transformer variants now dominate the landscape. Temporal Fusion Transformer (TFT) augments attention mechanism with gating and interpretable variable selection \citep{Lim2021TFT}. Informer introduces ProbSparse self-attention for efficient long-sequence inference \citep{Zhou2021Informer}, while Autoformer and FEDformer decompose series into trend and seasonal tokens to improve long-horizon extrapolation \citep{Wu2021Autoformer,Zhou2022FEDformer}.  Crossformer \citep{ZhangYan2023Crossformer} and Spacetimeformer \citep{Dai2021Spacetimeformer} rearrange inputs to model cross-dimension dependencies explicitly, and PatchTST \citep{Nie2023PatchTST} treats fixed-length segments as tokens to boost locality.  Most recently, TimePFN uses permutation-equivariant networks pre-trained on synthetic data to provide strong zero-shot generalization \citep{Taga2025TimePFN}. These advances demonstrate the transformer's flexibility in processing diverse token structures and incorporating auxiliary information, setting the foundation for integrating optimization-solving capabilities alongside forecasting tasks in multivariate time-series models.

\section{Theoretical Results}
\label{sec:all-theory}

\subsection{Problem Setup}
\label{sec:setup}

\paragraph{Notation.}
Let $[p]=\{1,\dots,p\}$ for an integer $p\ge 1$.
Lowercase letters (e.g., $x,y$) denote vectors; uppercase letters (e.g., $A,C$) denote matrices.
For $A\in\mathbb{R}^{n\times n}$, let $\lambda_{\max}(A)$ be its largest eigenvalue and $\|A\|_2$ its operator norm.
We write $\langle u,v\rangle=u^\top v$, use $\|\cdot\|_2$ for the Euclidean norm, and $\|\cdot\|_1$ for the $\ell_1$ norm.
For $u\in\mathbb{R}$, set $[u]_+=\max\{u,0\}$ (applied elementwise to vectors).
The soft-thresholding map is $\mathcal{S}_\theta(y)=\operatorname{sign}(y)\odot(|y|-\theta)_+$,
and $\operatorname{Proj}_{\{\|\cdot\|_1\le B\}}$ denotes Euclidean projection onto the $\ell_1$-ball.
We use $\operatorname{prox}_{\psi}(y)=\arg\min_{x}\{\tfrac12\|x-y\|_2^2+\psi(x)\}$ and write $L=\lambda_{\max}(A)$.


\paragraph{Quadratic programs.}
We study four convex QPs, each specified by problem data and a stepsize pair $(\gamma,\eta)$:
\newcommand{\qplabel}[1]{\makebox[3.5em][r]{\text{#1}\quad}}
\begin{equation}
\qplabel{(U)}
\min_{x\in\mathbb{R}^n}\ \tfrac12\,x^\top A x + b^\top x,
\label{eq:prob-U}
\end{equation}
\begin{equation}
\qplabel{(LC)}
\min_{x\in\mathbb{R}^n}\ \tfrac12\,x^\top A x + b^\top x
\ \ \text{s.t.}\ \ Cx\preceq d,
\label{eq:prob-LC}
\end{equation}
\begin{equation}
\qplabel{(R)}
\min_{x\in\mathbb{R}^n}\ \tfrac12\,x^\top A x + b^\top x + \lambda\|x\|_1,\qquad \lambda>0,
\label{eq:prob-R}
\end{equation}
\begin{equation}
\qplabel{(C)}
\min_{x\in\mathbb{R}^n}\ \tfrac12\,x^\top A x + b^\top x
\ \ \text{s.t.}\ \ \|x\|_1\le B,\quad B>0.
\label{eq:prob-C}
\end{equation}

\paragraph{Algorithms to be emulated.}
Let $x_k$ (and, when present, $\lambda_k\succeq 0$) denote the current iterate(s), and write $L=\lambda_{\max}(A)$. We have the following algorithms:

\begin{description}
    \item[GD (U)] Unconstrained Gradient Descent:
    \begin{equation}
        x_{k+1}=x_k-\gamma(Ax_k+b), \qquad 0<\gamma<\frac{2}{L}.
        \label{eq:map-U}
    \end{equation}

    \item[Arrow--Hurwicz (LC)] For problems with linear constraints:
    \begin{equation}
        \begin{aligned}
            x_{k+1}&=x_k-\gamma(Ax_k+b+C^\top\lambda_k),\\
            \lambda_{k+1}&=\big[\lambda_k+\eta(Cx_{k+1}-d)\big]_+,
        \end{aligned}
        \qquad
        \begin{gathered} 
            0<\gamma<\frac{2}{L}, \\
            \eta\gamma\|C\|_2^2<1.
        \end{gathered}
        \label{eq:map-LC}
    \end{equation}

    \item[ISTA (R)] Iterative Shrinkage-Thresholding Algorithm for regularization:
    \begin{equation}
        \begin{aligned} 
            y_k &= x_k-\gamma(Ax_k+b), \\
            x_{k+1} &= \mathcal{S}_{\gamma\lambda}(y_k)=\operatorname{prox}_{\gamma\lambda\|\cdot\|_1}(y_k),
        \end{aligned}
        \qquad 0<\gamma\le \frac{1}{L}.
        \label{eq:map-R}
    \end{equation}

    \item[PGD with $\ell_1$-projection (C)] Projected Gradient Descent onto the $\ell_1$-ball:
    \begin{equation}
        \begin{aligned} 
            y_k &= x_k-\gamma(Ax_k+b), \\
            x_{k+1} &= \operatorname{Proj}_{\{\|x\|_1\le B\}}(y_k),
        \end{aligned}
        \qquad 0<\gamma\le \frac{1}{L}.
        \label{eq:map-C}
    \end{equation}
\end{description}

We present explicit weight constructions showing how small transformer fragments implement one iteration of classical first-order methods for QPs. Tokens encode the current iterate \(x_k\), the multiplier \(\lambda_k\) when needed, and rows of problem data (e.g., \(A\), \(C\)). A residual on the \(x\)-row propagates the iterate, and repeating these one-step maps yields the standard convergence behavior under the usual step size conditions of the target algorithms. The summary of the required architectures for each algorithm is presented in Table \ref{tab:algo_summary_text_final}.

\begin{table*}[htbp]
\centering
\caption{Summary of transformer-based optimization algorithms.}
\label{tab:algo_summary_text_final}
\scriptsize
\setlength{\tabcolsep}{3pt}%
\renewcommand{\arraystretch}{1.10}%
\resizebox{0.95\linewidth}{!}{%
\begin{tabular}{@{} l l l @{}}
\toprule
\multicolumn{1}{c}{\textbf{Optimization Program}} &
\multicolumn{1}{c}{\textbf{Architecture}} &
\multicolumn{1}{c}{\textbf{Emulated Algorithm}} \\
\midrule
\rowcolors{2}{black!6}{white}%
Unconstrained QP & 1 Linear Attention Block & Gradient Descent \\
Linearly Constrained QP & 2 Sequential Attention Blocks & Arrow-Hurwicz Primal-Dual Method \\
$\ell_1$-Regularized QP & 1 Linear Attention Block + FFN & ISTA (Proximal Gradient Descent) \\
$\ell_1$-Constrained QP & 1 Linear Attention Block + FFN with Threshold Loop & Projected Gradient Descent \\
\bottomrule
\end{tabular}%
}
\end{table*}

\label{sec:theory}
\subsection{Unconstrained QP via Single-Block Linear Attention}
\label{sec:U-main}

We show that a single linear-attention head with a residual connection simulates one step of gradient descent for the unconstrained QP defined in Section \ref{sec:setup}.

\smallskip

\begin{proposition}
\label{prop:unconstrained}
A single linear-attention head that attends from the \(x\)-row to the \(A\)-rows, followed by a residual on the \(x\)-row, realizes the update \(x_{k+1}=x_k-\gamma(Ax_k+b)\). Hence, for any \(0<\gamma<2/L\), the iterates converge linearly to \(x^\star=-A^{-1}b\).
\end{proposition}

\begin{proof}[Construction]
We store tokens for the rows of \(A\), the vector \(b\), and the current iterate \(x_k\), and we design fixed query, key, and value maps \(W_Q,W_K,W_V\) (see Appendix~\ref{app:U}). Then the resulting \(q,k,v\) yield, when evaluated at the \(x\)-token is $o = A x_k + b$. Applying a residual with post-map \(-\gamma I\) on the \(x\)-row yields
\[
x_{k+1}=x_k-\gamma\,o=x_k-\gamma(Ax_k+b).
\]
Full token layout and selector matrices are given in the Appendix~\ref{app:U}. Furthermore, we empirically validate this construction in Figure~\ref{fig:convU} and discuss the algorithmic complexity in Appendix \ref{app:runtime:unconstrained}.
\end{proof}

\subsection{Linearly Constrained QP via Single Macro-Block Linear Attention}
\label{sec:LCQP-main}

We now turn to the linearly constrained QP, the second formulation presented in Section \ref{sec:setup}.

\begin{proposition}
\label{prop:LC}
A single macro-block consisting of two self-attention blocks and a residual on the \(x\)-row realizes the sequential update
\begin{align}
x_{k+1}=x_k-\gamma\big(Ax_k+b+C^\top\lambda_k\big), \\
\lambda_{k+1}=\big[\lambda_k+\eta\,(C x_{k+1}-d)\big]_+ .
\end{align}
Let \(L=\lambda_{\max}(A)\) and let $\|\cdot\|_2$ denote the operator norm. If the feasible set \(\{x:\,Cx\preceq d\}\) is nonempty and the step sizes satisfy \(0<\gamma<2/L\) and \(\eta\gamma\|C\|_2^2<1\), then the iterates \((x_k,\lambda_k)\) converge to a KKT point. 
\end{proposition}

\begin{proof}[Construction]
We fix \(W_Q,W_K,W_V\) so that two heads in the first linear-attention block, evaluated at the \(x\)- and \(\lambda\)-tokens, produce
\[
o^{Ax{+}b}=Ax_k+b, \qquad o^{C^\top\lambda}=C^\top\lambda_k,
\]
which yields the primal update
\[
x_{k+1}
= x_k-\gamma\big(o^{Ax{+}b}+o^{C^\top\lambda}\big)
= x_k-\gamma\big(Ax_k+b+C^\top\lambda_k\big).
\]
We then run a second attention block with a single head evaluated at \(x_{k+1}\), obtaining
\[
o^{Cx{-}d}=C x_{k+1}-d,
\]
and apply a token-wise ReLU for the dual step,
\[
\lambda_{k+1}=\big[\lambda_k+\eta\,o^{Cx{-}d}\big]_+.
\]
Full \(W_Q,W_K,W_V\) specifications and calculations of the linear attention blocks appear in Appendix~\ref{app:LC}, together with the per-iteration cost analysis in Appendix~\ref{app:runtime:linear}.
\end{proof}

\begin{figure*}[t]
  \centering
  \begin{subfigure}{0.31\textwidth}
    \includegraphics[width=\linewidth]{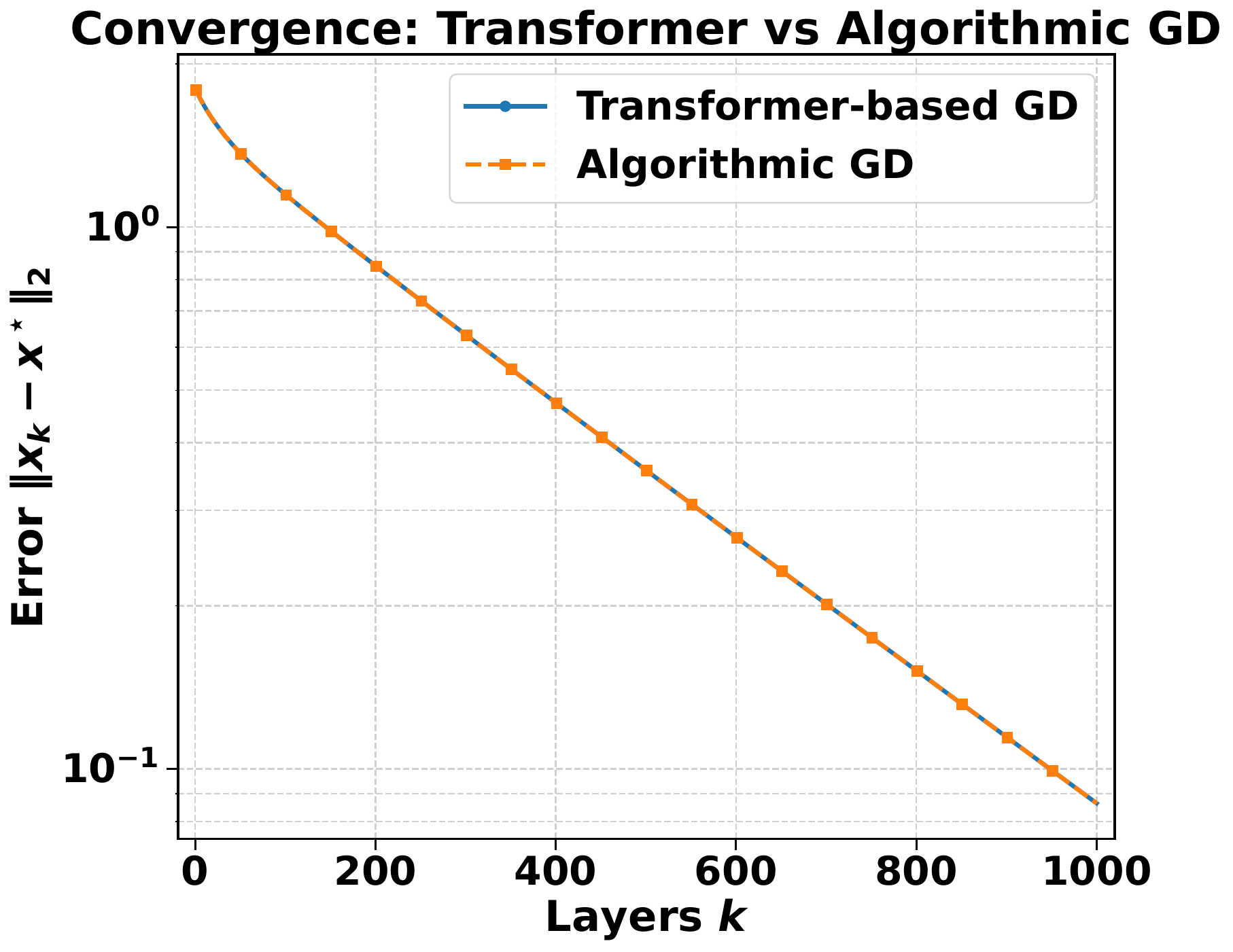}
    \caption{Unconstrained \((\mathrm{U})\).}
    \label{fig:convU}
  \end{subfigure}\hfill
  \begin{subfigure}{0.31\textwidth}
    \includegraphics[width=\linewidth]{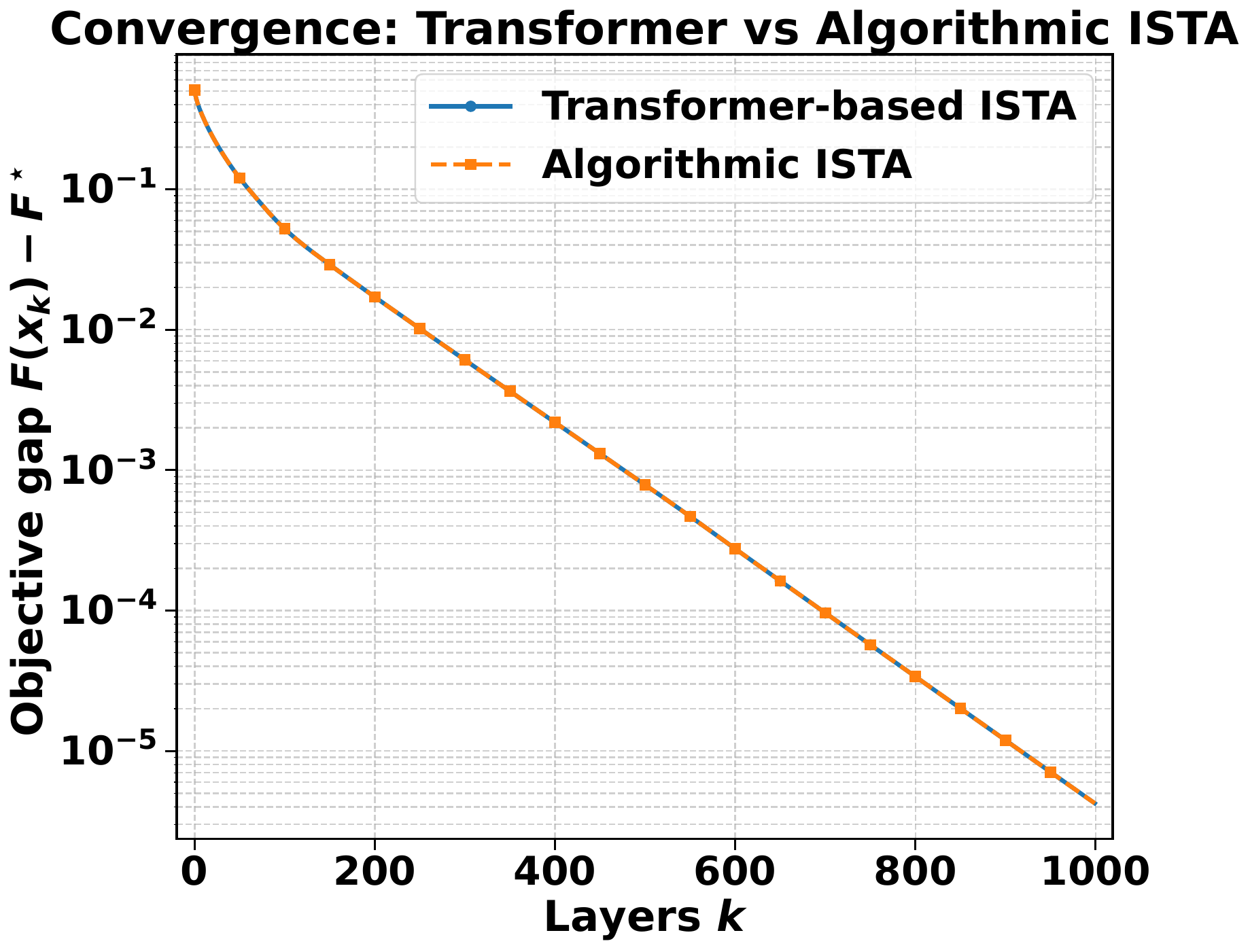}
    \caption{Regularized \((\mathrm{R})\).}
    \label{fig:convR}
  \end{subfigure}\hfill
  \begin{subfigure}{0.31\textwidth}
    \includegraphics[width=\linewidth]{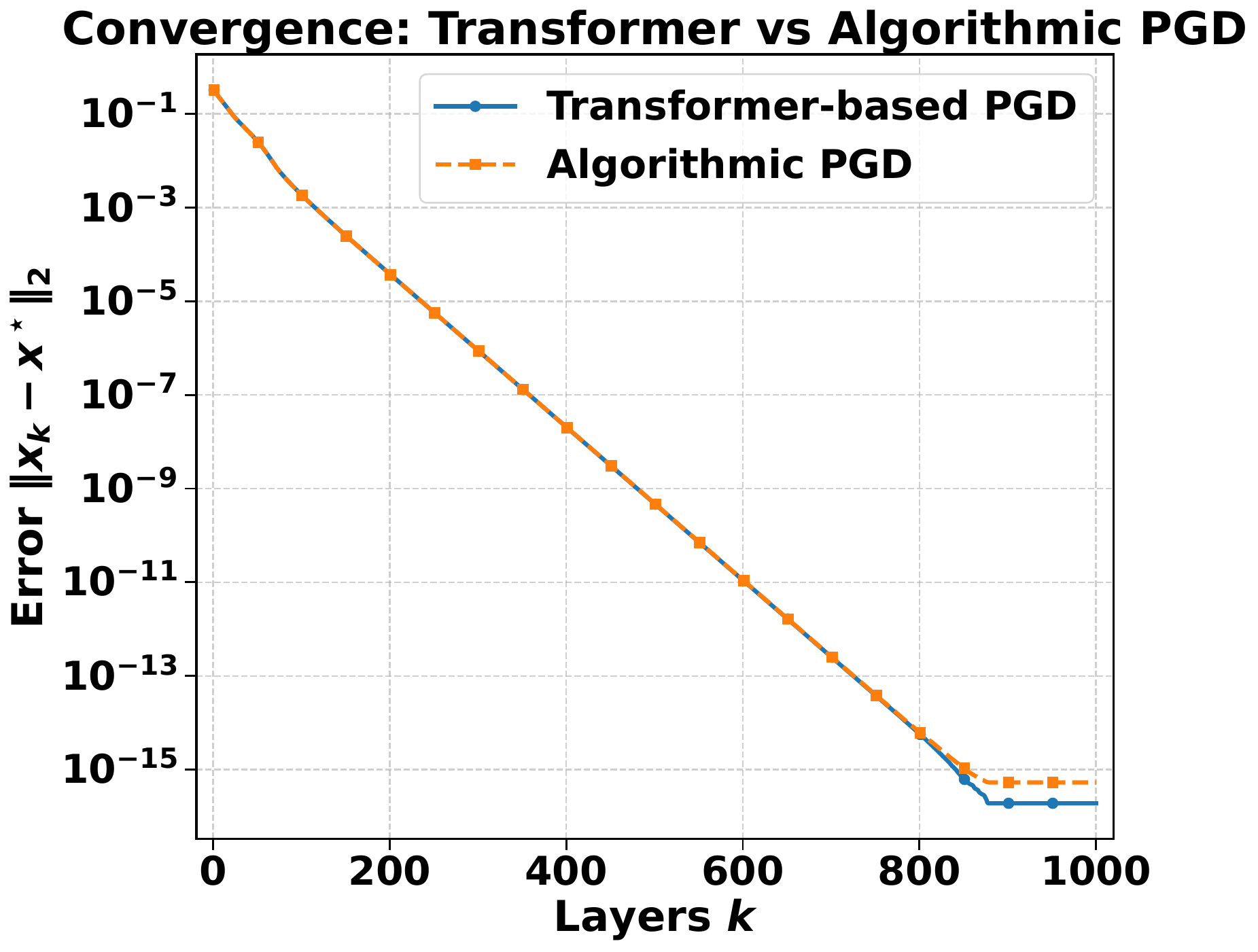}
    \caption{Constrained \((\mathrm{C})\).}
    \label{fig:convC}
  \end{subfigure}
\caption{Convergence on a semilog-\(y\) scale. Each panel compares a standard algorithm (orange) with our transformer-based construction (blue): \textbf{(a)} GD for the unconstrained case (U), \textbf{(b)} ISTA for the regularized case (R), and \textbf{(c)} PGD for the constrained case (C). The curves nearly overlap across layers (depth \(=\) iterations), showing that the transformer mirrors the reference methods.}

  \label{fig:convAll}
\end{figure*}

\subsection{Sparse QPs via Single-Block Linear Attention: \texorpdfstring{\(\ell_1\)}{L1}-Regularized and \texorpdfstring{\(\ell_1\)}{L1}-Constrained}
\label{sec:L1-main}
We now introduce sparsity via an $\ell_1$ penalty (R) or an $\ell_1$ budget (C)—corresponding to the last two QP formulations in Section~\ref{sec:setup}—and extend the (U) block, without changing the attention mechanism, to realize proximal and projected gradient descent algorithms.

\medskip\noindent

\begin{proposition}
\label{prop:ista}
One (U) gradient step followed by a fixed two-layer ReLU FFN with a scalar threshold \(\theta=\gamma\lambda\) yields
\[
x_{k+1}=\mathcal S_\theta\!\big(x_k-\gamma(Ax_k+b)\big)=\operatorname{prox}_{\gamma\lambda\|\cdot\|_1}(y_k).
\]
With \(0<\gamma\le 1/L\), the objective \(F(x)=\tfrac12 x^\top A x + b^\top x + \lambda\|x\|_1\) decreases and \(x_k\) converges to a minimizer.
\end{proposition}

\begin{proof}[Construction]
We compute \(y_k=x_k-\gamma(Ax_k+b)\) with the (U) head and implement soft-thresholding with a fixed two-layer width-\(2n\) ReLU FFN with weights 
\[
W_1=\begin{bmatrix}I_n\\[-2pt]-I_n\end{bmatrix}\!\in\R^{2n\times n},\qquad
W_2=\begin{bmatrix}I_n & -I_n\end{bmatrix}\!\in\R^{n\times 2n},
\]
and bias \(-\theta[\mathbf 1_n;\mathbf 1_n]\). Then
\[
\begin{split}
x_{k+1}
&= W_2\,\mathrm{ReLU}\!\Big(W_1 y_k-\theta\!\begin{bmatrix}\mathbf 1_n\\ \mathbf 1_n\end{bmatrix}\Big) \\
&= (y_k-\theta\mathbf 1_n)_+ - (-y_k-\theta\mathbf 1_n)_+ = \mathcal S_\theta(y_k).
\end{split}
\]

since \((u-\theta)_+ - (-u-\theta)_+=\mathrm{sign}(u)\,(|u|-\theta)_+\) coordinatewise.  
With \(\theta=\gamma\lambda\), this equals \(\operatorname{prox}_{\gamma\lambda\|\cdot\|_1}(y_k)\), i.e., the ISTA update. Full steps are provided in Appendix \ref{app:L1-reg}. Figure~\ref{fig:convR} shows that the transformer construction matches ISTA’s convergence behavior across depth. 
\end{proof}

\begin{proposition}
\label{prop:pgd}
One (U) gradient step followed by the same fixed two-layer ReLU FFN wrapped in a scalar threshold loop performs the exact Euclidean projection onto the \(\ell_1\)-ball:
\[
x_{k+1}=\mathrm{Proj}_{\{\|x\|_1\le B\}}\!\big(x_k-\gamma(Ax_k+b)\big).
\]
With \(0<\gamma\le 1/L\), this is projected gradient descent for (C), so \(x_k\) converges to an optimal solution.
\end{proposition}

\begin{proof}[Construction]
From the (U) head we form \(y_k=x_k-\gamma(Ax_k+b)\).
We then run a scalar threshold loop
\[
\theta_{t+1}=\theta_t+\eta\,\big[\!\|\mathcal S_{\theta_t}(y_k)\|_1-B\big]_+,\qquad
\theta_0=0,\ \ 0<\eta\le \tfrac{1}{n},
\]
and set \(x_{k+1}=\mathcal S_{\theta^\star}(y_k)\) with \(\theta^\star=\lim_t\theta_t\).
This yields the exact Euclidean projection \(x_{k+1}=\mathrm{Proj}_{\{\|x\|_1\le B\}}(y_k)\).
Full calculation is given in Appendix~\ref{app:L1-Cons}. We empirically verify that the explicit transformer constructions mirror the corresponding first-order iterations across depth.
Figure~\ref{fig:convAll} compares GD/ISTA/PGD with their transformer counterparts for (U), (R), and (C), respectively, showing near-overlapping convergence curves. Finally, we provide a corresponding empirical check for the $\ell_1$-constrained case in Figure~\ref{fig:convC}.

\end{proof}


\section{Neural QP Solver Experiments}
\label{sec:neural_qp}


Our theoretical findings establish that a transformer architecture has the expressive power to represent the computational steps of various QP-solving algorithms. We verify in Appendix~\ref{app:weight_recovery} that these constructive primitives are recoverable via gradient-based training from random initialization. Here, we study the more general setting in which a transformer learns the direct mapping from a QP instance to its optimal solution.

\subsection{Data-Driven QP Function Approximation}

We conduct a controlled function-approximation study to assess whether transformer architectures can generalize across families of linearly-constrained QPs. The task is to learn a mapping from a tokenized QP instance to its optimal solution. Each instance is of the form:
\[
\min_{x\in\mathbb R^{n}} \ \tfrac12\,x^\top A x + b^\top x\quad \text{s.t.}\quad C x \preceq d,
\]
where the dataset generator samples the problem data ($A\succ 0$, $b$, $C$, and $d$) and computes the ground-truth optimal solution $x^\star$ for supervision.

Our default configuration uses $n{=}5$ variables and $m{=}3$ constraints with datasets of 50k/10k/10k train/validation/test instances. Problem coefficients are sampled as: (i) $A$ is constructed as $A = GG^\top + 0.1I$ where $G \sim \mathcal{N}(0, I)$ to ensure it is positive definite; (ii) $b \sim \mathcal{N}(0, I)$; (iii) $C \sim \mathcal{N}(0, I)$; and (iv) $d \sim \text{Uniform}(1, 2)$.

To prepare the data for the models, we tokenize each QP instance into a sequence of length $n+m+3$. This sequence consists of one token for each row of $A \in \mathbb{R}^{n \times n}$ and $C \in \mathbb{R}^{m \times n}$ (total $n+m$ tokens), plus three additional tokens for the vector $b$, a zero-padded version of $d$, and a random initializer $x_{\mathrm{init}}$. Each token is a vector in $\mathbb{R}^n$. More details can be found in Appendix \ref{app:neural-exp-details}.

\paragraph{Model architectures and metrics.}
We compare two encoder-only transformer variants: a \textit{SoftmaxTransformer} and a \textit{LinearTransformer}. The SoftmaxTransformer employs standard multi-head self-attention with softmax normalization, and the LinearTransformer replaces the softmax operation with element-wise products.  For a comparison against parameter-matched non-attention baselines (MLP, LSTM) on the same QP tokenization, see Appendix \ref{sec:ablation}.

To evaluate these architectures, we perform a hyperparameter sweep over model depth (number of layers) and the number of attention heads. Performance is primarily assessed using the coefficient of determination ($R^2$) and Normalized Mean Squared Error (NMSE) on the solution vector, with Mean Squared Error (MSE) as the training loss. The metrics are defined as:
\begin{equation}
\begin{gathered}
\text{MSE} = \tfrac{1}{n}\|\hat{x} - x^\star\|_2^2, \quad
\text{NMSE} = \tfrac{\|\hat{x} - x^\star\|_2^2}{\|x^\star\|_2^2}, \quad
R^2 = 1 - \tfrac{\|\hat{x} - x^\star\|_2^2}{\|\bar{x}^\star - x^\star\|_2^2},
\end{gathered}
\end{equation}
where \(\hat{x}\) is the predicted solution, \(x^\star\) is the ground-truth optimal solution, and \(\bar{x}^\star\) is the mean of all ground-truth solutions in the test set. In addition, we report downstream feasibility via constraint violation and optimality via objective sub-optimality in Appendix \ref{sec:additional-qp-results}. We report $n=5,m=3$ in the main text and defer other settings to the supplementary materials.

\paragraph{Results.} Our experimental results demonstrate that transformers are capable of learning to solve QP problems. The best-performing LinearTransformer (8 layers, 2 heads) and SoftmaxTransformer achieved an $R^2$ of \textbf{0.974} and \textbf{0.935}, respectively. Notably, there is a performance advantage for the LinearTransformer architecture. Additional neural QP solver results for higher-dimensional settings, as well as error distributions and performance under distribution shift, are provided in Appendix \ref{sec:additional-qp-results}.


\begin{table*}[t]

\caption{$R^2$ on the QP function-approximation task ($n{=}5$, $m{=}3$) across layers and attention heads.}
  \centering
  
  \label{tab:qp_r2_n5m3}
  \small
  \setlength{\tabcolsep}{6pt}
  \renewcommand{\arraystretch}{1.15}
  \begin{tabular}{c | cccc | cccc}
    \noalign{\hrule height 1.1pt}
    & \multicolumn{4}{c|}{\textbf{SoftmaxTransformer}} & \multicolumn{4}{c}{\textbf{LinearTransformer}} \\
    \cline{2-5}\cline{6-9}
    \textbf{Number of Layers} & \multicolumn{4}{c|}{\textbf{Number of Heads}} & \multicolumn{4}{c}{\textbf{Number of Heads}} \\
    \cline{2-5}\cline{6-9}
     & \textbf{1\vphantom{gy}} & \textbf{2\vphantom{gy}} & \textbf{4\vphantom{gy}} & \textbf{8\vphantom{gy}}
     & \textbf{1\vphantom{gy}} & \textbf{2\vphantom{gy}} & \textbf{4\vphantom{gy}} & \textbf{8\vphantom{gy}} \\
    \hline
    \textbf{1}  & 0.275 & 0.334 & 0.468 & 0.635 & 0.331 & 0.401 & 0.561 & 0.632 \\
    \textbf{2}  & 0.578 & 0.658 & 0.711 & 0.779 & 0.693 & 0.817 & 0.865 & 0.881 \\
    \textbf{4}  & 0.863 & 0.894 & 0.898 & 0.917 & 0.934 & 0.958 & 0.962 & 0.965 \\
    \textbf{8}  & 0.923 & 0.934 & 0.934 & 0.934 & 0.972 & \textbf{0.974} & 0.970 & 0.971 \\
    \textbf{16} & 0.929 & \textbf{0.935} & 0.926 & 0.931 & 0.960 & 0.964 & 0.969 & 0.967 \\
    \noalign{\hrule height 1.1pt} 
  \end{tabular}
\end{table*}

\section{Decision Making with Covariance-Aware Transformers}
In many high-stakes domains like finance and operations, downstream decision rules are solved via quadratic programs with structural constraints. A natural approach is to first forecast the relevant parameters and then solve the QP explicitly—the \emph{predict-then-optimize} (PtO) pipeline. However, this two-stage approach can be suboptimal as noisy forecasts can lead to overly aggressive decisions.

\vspace*{0.5em}

\begin{proposition}[Informal: PtO is suboptimal under estimation noise]
\label{prop:pto_suboptimal}
In a two-asset mean--variance problem with a Gaussian prior on the mean gap and noisy observations, the PtO plug-in rule is strictly dominated in expected utility by the Bayes-optimal policy, which shrinks the estimate by a factor $\rho < 1$ before optimizing.
\end{proposition}\vspace*{0.5em}

While the optimal shrinkage depends on unknown noise parameters, an end-to-end model can learn analogous attenuation directly from the decision loss. Providing covariance information helps the model calibrate how aggressively to react to forecast signals. 

We validate this insight through a portfolio optimization case study, where allocations must satisfy non-negativity, budget, and $\ell_1$ rebalancing constraints. This setting highlights a practical way to integrate second-order information into foundation-model backbones for decision-making in an end-to-end fashion.

\subsection{Problem Formulation and Data}
\label{sec:prob_data}
We consider a portfolio optimization problem with $m$ assets over $T$ time periods. At time $t$, we observe the vector of asset returns $\mathbf{r}_t\in\mathbb{R}^m$ (where $r_{t,i}$ represents the return of asset $i$ at time $t$) and must choose portfolio weights $\mathbf{s}_t\in\mathbb{R}^m$ (where $s_{t,i}$ represents the fraction of wealth allocated to asset $i$ at time $t$) subject to the constraints:

\begin{equation}
\label{eq:portfolio_constraints}
\begin{aligned}
\mathbf{s}_t &\succeq \mathbf{0} && \text{(no short selling for simplicity)},\\
\mathbf{1}^\top \mathbf{s}_t &= 1 && \text{(allocation weights sum to 1)},\\
\|\mathbf{s}_t-\mathbf{s}_{t-1}\|_1 &\leq \gamma && \text{(transaction cost limit)}.
\end{aligned}
\end{equation}

The rebalancing constraint $\|\mathbf{s}_t-\mathbf{s}_{t-1}\|_1 \leq \gamma$ restricts the total absolute change in portfolio weights, where $\gamma$ represents the maximum allowed rebalancing budget. This constraint proxies transaction costs and prevents excessive portfolio turnover.

To establish a target for supervised training, we compute an \textit{oracle allocation}, $\mathbf{s}_t^\star$. This oracle represents the ideal decision at each step, calculated by solving a mean-variance optimization problem assuming perfect foresight of the next period's returns, $\mathbf{r}_t$. The objective is to maximize risk-adjusted returns:
\[
\mathbf{s}_t^\star=\arg\max_{\mathbf{s}_t}\ \mathbf{s}_t^\top \mathbf{r}_t-\lambda\,\mathbf{s}_t^\top \hat{\boldsymbol{\Sigma}}_t \mathbf{s}_t
\quad \text{s.t. }  (\ref{eq:portfolio_constraints}) \text{ holds. }
\]  
Here, $\hat{\boldsymbol{\Sigma}}_t$ is the empirical covariance matrix estimated from historical returns, and $\lambda=0.1$ is the risk aversion parameter that balances expected return against portfolio variance. The oracle solution $\mathbf{s}_t^\star$ represents the optimal allocation given perfect knowledge of future returns $\mathbf{r}_t$.

\paragraph{Data generation.} We synthesize multivariate time series using the Linear Model of Coregionalization (LMC) as described in \cite{geostatistics_LMC}, $\mathbf{X}_i(t)=\sum_{j=1}^{r} w_{ij} f_j(t)$, where $f_j(t)$ are latent Gaussian processes with varied kernels, $w_{ij}$ are mixing weights, and $r$ sets the covariance rank. We then process these time series by scaling with $\epsilon=0.01$ and clipping to the range $[-0.3, 0.8]$ to simulate asset returns. We generate separate training and evaluation datasets, each with $m=16$ assets and $T=1024$ time steps. After obtaining the series of returns, we finalize the input sequences by concatenating historical returns and allocations (a total of \(2m\) channels), enabling the model to reason jointly over past market dynamics and portfolio positions.

\subsection{Experimental Setup}

\paragraph{Proposed architecture.} There exist many transformer-based multivariate time series models. Among them, \emph{TimePFN}—adapted from one of the strongest baselines, PatchTST \citep{Nie2023PatchTST}—serves as the backbone of our approach. \emph{Time2Decide} augments TimePFN with \emph{Covariance-Augmented Tokenization} (CAT): it computes the empirical covariance matrix over the normalized input window of 96 previous steps, maps each covariance row to a token via a multi-layer nonlinear projection, prepends the resulting \(m\) covariance tokens to the standard patch embeddings, applies a unified positional encoding to the entire sequence, and processes them jointly through self-attention. After encoding, we remove the covariance tokens and pass only the temporal tokens to the output projectors to generate predictions for both returns and allocations.

\paragraph{Baselines.} We evaluate our proposed model, \textit{Time2Decide}, against a suite of baselines:  
(i) Oracle (perfect-foresight allocation, the upper bound in performance for any strategy);  
(ii) Predict-then-Optimize (TimePFN-predicted returns are fed into a conventional mathematical QP solver to compute allocations under the constraints);  
(iii) Pretrained (TimePFN used purely as a neural forecaster without covariance tokens);  
(iv) SFT (supervised fine-tuned TimePFN variant);  
(v) Uniform (equal-weight allocation).

\paragraph{Training and Inference.}
SFT and Time2Decide are trained with a combined loss that balances return and allocation prediction:
\[
\mathcal{L}=\lambda_r\|\hat{\mathbf{r}}-\mathbf{r}\|_2^2+\|\hat{\mathbf{s}}-\mathbf{s}^\star\|_2^2,
\]
where $\lambda_r$ is a weighting hyperparameter. We find the optimal $\lambda_r$ for each model via a hyperparameter sweep.

In inference, to ensure all portfolio constraints are satisfied, the model's raw output $\hat{\mathbf{s}}$ is projected into the feasible set by solving the following QP:
\begin{equation}
\begin{gathered}
\mathbf{s}_{t} = \arg\min_{\mathbf{x}} \|\mathbf{x}-\hat{\mathbf{s}}\|_2^2, \quad
\text{s.t. } \mathbf{x}\succeq 0,\ \mathbf{1}^\top \mathbf{x}=1,\ 
             \|\mathbf{x}-\mathbf{s}_{t-1}\|_1 \leq \gamma
\end{gathered}
\end{equation}

\paragraph{Evaluation.}
We evaluate each strategy across a range of rebalancing budgets $\gamma$. Performance is measured by the Mean Squared Error (MSE) between the model's predicted allocations and the oracle's allocations. Crucially, the MSE is computed on the raw model predictions \textit{before} the feasibility projection, providing a stringent standard for evaluation. More experimental details can be found in Appendix \ref{app:decision-exp-details}. 

\subsection{Experimental Results}
Table~\ref{tab:time2decide_mse_summary} summarizes the MSE of each strategy across rebalancing budgets $\gamma$. We report the best MSE for each learning-based method after selecting the optimal combined-loss weights $\lambda_r$  from our sweep, and provide additional results and hyperparameter details in Appendix~\ref{app:additional_portfolio}. Overall, Time2Decide consistently improves over SFT-TimePFN and Pretrained-TimePFN, highlighting the benefit of incorporating covariance information and learning decisions end-to-end. Predict-then-Optimize (PtO), which relies on exact conventional solvers, achieves the lowest errors at small rebalancing budgets (e.g., $0.0245$ at $\gamma=0.5$ and $0.0294$ at $\gamma=0.75$). 

\begin{table*}[htbp]
\caption{MSE performance for each strategy across different rebalancing budgets $\gamma$ (clean data). Lower is better; bold indicates the best method for each $\gamma$.}
\centering
\small
\setlength{\tabcolsep}{6pt}
\renewcommand{\arraystretch}{1.15}
\begin{adjustbox}{max width=0.98\textwidth}
\begin{tabular}{c|ccccccc}
\noalign{\hrule height 1.1pt}
\textbf{Strategy} & \textbf{$\gamma=0.5$} & \textbf{$\gamma=0.75$} & \textbf{$\gamma=1.0$} & \textbf{$\gamma=1.25$} & \textbf{$\gamma=1.5$} & \textbf{$\gamma=1.75$} & \textbf{$\gamma=2.0$} \\
\hline
\textbf{Uniform} & 0.0425 & 0.0444 & 0.0479 & 0.0485 & 0.0505 & 0.0539 & 0.0586 \\
\textbf{Pretrained} & 0.0447 & 0.0458 & 0.0500 & 0.0514 & 0.0523 & 0.0567 & 0.0604 \\
\textbf{SFT} & 0.0408 & 0.0411 & 0.0413 & 0.0446 & 0.0445 & 0.0515 & 0.0565 \\
\textbf{Predict-then-Optimize} & \textbf{0.0245} & \textbf{0.0294} & \textbf{0.0345} & \textbf{0.0376} & 0.0413 & 0.0472 & 0.0553 \\
\textbf{Time2Decide} & 0.0323 & 0.0346 & 0.0386 & 0.0390 & \textbf{0.0410} & \textbf{0.0442} & \textbf{0.0488} \\
\noalign{\hrule height 1.1pt}
\end{tabular}
\end{adjustbox}
\label{tab:time2decide_mse_summary}
\end{table*}

A natural interpretation is that for small $\gamma$, the turnover constraint tightens the feasible set and effectively clips how much the portfolio can change from one step to the next; in this conservative regime, even imperfect return estimates translate to limited action variation, making the plug-in PtO baseline particularly strong. On the other hand, as $\gamma$ increases, the feasible set loosens, leaving more room for estimation noise to induce larger portfolio changes, so calibration becomes more important. This aligns with the mechanism highlighted in Appendix~\ref{sec:app_portfolio}: when decisions are sensitive to noisy inputs, an optimal policy benefits from attenuating overconfident reactions, whereas a plug-in rule corresponds to a fixed, no-attenuation response. Consistent with this view, Time2Decide shows superior performance at larger budgets, surpassing PtO and achieving the best results at higher $\gamma$ values (e.g., $0.0442$ at $\gamma=1.75$ and $0.0488$ at $\gamma=2.0$). We provide detailed discussion in Appendix \ref{sec:app_portfolio}.

\vspace{-5pt}
\section{Conclusion}

In this paper, we presented a theoretical and empirical study of how transformer architectures can be explicitly constructed to solve classes of quadratic programs. By carefully designing attention heads and feed-forward modules, we showed that single transformer blocks can emulate gradient descent, primal–dual iterations, and proximal updates for unconstrained, linearly constrained, and sparse QPs. These constructions demonstrate that transformers can effectively replicate the iterative structures of classical optimization algorithms.  On the experimental side, we verified that transformers can learn to solve QPs when trained on tokenized problem instances, with linear attention variants consistently achieving higher accuracy. Extending beyond the original task, we showed that incorporating covariance information as tokens enables end-to-end decision making under structural constraints. Our results highlight that embedding second-order problem data into the transformer input leads to meaningful improvements over baselines. Overall, our results connect classical optimization methods with transformers, showing that transformers can efficiently emulate iterative solvers. 


\subsection*{Acknowledgments}
This work is supported by the National Science Foundation grants CCF-2046816, CCF-2403075, CCF-2212426, the Office of Naval Research grant N000142412289, the Army Research Office grant W911NF261A127, and a gift from Open Philanthropy. The computational aspects of the research are generously supported by computational resources provided by the Amazon Research Award on Foundation Model Development.

\bibliography{colm2026_conference}
\bibliographystyle{colm2026_conference}

\newpage
\appendix
\section{Appendix}

\section{Proofs of Main Results from Section 3}

\subsection{Full Proof of the Construction for Unconstrained QP}
\label{app:U}

\begin{proof}[Construction]
At iteration \(k\), we assemble the token matrix
\[
Z_k=\big[\tilde a_1^\top;\ldots;\tilde a_n^\top;\ \tilde b^\top;\ \tilde x_k^\top\big]\in\R^{(n+2)\times (2n+1)},
\]
where
\[
\tilde a_i^\top=[a_i^\top\ \ e_i^\top\ \ 0],\qquad
\tilde b^\top=[0_n^\top\ \ b^\top\ \ 1],\qquad
\tilde x_k^\top=[x_k^\top\ \ 0_n^\top\ \ 1].
\]
Here \(a_i^\top\) is the \(i\)-th row of \(A\) and \(\{e_i\}_{i=1}^n\) are standard basis vectors.

We split each token into \emph{content} vs.\ \emph{ID} using
\[
E=[I_n\ 0_{n\times(n+1)}]\in\R^{n\times(2n+1)},\qquad
\mathbf R=[0_{n\times n}\ I_n\ 0_{n\times 1}]\in\R^{n\times(2n+1)},
\]
and we expose a trailing constant via \(S=[0_{1\times (2n)}\ 1]\).

We use a single linear-attention head with
\[
W_Q=[E^\top\ S^\top],\qquad
W_K=[E^\top\ S^\top],\qquad
W_V=\mathbf R^\top .
\]
Then the \(x\)-token query is \(q=\tilde x_k W_Q=[x_k;1]\).
For \(A\)-rows, \(\tilde a_i W_K=[a_i;0]\) and \(\tilde a_i W_V=e_i\);
for the \(b\)-token, \(\tilde b W_K=[0;1]\) and \(\tilde b W_V=b\).
Evaluating the head at the \(x\)-token,
\[
o=\sum_{i=1}^n \langle[x_k;1],[a_i;0]\rangle e_i + \langle[x_k;1],[0;1]\rangle b
= \sum_{i=1}^n (a_i^\top x_k)\,e_i + b = Ax_k+b.
\]
Applying a residual on the \(x\)-row with post-map \(-\gamma I_n\) yields
\[
x_{k+1}=x_k-\gamma(Ax_k+b).
\]
\end{proof}

\subsection{Full Proof of the Construction for Linearly Constrained QP}
\label{app:LC}

\begin{proof}[Construction]
We add a dual token \(\lambda_k\in\mathbb{R}^m_{\ge 0}\). As in the unconstrained case, we split tokens into \emph{content} and \emph{ID} parts and embed every token to a common width \(p:=2n+m\) (zero-padded), then append a constant slot to reach \(p':=p+1\).
Let
\[
S=\big[0_{1\times n}\ 0_{1\times n}\ 0_{1\times m}\ 1\big]\in\mathbb{R}^{1\times p'} .
\]
Define fixed selectors
\[
\begin{aligned}
E_1 &= \big[I_n\ 0_{n\times n}\ 0_{n\times m}\ 0\big]\in\mathbb{R}^{n\times p'}, &
\mathbf R_1 &= \big[0_{n\times n}\ I_n\ 0_{n\times m}\ 0\big]\in\mathbb{R}^{n\times p'},\\
E_2 &= \big[0_{m\times n}\ 0_{m\times n}\ I_m\ 0\big]\in\mathbb{R}^{m\times p'}, &
\mathbf R_2 &= \big[0_{m\times n}\ 0_{m\times n}\ I_m\ 0\big]\in\mathbb{R}^{m\times p'}.
\end{aligned}
\]
Tokens are
\[
\begin{aligned}
\hat a_i^\top &= \big[a_i^\top\ \ e_i^{(n)\top}\ \ 0_m^\top\ \ b_i\big], &
\hat c_i^\top &= \big[c_i^\top\ \ 0_n^\top\ \ e_i^{(m)\top}\ \ {-}d_i\big],\\
\hat x_k^\top &= \big[x_k^\top\ \ 0_n^\top\ \ 0_m^\top\ \ 1\big], &
\hat\lambda_k^\top &= \big[0_n^\top\ \ 0_n^\top\ \ \lambda_k^\top\ \ 1\big].
\end{aligned}
\]

We use three fixed (linear) attention heads with \(W_Q,W_K,W_V\) chosen as below.

\paragraph{Head H\!A{+}b (evaluated at \(x_k\)).}
Queries/keys/values:
\[
q=\hat x_k\!\begin{bmatrix}E_1^\top\\[0pt] S^\top\end{bmatrix}=[x_k;1],\quad
k_i=\hat a_i\!\begin{bmatrix}E_1^\top\\[0pt] S^\top\end{bmatrix}=[a_i;b_i],\quad
v_i=\hat a_i\,\mathbf R_1^\top=e_i^{(n)}.
\]
Output:
\[
o^{Ax{+}b}=\sum_{i=1}^n \langle[x_k;1],[a_i;b_i]\rangle\,e_i^{(n)}=Ax_k+b\in\mathbb{R}^n.
\]

\paragraph{Head H\!C\(^\top\) (evaluated at \(\lambda_k\)).}
\[
q=\hat\lambda_k E_2^\top=\lambda_k,\quad
k_i=\hat c_i\,\mathbf R_2^\top=e_i^{(m)},\quad
v_i=\hat c_i E_1^\top=c_i,
\]
so \(o^{C^\top\lambda}=\sum_{i=1}^m(\lambda_k)_i c_i=C^\top\lambda_k\in\mathbb{R}^n\).

\paragraph{Head H\!C{-}d (evaluated at \(x\)).}
\[
q=\hat x\!\begin{bmatrix}E_1^\top\\[0pt] S^\top\end{bmatrix}=[x;1],\quad
k_i=\hat c_i\!\begin{bmatrix}E_1^\top\\[0 pt] S^\top\end{bmatrix}=[c_i;-d_i],\quad
v_i=\hat c_i\,\mathbf R_2^\top=e_i^{(m)},
\]
so \(o^{Cx{-}d}=\sum_{i=1}^m \langle[x;1],[c_i;-d_i]\rangle e_i^{(m)}=Cx-d\in\mathbb{R}^m\).

\paragraph{Two-block macro.}
We first compute the two head outputs at the \(x\)- and \(\lambda\)-tokens, \(o^{Ax{+}b}=Ax_k+b\) and \(o^{C^\top\lambda}=C^\top\lambda_k\), concatenate them, and apply the linear map \(W_O^{(1)}=[-\gamma I_n\;\;-\gamma I_n]\) with a residual on the \(x\)-row. This yields
\[
x_{k+1}=x_k-\gamma\big(Ax_k+b+C^\top\lambda_k\big).
\]

For the second block, using the updated \(x_{k+1}\), we evaluate the head \(o^{Cx{-}d}=C x_{k+1}-d\), scale it with \(W_O^{(2)}=\eta I_m\), and perform a token-wise ReLU to enforce non-negativity:
\[
\lambda_{k+1}=\big[\lambda_k+\eta\,(C x_{k+1}-d)\big]_+.
\]

Together, these two blocks implement exactly the Arrow–Hurwicz (projected primal–dual) step of Proposition~\ref{prop:LC}.

\end{proof}

\subsection{Full Proof of the Construction for \texorpdfstring{\(\ell_1\)}{L1}-Regularized QP}
\label{app:L1-reg}

\begin{proof}[Construction]
We compute \(y_k=x_k-\gamma(Ax_k+b)\) with the (U) head and implement soft-thresholding with a fixed two-layer width-\(2n\) ReLU FFN with weights 
\[
W_1=\begin{bmatrix}I_n\\[-2pt]-I_n\end{bmatrix}\!\in\R^{2n\times n},\qquad
W_2=\begin{bmatrix}I_n & -I_n\end{bmatrix}\!\in\R^{n\times 2n},\qquad
h=W_1 y_k-\theta\begin{bmatrix}\mathbf 1_n\\ \mathbf 1_n\end{bmatrix},
\]
Let us write \(h=\begin{bmatrix}h^{+}\\ h^{-}\end{bmatrix}\) with
\[
h^{+}=y_k-\theta\mathbf 1,\qquad h^{-}=-y_k-\theta\mathbf 1.
\]
After the ReLU, we have
\[
r=ReLU(h)=\begin{bmatrix}ReLU(h^{+})\\ ReLU(h^{-})\end{bmatrix}
=\begin{bmatrix}(y_k-\theta\mathbf 1_n)_+\\ (-y_k-\theta\mathbf 1_n)_+\end{bmatrix}.
\]
Applying the second linear layer gives
\[
x_{k+1}=W_2 r
= (y_k-\theta\mathbf 1_n)_+ - (-y_k-\theta\mathbf 1_n)_+.
\]
Coordinatewise, for each \(i\in[n]\),
\[
x_{k+1,i}=\max\{y_{k,i}-\theta,0\}-\max\{-y_{k,i}-\theta,0\}.
\]
A short case analysis shows
\[
x_{k+1,i}=
\begin{cases}
y_{k,i}-\theta, & y_{k,i}>\theta,\\[2pt]
0, & |y_{k,i}|\le \theta,\\[2pt]
y_{k,i}+\theta, & y_{k,i}<-\theta,
\end{cases}
\quad\Longleftrightarrow\quad
x_{k+1,i}=\operatorname{sign}(y_{k,i})\,\big(|y_{k,i}|-\theta\big)_+ \;=\; \mathcal S_\theta(y_{k,i}).
\]
Thus, in vector form,
\[
x_{k+1}=W_2\,ReLU\!\Big(W_1 y_k-\theta\!\begin{bmatrix}\mathbf 1_n\\ \mathbf 1_n\end{bmatrix}\Big)
=\mathcal S_\theta(y_k)=\operatorname{prox}_{\gamma\lambda\|\cdot\|_1}(y_k),
\]
with \(\theta=\gamma\lambda\). This realizes the ISTA update we stated.
\end{proof}

\subsection{Full Proof of the Construction for \texorpdfstring{\(\ell_1\)}{L1}-Constrained QP}
\label{app:L1-Cons}

\begin{proof}[Construction]
We reuse the (U) head to compute \(y_k=x_k-\gamma(Ax_k+b)\). The fixed two-layer ReLU FFN from Proposition~\ref{prop:ista} implements soft-thresholding, so
\[
x_t=\mathcal S_{\theta_t}(y_k),\qquad
s_t=\|x_t\|_1.
\]
We drive the scalar threshold toward the \(\ell_1\)-budget by
\[
\theta_{t+1}=\theta_t+\eta\,\mathrm{ReLU}(s_t-B),\qquad \theta_0=0,\ \ 0<\eta\le\frac{1}{n}.
\]

Let \(s(\theta)=\|\mathcal S_\theta(y_k)\|_1=\sum_{i=1}^n(|y_{k,i}|-\theta)_+\) and \(r_t:=\mathrm{ReLU}(s_t-B)\ge0\).
On any interval where the active set \(A(\theta)=\{i:\ |y_{k,i}|>\theta\}\) is fixed with size \(m\ge1\),
\[
s(\theta)=\sum_{i\in A(\theta)}|y_{k,i}|-m\,\theta,
\quad\Rightarrow\quad
r_{t+1}=s(\theta_t+\eta r_t)-B\le (1-m\eta)\,r_t.
\]
Thus \(0<\eta\le 1/n\) yields \(r_{t+1}\in[0,r_t]\) and \(r_t\downarrow0\); at breakpoints \(s\) is nonincreasing, so monotonicity persists. Hence \(\theta_t\to\theta^\star\ge0\) with \(s(\theta^\star)=\min\{B,\|y_k\|_1\}\), and we return
\[
x_{k+1}=\mathcal S_{\theta^\star}(y_k).
\]
Finally, \(\mathcal S_{\theta^\star}(y_k)\) is the unique solution of
\(\min_x \tfrac12\|x-y_k\|_2^2\ \text{s.t.}\ \|x\|_1\le B\) by the KKT conditions
\(0\in x-y_k+\theta^\star\partial\|x\|_1\), \(\theta^\star\ge0\), \(\theta^\star(\|x\|_1-B)=0\).
Therefore, the loop computes the exact Euclidean projection onto the \(\ell_1\)-ball.
\end{proof}

\section{Computational Cost per Iteration}
\label{app:runtime}

Our constructions implement \emph{one solver iteration per transformer layer}. Therefore, it is natural to compare the cost of a single transformer step to the corresponding first-order iteration. We organize our analysis by problem class: unconstrained and $\ell_1$-penalized QPs (Subsection~\ref{app:runtime:unconstrained}), which depend only on the variable dimension $n$, and linearly constrained QPs (Subsection~\ref{app:runtime:linear}), which depend on both $n$ and the number of constraints $m$. We report the multiplicative overhead:
\[
\text{Overhead} \;=\; \frac{\text{Transformer}}{\text{Classical}}\,.
\]

\subsection{Per-Iteration Cost for (U), (R), and (C)}
\label{app:runtime:unconstrained}

These three problem classes---corresponding to Propositions~\ref{prop:unconstrained}, \ref{prop:ista}, and~\ref{prop:pgd}---involve only the quadratic matrix $A \in \mathbb{R}^{n \times n}$ and vectors in $\mathbb{R}^n$, with no explicit linear inequality constraints. The computational cost depends solely on the variable dimension $n$.

\paragraph{Asymptotic complexity.}
For dense QPs, both classical GD/ISTA and our transformer steps are dominated by the matrix--vector product $Ax$, which costs $O(n^2)$ time per iteration.
ISTA additionally applies soft-thresholding at $O(n)$ cost, negligible relative to $n^2$ in the dense regime.
For $\ell_1$-constrained PGD, the classical exact projection onto the $\ell_1$-ball runs in $O(n\log n)$ time, whereas our threshold-loop projection costs $O(Tn)$ for $T$ inner updates; the overall step remains dominated by the $O(n^2)$ product.

\paragraph{Wall-clock time per step.}
Table~\ref{tab:runtime_unconstrained} reports per-step timing for these three problem classes.

\begin{table}[htbp]
\centering
\small
\setlength{\tabcolsep}{3pt}
\renewcommand{\arraystretch}{1.05}
\caption{Per-step timing (ms) for unconstrained QP (U), $\ell_1$-regularized QP (R), and $\ell_1$-constrained QP (C). These problems depend only on the variable dimension $n$.}
\label{tab:runtime_unconstrained}
\begin{tabular}{@{}c|ccc|ccc|ccc@{}}
\toprule
& \multicolumn{3}{c|}{\textbf{Gradient Descent (U)}} 
& \multicolumn{3}{c|}{\textbf{ISTA (R)}} 
& \multicolumn{3}{c}{\textbf{Projected GD (C)}} \\
\cmidrule(lr){2-4}\cmidrule(lr){5-7}\cmidrule(lr){8-10}
$n$ 
& Class. & Transf. & Overhead 
& Class. & Transf. & Overhead
& Class. & Transf. & Overhead \\
\midrule
16  & 0.040 & 0.050 & 1.25$\times$ & 0.052 & 0.070 & 1.35$\times$ & 0.108 & 0.364 & 3.37$\times$ \\
32  & 0.286 & 0.310 & 1.08$\times$ & 0.305 & 0.341 & 1.12$\times$ & 0.416 & 0.696 & 1.67$\times$ \\
64  & 0.327 & 0.341 & 1.04$\times$ & 0.358 & 0.389 & 1.09$\times$ & 0.700 & 0.970 & 1.39$\times$ \\
128 & 0.595 & 0.642 & 1.08$\times$ & 0.658 & 0.719 & 1.09$\times$ & 1.452 & 1.653 & 1.14$\times$ \\
\bottomrule
\end{tabular}
\end{table}

\paragraph{Discussion.}
For Gradient Descent and ISTA, the transformer step achieves near iteration-level parity with overheads of $1.04$--$1.35\times$ across dimensions. 
Projected Gradient Descent exhibits larger overhead at small $n$ due to the $O(Tn)$ threshold-loop projection, but this decreases from $3.37\times$ at $n{=}16$ to $1.14\times$ at $n{=}128$ as the $O(n^2)$ matrix--vector product dominates.

\subsection{Per-Iteration Cost for (LC)}
\label{app:runtime:linear}

The linearly constrained problem class---corresponding to Proposition~\ref{prop:LC}---involves explicit linear inequality constraints $Cx \preceq d$ with constraint matrix $C \in \mathbb{R}^{m \times n}$. The Arrow--Hurwicz primal-dual method requires two sequential attention blocks per iteration: one for the primal update and one for the dual update. The computational cost depends on both the variable dimension $n$ and the number of constraints $m$.

\paragraph{Asymptotic complexity.}
Each Arrow--Hurwicz iteration requires forming $Ax$ at $O(n^2)$ cost, plus $Cx$ and $C^\top\lambda$ at $O(mn)$ cost each, giving $O(n^2 + mn)$ per step. The ReLU operation for dual feasibility adds $O(m)$ cost, which is negligible.

\paragraph{Wall-clock time per step.}
Table~\ref{tab:runtime_arrow_hurwicz} reports per-step timing for various $(n, m)$ configurations.

\begin{table}[htbp]
\centering
\small
\setlength{\tabcolsep}{5pt}
\renewcommand{\arraystretch}{1.05}
\caption{Per-step timing (ms) for linearly constrained QP (LC) using Arrow--Hurwicz iterations. This problem depends on both the variable dimension $n$ and the number of constraints $m$.}
\label{tab:runtime_arrow_hurwicz}
\begin{tabular}{@{}cc|ccc@{}}
\toprule
$n$ & $m$ & Classical (ms) & Transformer (ms) & Overhead \\
\midrule
16  & 8   & 0.136 & 0.169 & 1.25$\times$ \\
16  & 16  & 0.178 & 0.210 & 1.18$\times$ \\
32  & 8   & 0.543 & 0.580 & 1.07$\times$ \\
32  & 16  & 1.241 & 1.263 & 1.02$\times$ \\
32  & 32  & 1.316 & 1.336 & 1.02$\times$ \\
64  & 16  & 1.339 & 1.388 & 1.04$\times$ \\
64  & 32  & 1.448 & 1.432 & 0.99$\times$ \\
64  & 64  & 1.499 & 1.572 & 1.05$\times$ \\
128 & 32  & 2.282 & 2.712 & 1.19$\times$ \\
128 & 64  & 2.911 & 3.501 & 1.20$\times$ \\
\bottomrule
\end{tabular}
\end{table}

\paragraph{Discussion.} The Arrow--Hurwicz transformer construction achieves overheads ranging from $0.99\times$ to $1.25\times$ across all tested configurations. Notably, at $(n, m) = (64, 32)$, the transformer construction is \emph{faster} than the classical implementation, likely due to improved memory access patterns from the unified token representation $Z = [Q; C; c; d; x; \lambda]$.
As problem size increases, overheads remain modest ($1.19$--$1.20\times$ at $n{=}128$), demonstrating that the two-block attention structure required for primal-dual updates does not introduce significant computational burden beyond the single-block constructions.

\section{Experimental Details} 

\label{sec:exp-details}
In this section, we provide comprehensive details regarding the setup of our experiments to ensure reproducibility. All experiments are implemented in PyTorch and executed on NVIDIA L40S GPUs.

\subsection{Neural QP Solver}
\label{app:neural-exp-details}

\subsubsection{Constructive Primitive Recovery}
\label{app:weight_recovery}

To test whether the constructive QP primitives in
Section~\ref{sec:all-theory} can be learned from random initialization, we
instantiate the exact tokenization and architectures used in the corresponding
unconstrained, $\ell_1$-regularized, and $\ell_1$-constrained constructions.
All trainable parameters are randomly initialized, with no warm-start from the
analytical weights, and models are supervised using one-step updates generated
by the corresponding first-order algorithm. We additionally supervise the
intermediate attention output, e.g., $\|o-(Ax_k+b)\|_2^2$, to remove the
scale ambiguity between the attention weights and the subsequent output map.
We use $n=5$, with $20{,}000$ training and $2{,}000$ test samples. We sample
$A=GG^\top+0.1I$ with $G\sim\mathcal{N}(0,I)$, $b\sim\mathcal{N}(0,I)$, and
$x_0\sim\mathcal{N}(0,I)$. The step size is $\gamma=0.1$; for the
$\ell_1$-regularized problem we use $\lambda=0.5$ (soft-threshold
$\theta=0.05$), and for the $\ell_1$-constrained problem we set $B=2.0$.
Models are trained using Adam with learning rate $3\times10^{-3}$, cosine
annealing, and batch size $256$.

\begin{wraptable}{r}{0.50\linewidth}
\vspace{-8pt}
\centering
\caption{Recovery of the constructive QP primitives.}
\label{tab:weight_recovery}
\small
\renewcommand{\arraystretch}{1.15}
\setlength{\tabcolsep}{4pt}
\begin{adjustbox}{max width=0.49\textwidth}
\begin{tabular}{lccc}
\toprule
\textbf{Metric}
& \textbf{Unconstr.}
& \textbf{$\ell_1$-Reg.}
& \textbf{$\ell_1$-Constr.} \\
\midrule
$W_Q$ col. cos. & 0.9718 & 0.9716 & 0.9712 \\
$W_K$ col. cos. & 0.9457 & 0.9529 & 0.9464 \\
$W_V$ col. cos. & 1.0000 & 1.0000 & 1.0000 \\
FFN / proj. $R^2$ & -- & 0.9988 & 1.0000 \\
\bottomrule
\end{tabular}
\end{adjustbox}
\vspace{-6pt}
\end{wraptable}

Table~\ref{tab:weight_recovery} reports the column-wise cosine similarity
between the learned and analytical $W_Q$, $W_K$, and $W_V$. The learned $W_Q$
and $W_K$ strongly align with the constructive weights, while $W_V$ achieves
perfect column-wise cosine alignment across all three QP classes. For the
$\ell_1$ variants, the FFN realizes soft-thresholding in the regularized case
and the exact $\ell_1$-ball projection in the constrained case, both recovered
with near-perfect $R^2$ values. Together, these results confirm that the
QP-solving primitives identified in Section~\ref{sec:all-theory} are
recoverable from random initialization through gradient-based optimization
under the corresponding tokenization and supervision.

\subsubsection{QP Dataset Generation}
\label{app:qp-data-gen}
We generate synthetic QP instances of the form $\min_{x} \frac{1}{2}x^\top A x + b^\top x \text{ s.t. } Cx \preceq d$. The symmetric positive definite matrix $A \in \mathbb{R}^{n \times n}$ is constructed as $A = G G^\top + 0.1 I$, where entries of $G \in \mathbb{R}^{n \times n}$, as well as $b \in \mathbb{R}^n$ and constraint matrix $C \in \mathbb{R}^{m \times n}$, are sampled from a standard normal distribution $\mathcal{N}(0, 1)$. The bound vector $d \in \mathbb{R}^m$ is sampled from the uniform distribution $\mathcal{U}[1, 2]$. Ground truth solutions $x^*$ are computed using the OSQP solver accessed via CVXPY. For problem dimension configurations $(n, m) \in \{(5, 3), (7, 3), (7, 6), (10, 3), (10, 6), (10, 9)\}$, we generate a dataset consisting of $50,000$ training samples, $10,000$ validation samples, and $10,000$ test samples. For larger instances with configurations $(n, m) \in \{(15, 10), (20, 8), (20, 16)\}$, we increase the dataset size to $100,000$ training samples, $20,000$ validation samples, and $20,000$ test samples to ensure sufficient coverage of the higher-dimensional space.

\subsubsection{Model Architectures}
\label{app:model-arch}
We evaluate two architectures: \texttt{SoftmaxTransformer}, which uses standard scaled dot-product attention, and \texttt{LinearTransformer}, which employs the linear attention mechanism. Both models project the input problem data (tokenized rows of $A, C$ and vectors $b, d$) into a hidden embedding dimension of $d_{\text{model}} = 256$. We conduct a grid search over model depth $L \in \{1, 2, 4, 8, 16, 32\}$ and the number of attention heads $H \in \{1, 2, 4, 8\}$ to analyze the impact of model capacity. A dropout rate of $0.1$ is applied during training to prevent overfitting.

\subsubsection{Training}
\label{app:training}
All models are trained to minimize the Mean Squared Error (MSE) between the predicted $x$ and the ground truth optimal solution $x^*$. We use the AdamW optimizer with an initial learning rate of $1 \times 10^{-4}$ and a weight decay of $0.02$. We utilize a learning rate scheduler (\texttt{ReduceLROnPlateau}) that reduces the learning rate by a factor of $0.5$ if the validation loss does not improve for $5$ epochs, with a minimum learning rate of $10^{-6}$. The random seed is fixed at $42$ for all experiments to ensure reproducibility. Training proceeds for a maximum of $500$ epochs. To prevent overfitting, we employ early stopping with a patience of $30$ epochs, monitoring the validation loss.

\subsection{Decision Making with Covariance-Aware Transformers}
\label{app:decision-exp-details}

\subsubsection{Data Generation}
We utilize LMC to generate synthetic financial return data with dynamic covariance structures. The generative process uses the following specific hyperparameters to simulate market conditions: \texttt{dirichlet\_min=0.01}, \texttt{dirichlet\_max=0.75}, \texttt{weibull\_shape=6}, and \texttt{weibull\_scale=125}. Each generated time series has a length of $T=1024$. For our experiments, we construct training datasets of size $N_{\text{train}}=300$ and evaluation datasets of size $N_{\text{eval}}=30$. The input sequence length (lookback window) $L$ is set to 96.

\subsubsection{Inference}
During inference, we utilize a rolling window protocol across all strategies. At each time step $t$, the system accesses a historical context window to update the portfolio. The specific mechanism for computing the next allocation $\mathbf{s}_{t}$ depends on the strategy:

\begin{itemize}
    \item \textbf{(i) Oracle}: At time $t$, it uses the \textit{ground-truth} return $\mathbf{r}_{t}$ (perfect foresight) and the empirical covariance matrix $\hat{\bm{\Sigma}}_t$ estimated from the previous $L=96$ returns. It solves the constrained QP problem using the OSQP solver via the CVXPY interface:
    \begin{align*}
    \max_{\mathbf{s}_{t}} \quad & \mathbf{s}_{t}^\top \mathbf{r}_{t} - \lambda \mathbf{s}_{t}^\top \hat{\bm{\Sigma}}_t \mathbf{s}_{t} \\
    \text{s.t.} \quad & \mathbf{s}_{t} \in \Delta_m, \quad \|\mathbf{s}_{t} - \mathbf{s}_{t-1}\|_1 \leq \gamma
    \end{align*}
    where $\Delta_m$ is the standard simplex. The resulting $\mathbf{s}_{t}$ is utilized for the next step's turnover constraint.
    
    \item \textbf{(ii) Uniform}: A naive baseline where the portfolio is equally distributed across all assets at every step, i.e., $\mathbf{s}_{t, i} = \frac{1}{m}$.
    
    \item \textbf{(iii) Predict-then-Optimize}: First, a pretrained TimePFN model uses the history of actual returns $\mathbf{r}_{t-L:t-1}$ to forecast the expected return vector $\hat{\mathbf{r}}_{t}$. The empirical covariance matrix $\hat{\bm{\Sigma}}_t$ is estimated from the actual history $\mathbf{r}_{t-L:t-1}$. We then solve the same QP problem as the Oracle, but using the \textit{predicted} return $\hat{\mathbf{r}}_{t}$ instead of the ground truth.
    
    \item \textbf{(iv) End-to-End Neural Strategies (Time2Decide, Pretrained, SFT)}: These models generate allocations directly without solving a QP explicitly. At time $t$, the input window $H_t$ consists of past \textit{actual} returns $\mathbf{r}_{t-L:t-1}$ and past \textit{projected} allocations $\mathbf{s}_{t-L:t-1}$. The models output a raw prediction vector $\tilde{\mathbf{s}}_{t}$.
    To ensure validity, we apply a differentiable projection layer that maps $\tilde{\mathbf{s}}_{t}$ to the closest valid allocation $\hat{\mathbf{s}}_{t}$ satisfying the simplex and turnover constraints:
    \begin{align*}
    \min_{\hat{\mathbf{s}}_{t}} \quad & \| \hat{\mathbf{s}}_{t} - \tilde{\mathbf{s}}_{t} \|_2^2 \\
    \text{s.t.} \quad & \hat{\mathbf{s}}_{t} \in \Delta_m, \quad \| \hat{\mathbf{s}}_{t} - \mathbf{s}_{t-1} \|_1 \leq \gamma
    \end{align*}
    This projection is solved using an SLSQP solver. The window for the next time step $t+1$ is updated autoregressively: we discard the oldest observation and append the \textit{actual} realized return $\mathbf{r}_{t}$ and the \textit{valid} projected allocation $\hat{\mathbf{s}}_{t}$ to the history.

\end{itemize}

\subsubsection{Training}
We employ a Supervised Fine-Tuning (SFT) approach to train both the \textit{Time2Decide} model and the \textit{SFT} baseline. Both models are initialized from a pretrained TimePFN checkpoint. The training objective is a combined loss function that jointly minimizes the error in return forecasting and portfolio allocation:
\[
\mathcal{L} = \lambda_r \|\hat{\mathbf{r}} - \mathbf{r}\|_2^2 + \|\hat{\mathbf{s}} - \mathbf{s}^\star\|_2^2
\]
where $\hat{\mathbf{r}}$ and $\mathbf{r}$ denote the predicted and ground truth return vectors, $\hat{\mathbf{s}}$ and $\mathbf{s}^\star$ denote the predicted and Oracle (optimal) allocation vectors, and $\lambda_r$ is a hyperparameter that weights the return loss. We conduct a comprehensive hyperparameter sweep for $\lambda_r \in \{0, 10, 50, \dots, 320000\}$ to identify the optimal configuration for each model.
Optimization is performed using the Adam optimizer with a maximum learning rate of $1 \times 10^{-4}$ and a OneCycleLR scheduler with \texttt{pct\_start=0.3}. The models are trained for a maximum of 50 epochs with a batch size of 32, employing early stopping with a patience of 5 epochs to prevent overfitting. 

\subsubsection{Evaluation Metrics}
The metric for experimental evaluation is the Mean Squared Error (MSE) relative to the optimal Oracle strategy.

Crucially, to rigorously assess the models' learning capability, we compute the MSE using the \textit{raw} model predictions $\tilde{\mathbf{s}}$ (before the feasibility projection described in the Inference section) against the Oracle allocations:
\[
\text{MSE} = \frac{1}{T \cdot N_{\text{eval}}} \sum_{n=1}^{N_{\text{eval}}} \sum_{t=1}^{T} \| \tilde{\mathbf{s}}_{n,t} - \mathbf{s}^\star_{n,t} \|_2^2
\]

\section{Architectural Ablation: MLP, LSTM vs. Transformer}
\label{sec:ablation}

In our neural QP solver experiments, we include MLP and LSTM baselines designed specifically for QP inputs for comparison. The \textbf{MLP} is a deep residual network that flattens the input tokens into a single vector. It employs a 6-layer architecture with widths ranging from $6\times$ to $4\times$ the hidden dimension ($d=256$), using pre-activation LayerNorm, ReLU, dropout, and projection-based residual connections to ensure trainability. The \textbf{LSTM} baseline processes the QP tokens sequentially using a 4-layer stacked LSTM with the same hidden dimension. Crucially, the LSTM shares the exact same input embedding and output MLP head as the transformer models, ensuring that performance differences arise solely from the sequence modeling mechanism (recurrence vs. attention). To ensure a fair comparison, the parameter counts of both baselines match those of our largest transformer model. In the following, we present the $R^2$ results for the different models, where each problem dimension is evaluated using its best-performing transformer architecture.

\begin{table*}[htbp]
  \centering
  \caption{$R^2$ Comparison: Transformer architectures vs. MLP and LSTM baselines.}
  \label{tab:arch_ablation}
  \small
  \renewcommand{\arraystretch}{1.15}
  \begin{tabular}{cc|cccc}
    \noalign{\hrule height 1.1pt}
    \multicolumn{2}{c|}{\textbf{Problem Dimensions}} 
      & \multicolumn{4}{c}{\textbf{Model Architecture}} \\
    \cline{1-2}\cline{3-6}
      \textbf{$n$} & \textbf{$m$} & \textbf{LinearTransformer} & \textbf{SoftmaxTransformer} & \textbf{MLP} & \textbf{LSTM} \\
    \hline
    5 & 3   & 0.9738 & 0.9354 & 0.6715 & 0.7045 \\
    7 & 3   & 0.9600 & 0.8914 & 0.5271 & 0.4734 \\
    7 & 6   & 0.9355 & 0.8713 & 0.5071 & 0.4880 \\
    10 & 3  & 0.9054 & 0.8423 & 0.3566 & 0.2189 \\
    10 & 6  & 0.8921 & 0.8157 & 0.3642 & 0.2357 \\
    10 & 9  & 0.9062 & 0.8430 & 0.4043 & 0.2693 \\
    \noalign{\hrule height 1.1pt}
  \end{tabular}
\end{table*}

\section{Additional Neural QP Solver Results} 
\label{sec:additional-qp-results}

We present more results for neural QP solver experiments across higher problem dimensions beyond the main text's focus on $n=5, m=3$ configurations. See Tables~\ref{tab:qp_r2_n7m3_appendix}, \ref{tab:qp_r2_n7m6_appendix}, \ref{tab:qp_r2_n10m3_appendix}, \ref{tab:qp_r2_n10m6_appendix}, \ref{tab:qp_r2_n10m9_appendix}, 
 \ref{tab:qp_r2_n15m10_appendix},  \ref{tab:qp_r2_n20m8_appendix}, and  \ref{tab:qp_r2_n20m16_appendix} across problem sizes $(n=7, m=3)$, $(n=7, m=6)$, $(n=10, m=3)$, $(n=10, m=6)$, $(n=10, m=9)$, $(n=15, m=10)$, $(n=20, m=8)$, and $(n=20, m=16)$  respectively.

These comprehensive results demonstrate that transformers are capable of solving unseen QP problems of higher dimensions. Again, LinearTransformer consistently outperforms SoftmaxTransformer across problem dimensions and hyperparameter combinations. The entries marked as NaN indicate that the Transformer architecture with the specified depth $l$ and hidden dimension $h$ was unable to learn the QP solution for the given higher-dimensional problem and the data size, resulting in a negative $R^2$.

\vspace{0.5cm}

\begin{table*}[ht]

  \caption{$R^2$ on the QP Function-Approximation Task ($n{=}7$, $m{=}3$).}
  \centering
  
  \label{tab:qp_r2_n7m3_appendix}
  \small
  \setlength{\tabcolsep}{6pt}
  \renewcommand{\arraystretch}{1.15}
  \begin{tabular}{c | cccc | cccc}
    \noalign{\hrule height 1.1pt}
    & \multicolumn{4}{c|}{\textbf{SoftmaxTransformer}} & \multicolumn{4}{c}{\textbf{LinearTransformer}} \\
    \cline{2-5}\cline{6-9}
    \textbf{Number of Layers} & \multicolumn{4}{c|}{\textbf{Number of Heads}} & \multicolumn{4}{c}{\textbf{Number of Heads}} \\
    \cline{2-5}\cline{6-9}
     & \textbf{1\vphantom{gy}} & \textbf{2\vphantom{gy}} & \textbf{4\vphantom{gy}} & \textbf{8\vphantom{gy}}
     & \textbf{1\vphantom{gy}} & \textbf{2\vphantom{gy}} & \textbf{4\vphantom{gy}} & \textbf{8\vphantom{gy}} \\
    \hline
    \textbf{1}  & 0.226 & 0.253 & 0.315 & 0.472 & 0.239 & 0.299 & 0.377 & 0.478 \\
    \textbf{2}  & 0.485 & 0.544 & 0.482 & 0.485 & 0.652 & 0.755 & 0.784 & 0.811 \\
    \textbf{4}  & 0.782 & 0.829 & 0.844 & 0.814 & 0.928 & 0.948 & 0.942 & 0.937 \\
    \textbf{8}  & 0.864 & 0.891 & 0.890 & 0.883 & \textbf{0.960} & 0.952 & 0.949 & 0.943 \\
    \textbf{16} & 0.871 & 0.878 & 0.876 & \textbf{0.891} & 0.907 & 0.927 & 0.937 & 0.937 \\
    \noalign{\hrule height 1.1pt} 
  \end{tabular}
\end{table*}

\vspace{0.5cm}

\begin{table*}[h]

  \caption{$R^2$ on the QP Function-Approximation Task ($n{=}7$, $m{=}6$).}
  \centering
  
  \label{tab:qp_r2_n7m6_appendix}
  \small
  \setlength{\tabcolsep}{6pt}
  \renewcommand{\arraystretch}{1.15}
  \begin{tabular}{c | cccc | cccc}
    \noalign{\hrule height 1.1pt}
    & \multicolumn{4}{c|}{\textbf{SoftmaxTransformer}} & \multicolumn{4}{c}{\textbf{LinearTransformer}} \\
    \cline{2-5}\cline{6-9}
    \textbf{Number of Layers} & \multicolumn{4}{c|}{\textbf{Number of Heads}} & \multicolumn{4}{c}{\textbf{Number of Heads}} \\
    \cline{2-5}\cline{6-9}
     & \textbf{1\vphantom{gy}} & \textbf{2\vphantom{gy}} & \textbf{4\vphantom{gy}} & \textbf{8\vphantom{gy}}
     & \textbf{1\vphantom{gy}} & \textbf{2\vphantom{gy}} & \textbf{4\vphantom{gy}} & \textbf{8\vphantom{gy}} \\
    \hline
    \textbf{1}  & 0.260 & 0.286 & 0.335 & 0.484 & 0.285 & 0.329 & 0.394 & 0.488 \\
    \textbf{2}  & 0.452 & 0.462 & 0.511 & 0.535 & 0.721 & 0.729 & 0.749 & 0.766 \\
    \textbf{4}  & 0.751 & 0.807 & 0.797 & 0.763 & 0.897 & 0.920 & 0.925 & 0.915 \\
    \textbf{8}  & 0.847 & 0.870 & 0.867 & \textbf{0.871} & 0.929 & 0.932 & \textbf{0.936} & 0.928 \\
    \textbf{16} & 0.863 & 0.866 & 0.865 & 0.869 & 0.914 & 0.913 & 0.922 & 0.927 \\
    \noalign{\hrule height 1.1pt} 
  \end{tabular}
\end{table*}

\vspace{0.5cm}

\newpage
\begin{table*}[h]

  \caption{$R^2$ on the QP Function-Approximation Task ($n{=}10$, $m{=}3$).}
  \centering
  
  \label{tab:qp_r2_n10m3_appendix}
  \small
  \setlength{\tabcolsep}{6pt}
  \renewcommand{\arraystretch}{1.15}
  \begin{tabular}{c | cccc | cccc}
    \noalign{\hrule height 1.1pt}
    & \multicolumn{4}{c|}{\textbf{SoftmaxTransformer}} & \multicolumn{4}{c}{\textbf{LinearTransformer}} \\
    \cline{2-5}\cline{6-9}
    \textbf{Number of Layers} & \multicolumn{4}{c|}{\textbf{Number of Heads}} & \multicolumn{4}{c}{\textbf{Number of Heads}} \\
    \cline{2-5}\cline{6-9}
     & \textbf{1\vphantom{gy}} & \textbf{2\vphantom{gy}} & \textbf{4\vphantom{gy}} & \textbf{8\vphantom{gy}}
     & \textbf{1\vphantom{gy}} & \textbf{2\vphantom{gy}} & \textbf{4\vphantom{gy}} & \textbf{8\vphantom{gy}} \\
    \hline
    \textbf{1}  & 0.186 & 0.195 & 0.213 & 0.242 & 0.191 & 0.221 & 0.261 & 0.298 \\
    \textbf{2}  & 0.336 & 0.361 & 0.357 & 0.357 & 0.645 & 0.674 & 0.669 & 0.656 \\
    \textbf{4}  & 0.666 & 0.657 & 0.670 & 0.661 & 0.867 & 0.893 & 0.896 & 0.879 \\
    \textbf{8}  & 0.785 & 0.813 & 0.842 & 0.799 & 0.900 & \textbf{0.905} & 0.893 & 0.899 \\
    \textbf{16} & 0.774 & 0.786 & 0.822 & \textbf{0.842} & 0.600 & 0.860 & 0.884 & 0.873 \\
    \noalign{\hrule height 1.1pt} 
  \end{tabular}
\end{table*}

\vspace{1cm}

\begin{table*}[h]

  \caption{$R^2$ on the QP Function-Approximation Task ($n{=}10$, $m{=}6$).}
  \centering
  
  \label{tab:qp_r2_n10m6_appendix}
  \small
  \setlength{\tabcolsep}{6pt}
  \renewcommand{\arraystretch}{1.15}
  \begin{tabular}{c | cccc | cccc}
    \noalign{\hrule height 1.1pt}
    & \multicolumn{4}{c|}{\textbf{SoftmaxTransformer}} & \multicolumn{4}{c}{\textbf{LinearTransformer}} \\
    \cline{2-5}\cline{6-9}
    \textbf{Number of Layers} & \multicolumn{4}{c|}{\textbf{Number of Heads}} & \multicolumn{4}{c}{\textbf{Number of Heads}} \\
    \cline{2-5}\cline{6-9}
     & \textbf{1\vphantom{gy}} & \textbf{2\vphantom{gy}} & \textbf{4\vphantom{gy}} & \textbf{8\vphantom{gy}}
     & \textbf{1\vphantom{gy}} & \textbf{2\vphantom{gy}} & \textbf{4\vphantom{gy}} & \textbf{8\vphantom{gy}} \\
    \hline
    \textbf{1}  & 0.207 & 0.215 & 0.230 & 0.258 & 0.212 & 0.242 & 0.270 & 0.320 \\
    \textbf{2}  & 0.341 & 0.342 & 0.363 & 0.364 & 0.554 & 0.665 & 0.667 & 0.641 \\
    \textbf{4}  & 0.603 & 0.663 & 0.651 & 0.642 & 0.816 & 0.868 & 0.852 & 0.856 \\
    \textbf{8}  & 0.792 & 0.799 & 0.786 & 0.796 & \textbf{0.892} & 0.889 & 0.869 & 0.873 \\
    \textbf{16} & 0.795 & 0.797 & 0.783 & \textbf{0.816} & 0.614 & 0.843 & 0.855 & 0.864 \\
    \noalign{\hrule height 1.1pt} 
  \end{tabular}
\end{table*}

\vspace{1cm}

\begin{table*}[h!]

  \caption{$R^2$ on the QP Function-Approximation Task ($n{=}10$, $m{=}9$).}
  \centering
  
  \label{tab:qp_r2_n10m9_appendix}
  \small
  \setlength{\tabcolsep}{6pt}
  \renewcommand{\arraystretch}{1.15}
  \begin{tabular}{c | cccc | cccc}
    \noalign{\hrule height 1.1pt}
    & \multicolumn{4}{c|}{\textbf{SoftmaxTransformer}} & \multicolumn{4}{c}{\textbf{LinearTransformer}} \\
    \cline{2-5}\cline{6-9}
    \textbf{Number of Layers} & \multicolumn{4}{c|}{\textbf{Number of Heads}} & \multicolumn{4}{c}{\textbf{Number of Heads}} \\
    \cline{2-5}\cline{6-9}
     & \textbf{1\vphantom{gy}} & \textbf{2\vphantom{gy}} & \textbf{4\vphantom{gy}} & \textbf{8\vphantom{gy}}
     & \textbf{1\vphantom{gy}} & \textbf{2\vphantom{gy}} & \textbf{4\vphantom{gy}} & \textbf{8\vphantom{gy}} \\
    \hline
    \textbf{1}  & 0.247 & 0.256 & 0.273 & 0.309 & 0.251 & 0.279 & 0.317 & 0.371 \\
    \textbf{2}  & 0.382 & 0.397 & 0.505 & 0.485 & 0.645 & 0.683 & 0.679 & 0.667 \\
    \textbf{4}  & 0.626 & 0.659 & 0.755 & 0.746 & 0.850 & 0.880 & 0.892 & 0.882 \\
    \textbf{8}  & 0.800 & 0.833 & 0.825 & 0.838 & 0.900 & \textbf{0.906} & 0.895 & 0.896 \\
    \textbf{16} & 0.828 & 0.826 & 0.824 & \textbf{0.843} & 0.762 & 0.864 & 0.890 & 0.900 \\
    \noalign{\hrule height 1.1pt} 
  \end{tabular}
\end{table*}

\newpage

\begin{table*}[h!]

  \caption{$R^2$ on the QP Function-Approximation Task ($n{=}15$, $m{=}10$).}
  \centering
  
  \label{tab:qp_r2_n15m10_appendix}
  \small
  \setlength{\tabcolsep}{6pt}
  \renewcommand{\arraystretch}{1.15}
  \begin{tabular}{c | cccc | cccc}
    \noalign{\hrule height 1.1pt}
    & \multicolumn{4}{c|}{\textbf{SoftmaxTransformer}} & \multicolumn{4}{c}{\textbf{LinearTransformer}} \\
    \cline{2-5}\cline{6-9}
    \textbf{Number of Layers} & \multicolumn{4}{c|}{\textbf{Number of Heads}} & \multicolumn{4}{c}{\textbf{Number of Heads}} \\
    \cline{2-5}\cline{6-9}
     & \textbf{1\vphantom{gy}} & \textbf{2\vphantom{gy}} & \textbf{4\vphantom{gy}} & \textbf{8\vphantom{gy}}
     & \textbf{1\vphantom{gy}} & \textbf{2\vphantom{gy}} & \textbf{4\vphantom{gy}} & \textbf{8\vphantom{gy}} \\
    \hline
    \textbf{1}  & 0.198 & 0.208 & 0.213 & 0.226 & 0.210 & 0.219 & 0.236 & 0.259 \\
    \textbf{2}  & 0.306 & 0.325 & 0.338 & 0.333 & 0.608 & 0.636 & 0.639 & 0.628 \\
    \textbf{4}  & 0.636 & 0.706 & 0.739 & 0.753 & 0.861 & 0.895 & 0.889 & 0.876 \\
    \textbf{8}  & 0.767 & 0.815 & \textbf{0.824} & 0.820 & 0.854 & 0.902 & \textbf{0.911} & 0.911 \\
    \noalign{\hrule height 1.1pt} 
  \end{tabular}
\end{table*}

\vspace{0.5cm}

\begin{table*}[h!]

  \caption{$R^2$ on the QP Function-Approximation Task ($n{=}20$, $m{=}8$).}
  \centering
  
  \label{tab:qp_r2_n20m8_appendix}
  \small
  \setlength{\tabcolsep}{6pt}
  \renewcommand{\arraystretch}{1.15}
  \begin{tabular}{c | cccc | cccc}
    \noalign{\hrule height 1.1pt}
    & \multicolumn{4}{c|}{\textbf{SoftmaxTransformer}} & \multicolumn{4}{c}{\textbf{LinearTransformer}} \\
    \cline{2-5}\cline{6-9}
    \textbf{Number of Layers} & \multicolumn{4}{c|}{\textbf{Number of Heads}} & \multicolumn{4}{c}{\textbf{Number of Heads}} \\
    \cline{2-5}\cline{6-9}
     & \textbf{1\vphantom{gy}} & \textbf{2\vphantom{gy}} & \textbf{4\vphantom{gy}} & \textbf{8\vphantom{gy}}
     & \textbf{1\vphantom{gy}} & \textbf{2\vphantom{gy}} & \textbf{4\vphantom{gy}} & \textbf{8\vphantom{gy}} \\
    \hline
    \textbf{1}  & 0.150 & 0.156 & 0.155 & 0.158 & 0.161 & 0.158 & 0.157 & 0.155 \\
    \textbf{2}  & 0.208 & 0.201 & 0.199 & 0.196 & 0.462 & 0.561 & 0.540 & 0.537 \\
    \textbf{4}  & 0.351 & 0.360 & 0.370 & 0.297 & 0.821 & 0.828 & 0.754 & 0.793 \\
    \textbf{8}  & 0.510 & \textbf{0.523} & 0.469 & 0.466 & NaN & \textbf{0.846} & 0.842 & 0.828 \\
    \noalign{\hrule height 1.1pt} 
  \end{tabular}
\end{table*}

\vspace{0.5cm}

\begin{table*}[h!]

  \caption{$R^2$ on the QP Function-Approximation Task ($n{=}20$, $m{=}16$).}
  \centering
  
  \label{tab:qp_r2_n20m16_appendix}
  \small
  \setlength{\tabcolsep}{6pt}
  \renewcommand{\arraystretch}{1.15}
  \begin{tabular}{c | cccc | cccc}
    \noalign{\hrule height 1.1pt}
    & \multicolumn{4}{c|}{\textbf{SoftmaxTransformer}} & \multicolumn{4}{c}{\textbf{LinearTransformer}} \\
    \cline{2-5}\cline{6-9}
    \textbf{Number of Layers} & \multicolumn{4}{c|}{\textbf{Number of Heads}} & \multicolumn{4}{c}{\textbf{Number of Heads}} \\
    \cline{2-5}\cline{6-9}
     & \textbf{1\vphantom{gy}} & \textbf{2\vphantom{gy}} & \textbf{4\vphantom{gy}} & \textbf{8\vphantom{gy}}
     & \textbf{1\vphantom{gy}} & \textbf{2\vphantom{gy}} & \textbf{4\vphantom{gy}} & \textbf{8\vphantom{gy}} \\
    \hline
    \textbf{1}  & 0.211 & 0.220 & 0.222 & 0.229 & 0.218 & 0.226 & 0.224 & 0.228 \\
    \textbf{2}  & 0.305 & 0.342 & 0.332 & 0.342 & 0.616 & 0.637 & 0.634 & 0.532 \\
    \textbf{4}  & 0.599 & 0.671 & 0.720 & 0.732 & \textbf{0.887} & 0.877 & 0.849 & 0.868 \\
    \textbf{8}  & 0.750 & \textbf{0.800} & 0.764 & 0.795 & NaN & NaN & 0.347 & 0.883 \\
    \noalign{\hrule height 1.1pt} 
  \end{tabular}
\end{table*}

\vspace{0.5cm}

\paragraph{Additional Error Distributions.}
Figures~\ref{fig:qp_lin_best_distros_n10m3} through \ref{fig:qp_softmax_best_distros_n10m9} present error distribution analyses for the best-performing configurations in higher-dimensional challenging QP problems ($n=10)$. These distributions complement performance metrics $R^2$ by providing detailed information on the error characteristics of our transformer-based solvers. The figures show NMSE distributions, constraint violation patterns, and objective error distributions, revealing how both LinearTransformer and SoftmaxTransformer architectures handle the increased complexity of larger-scale QP instances.

\clearpage
\begin{figure} [H]
  \centering
  \begin{subfigure}{0.32\textwidth}
    \includegraphics[width=\linewidth]{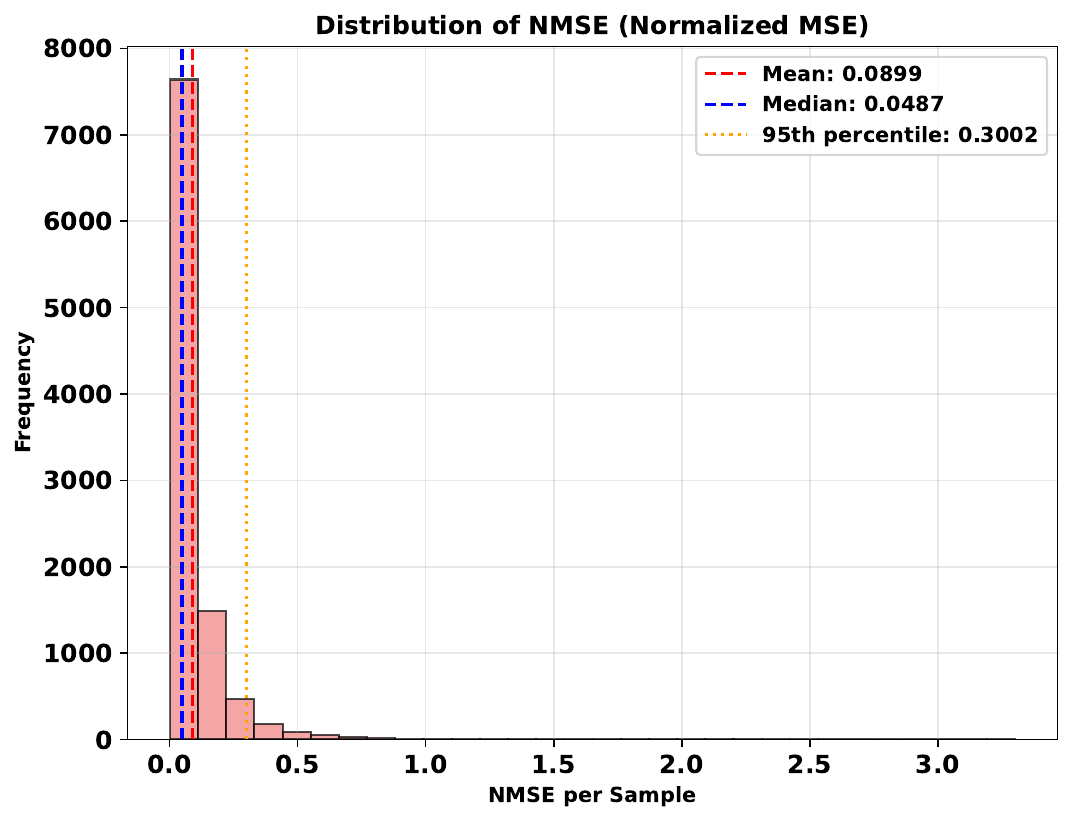}
    \caption{NMSE distribution}
  \end{subfigure}\hfill
  \begin{subfigure}{0.32\textwidth}
    \includegraphics[width=\linewidth]{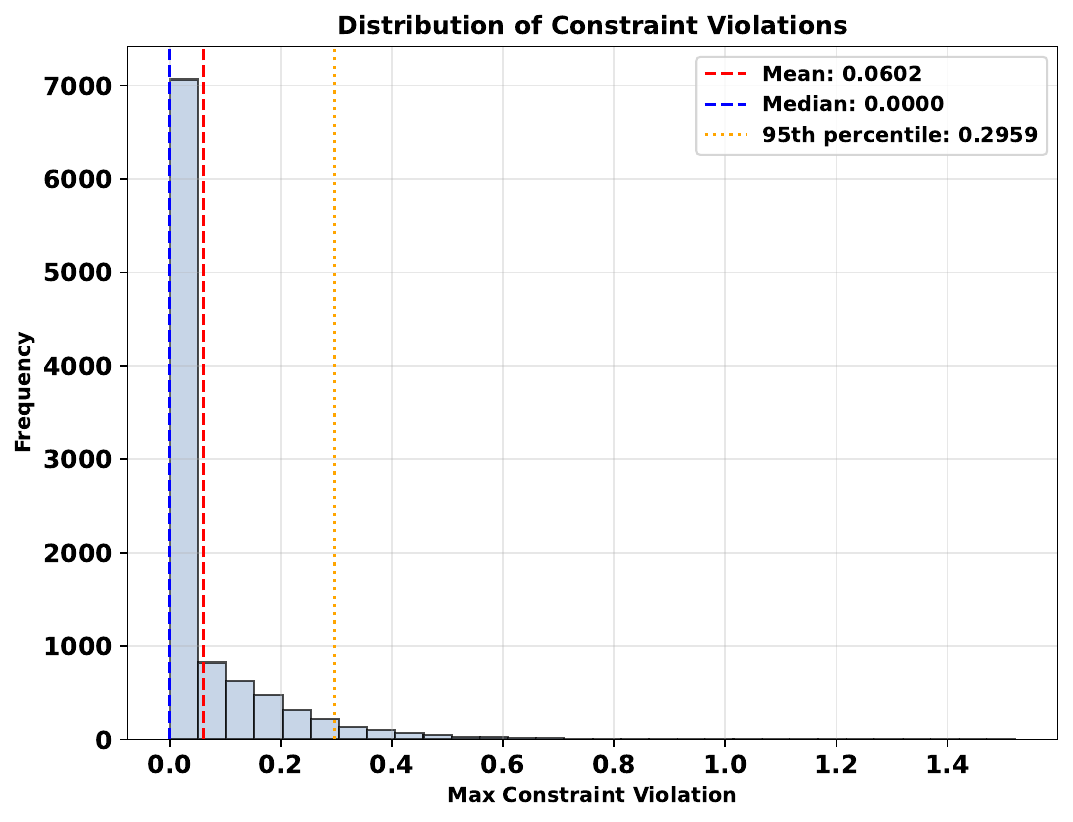}
    \caption{Constraint violation}
  \end{subfigure}\hfill
  \begin{subfigure}{0.32\textwidth}
    \includegraphics[width=\linewidth]{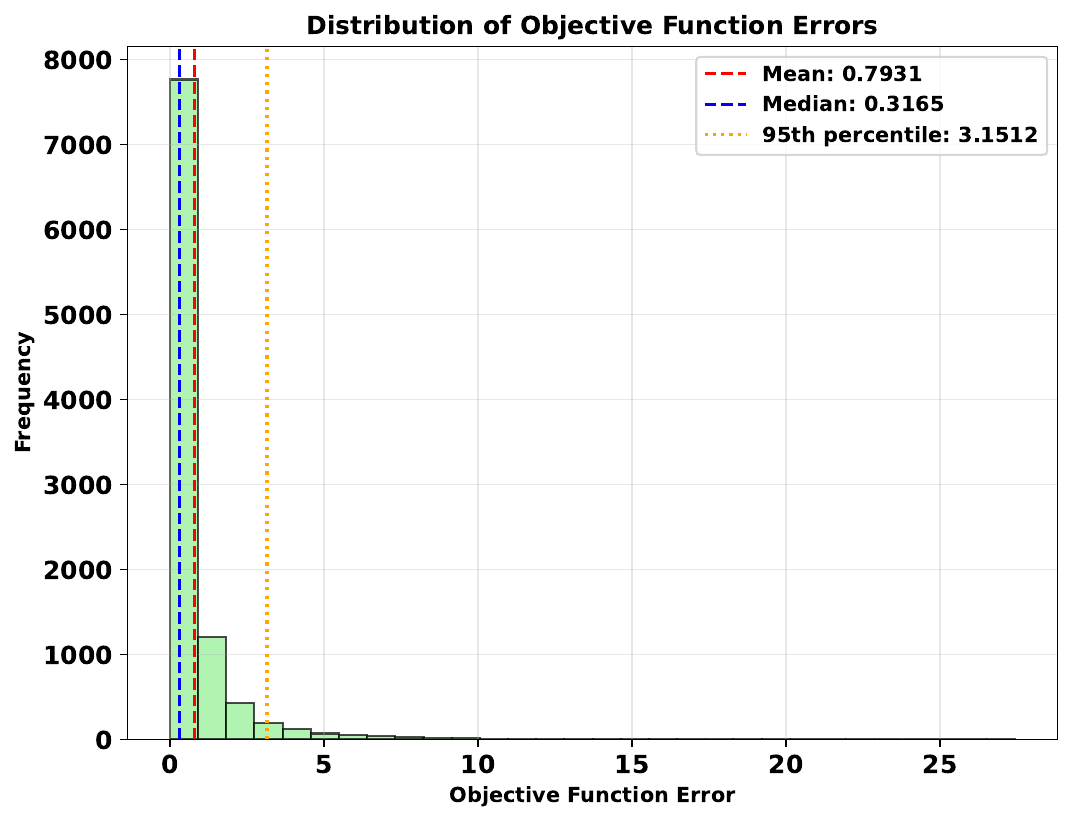}
    \caption{Objective error}
  \end{subfigure}
  \caption{Error distributions for the best LinearTransformer configuration ($n=10, m=3$, layers=8, heads=2).}
  \label{fig:qp_lin_best_distros_n10m3}
\end{figure}

\begin{figure} [H]
  \centering
  \begin{subfigure}{0.32\textwidth}
    \includegraphics[width=\linewidth]{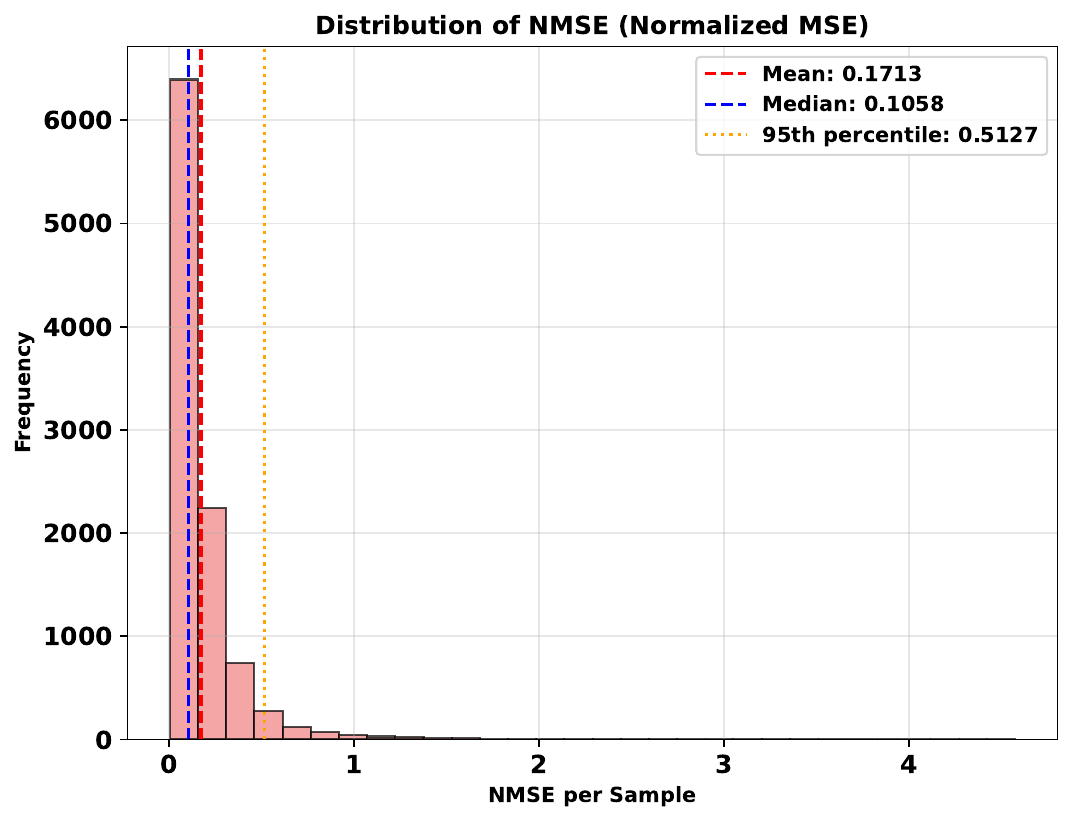}
    \caption{NMSE distribution}
  \end{subfigure}\hfill
  \begin{subfigure}{0.32\textwidth}
    \includegraphics[width=\linewidth]{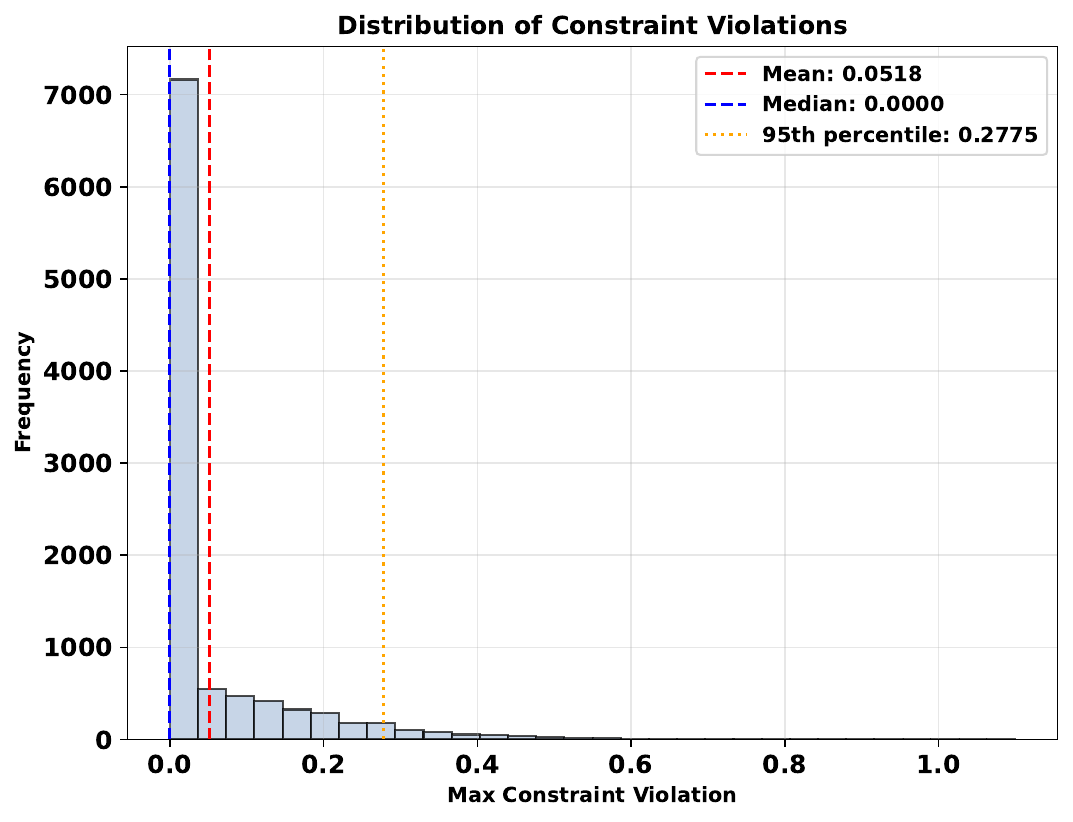}
    \caption{Constraint violation}
  \end{subfigure}\hfill
  \begin{subfigure}{0.32\textwidth}
    \includegraphics[width=\linewidth]{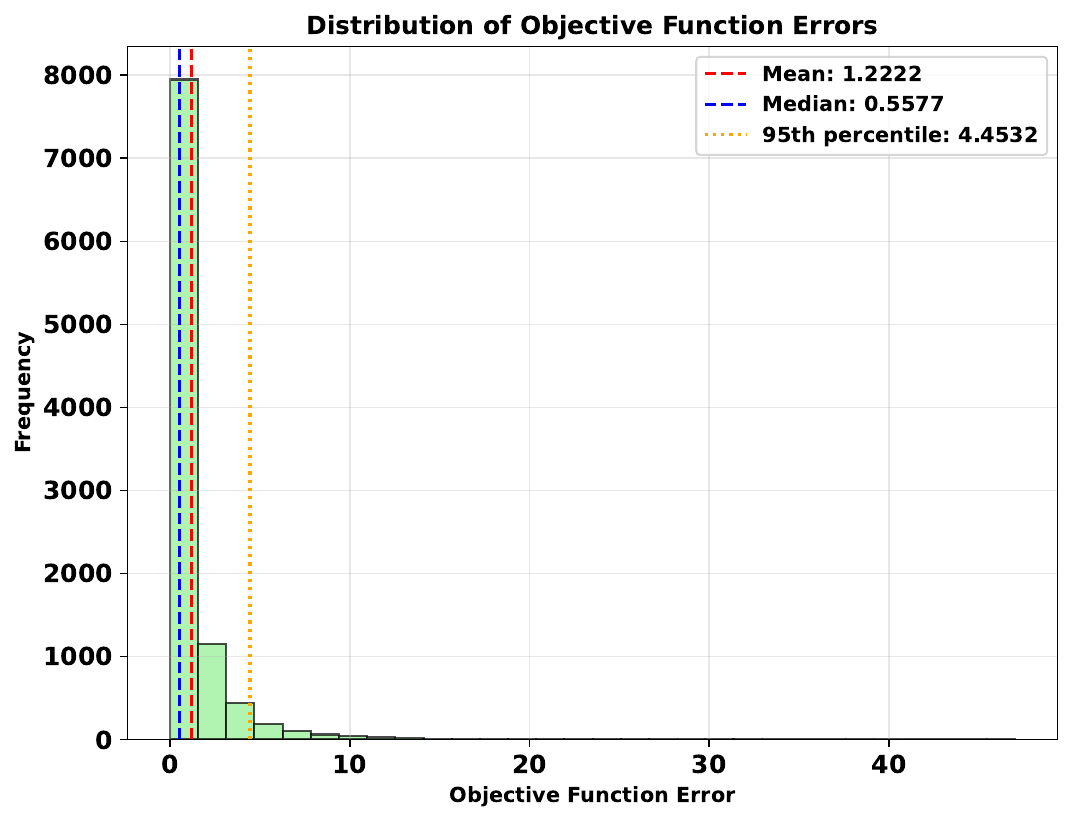}
    \caption{Objective error}
  \end{subfigure}
  \caption{Error distributions for the best SoftmaxTransformer configuration ($n=10, m=3$, layers=16, heads=8).}
  \label{fig:qp_softmax_best_distros_n10m3}
\end{figure}

\begin{figure} [H]
  \centering
  \begin{subfigure}{0.32\textwidth}
    \includegraphics[width=\linewidth]{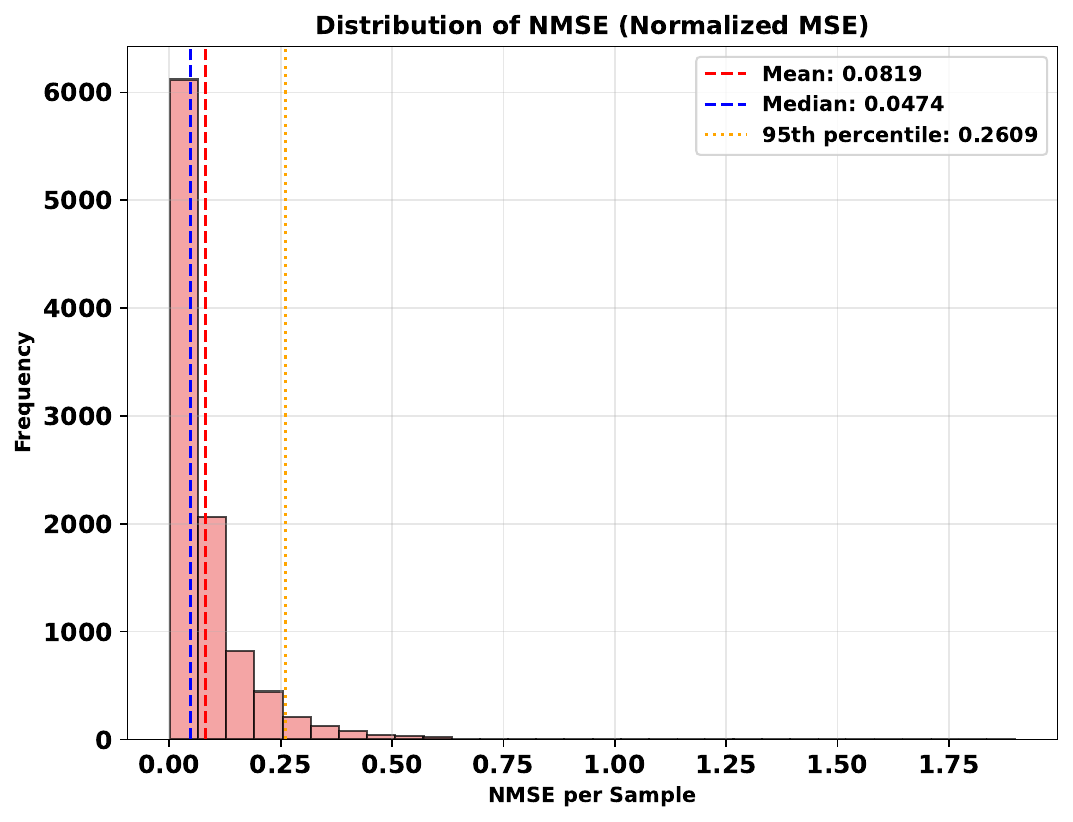}
    \caption{NMSE distribution}
  \end{subfigure}\hfill
  \begin{subfigure}{0.32\textwidth}
    \includegraphics[width=\linewidth]{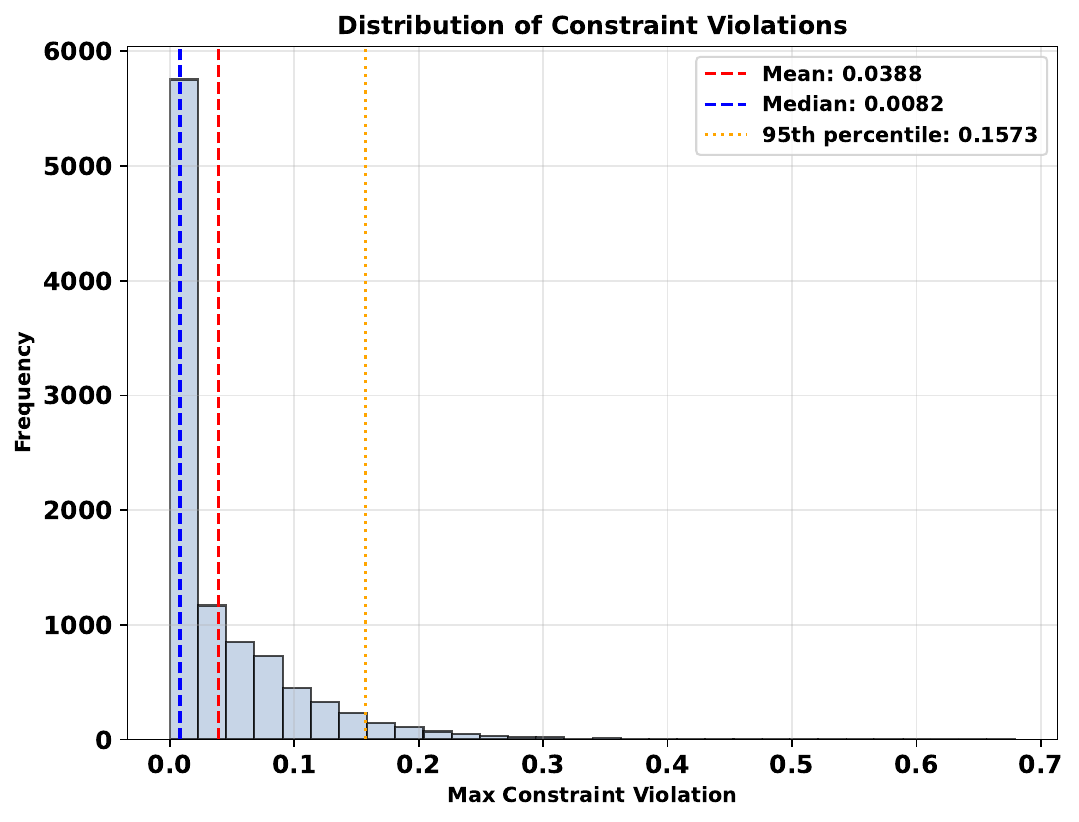}
    \caption{Constraint violation}
  \end{subfigure}\hfill
  \begin{subfigure}{0.32\textwidth}
    \includegraphics[width=\linewidth]{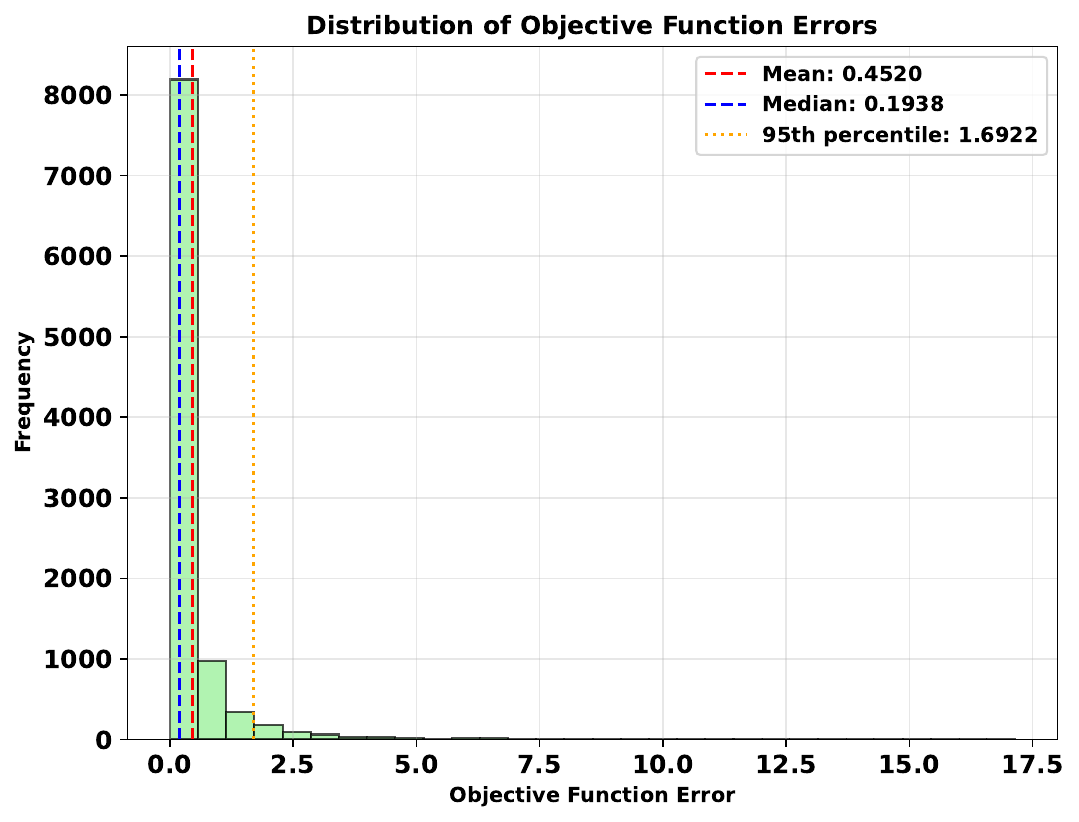}
    \caption{Objective error}
  \end{subfigure}
  \caption{Error distributions for the best LinearTransformer configuration ($n=10, m=6$, layers=8, heads=1).}
  \label{fig:qp_lin_best_distros_n10m6}
\end{figure}

\clearpage

\begin{figure} [H]
  \centering
  \begin{subfigure}{0.32\textwidth}
    \includegraphics[width=\linewidth]{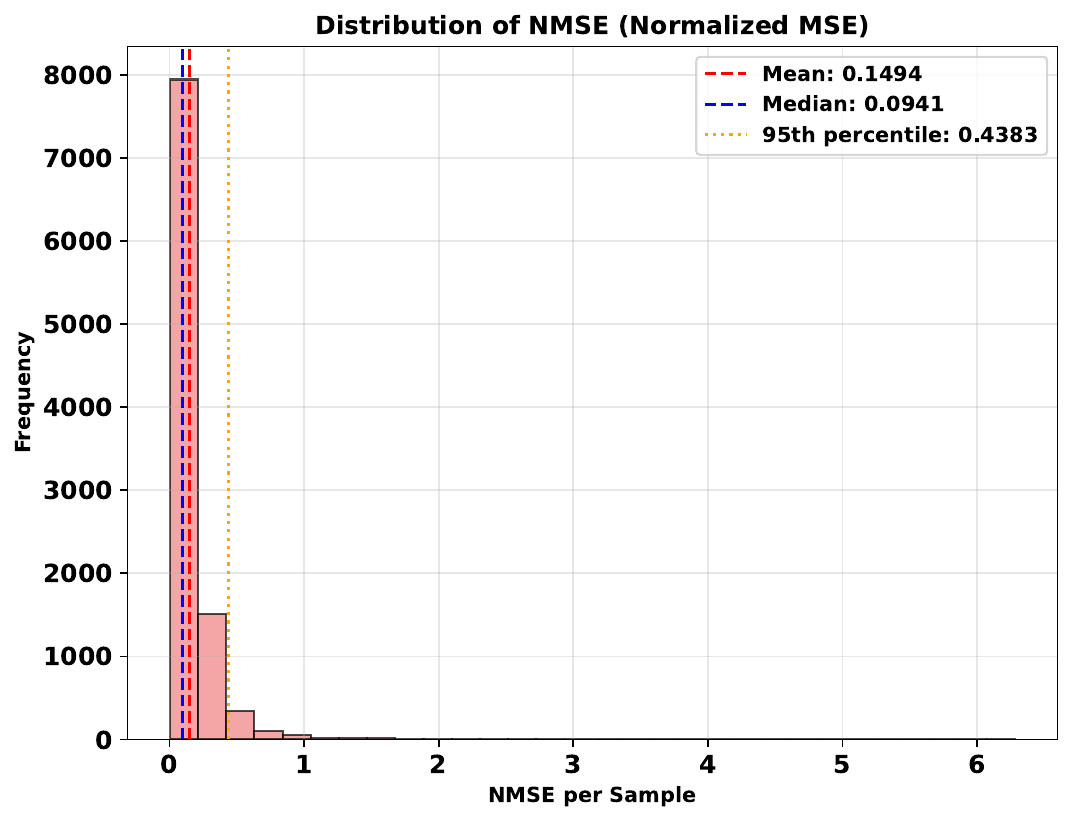}
    \caption{NMSE distribution}
  \end{subfigure}\hfill
  \begin{subfigure}{0.32\textwidth}
    \includegraphics[width=\linewidth]{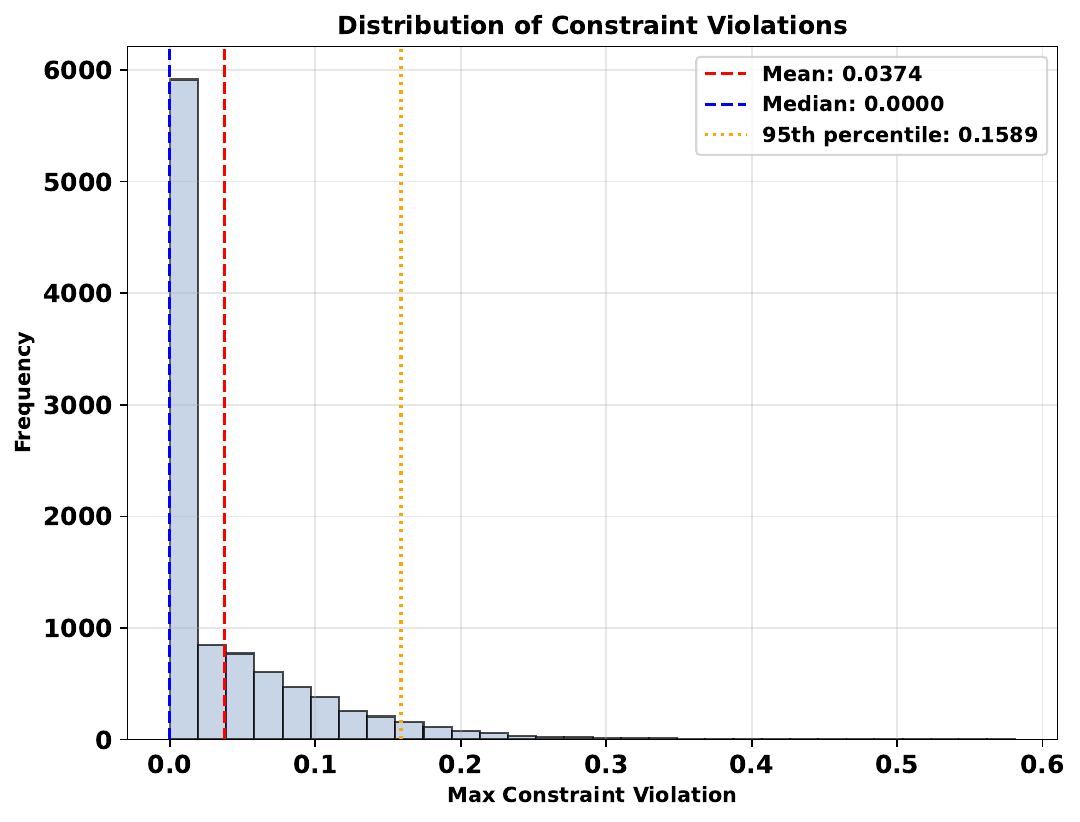}
    \caption{Constraint violation}
  \end{subfigure}\hfill
  \begin{subfigure}{0.32\textwidth}
    \includegraphics[width=\linewidth]{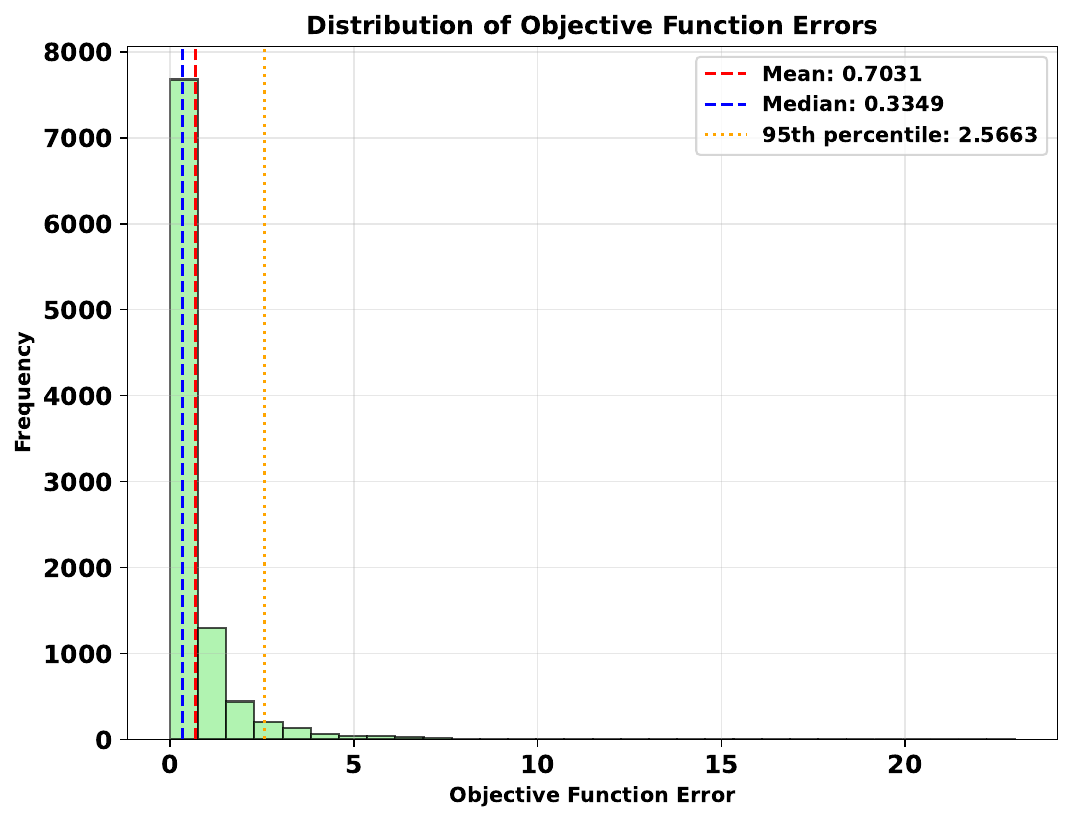}
    \caption{Objective error}
  \end{subfigure}
  \caption{Error distributions for the best SoftmaxTransformer configuration ($n=10, m=6$, layers=16, heads=8).}
  \label{fig:qp_softmax_best_distros_n10m6}
\end{figure}

\begin{figure} [H]
  \centering
  \begin{subfigure}{0.32\textwidth}
    \includegraphics[width=\linewidth]{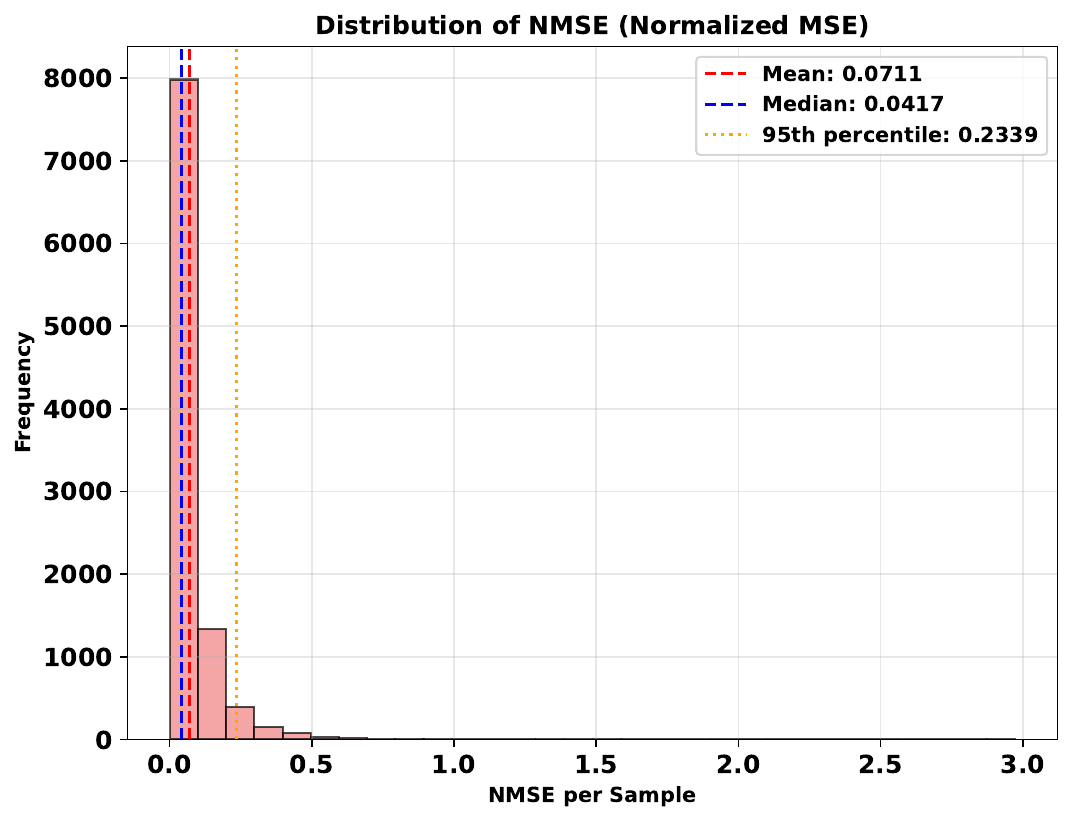}
    \caption{NMSE distribution}
  \end{subfigure}\hfill
  \begin{subfigure}{0.32\textwidth}
    \includegraphics[width=\linewidth]{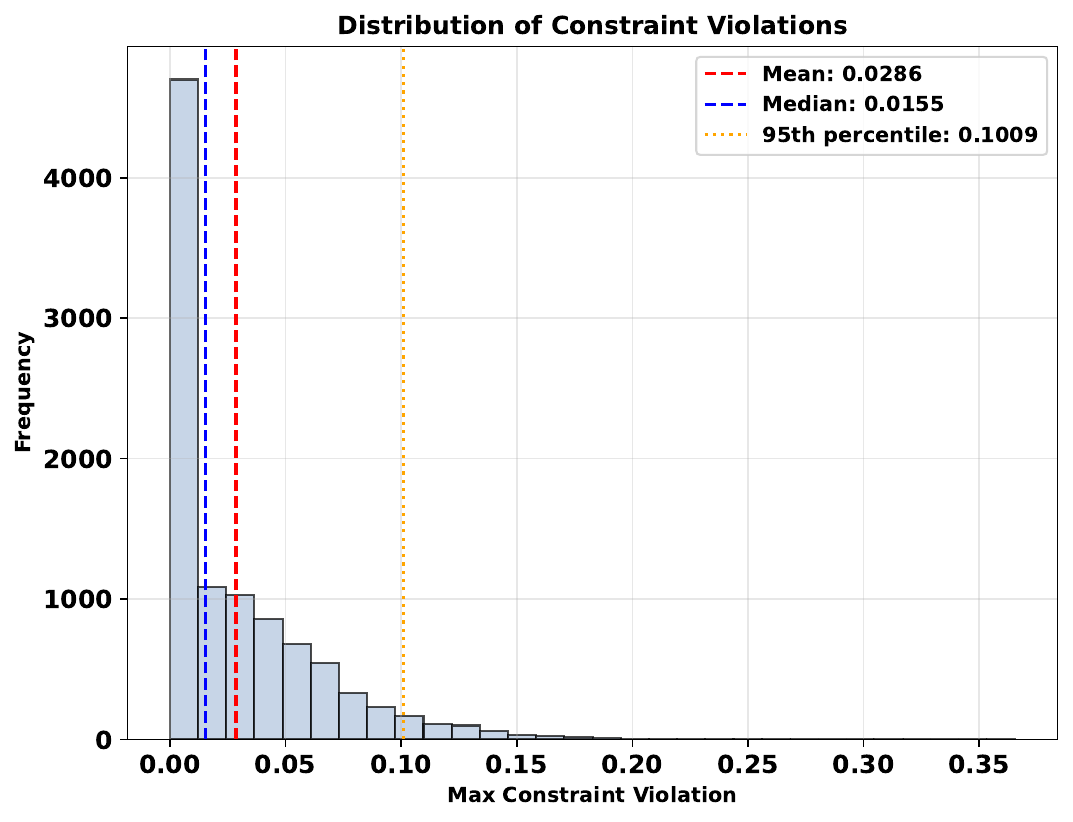}
    \caption{Constraint violation}
  \end{subfigure}\hfill
  \begin{subfigure}{0.32\textwidth}
    \includegraphics[width=\linewidth]{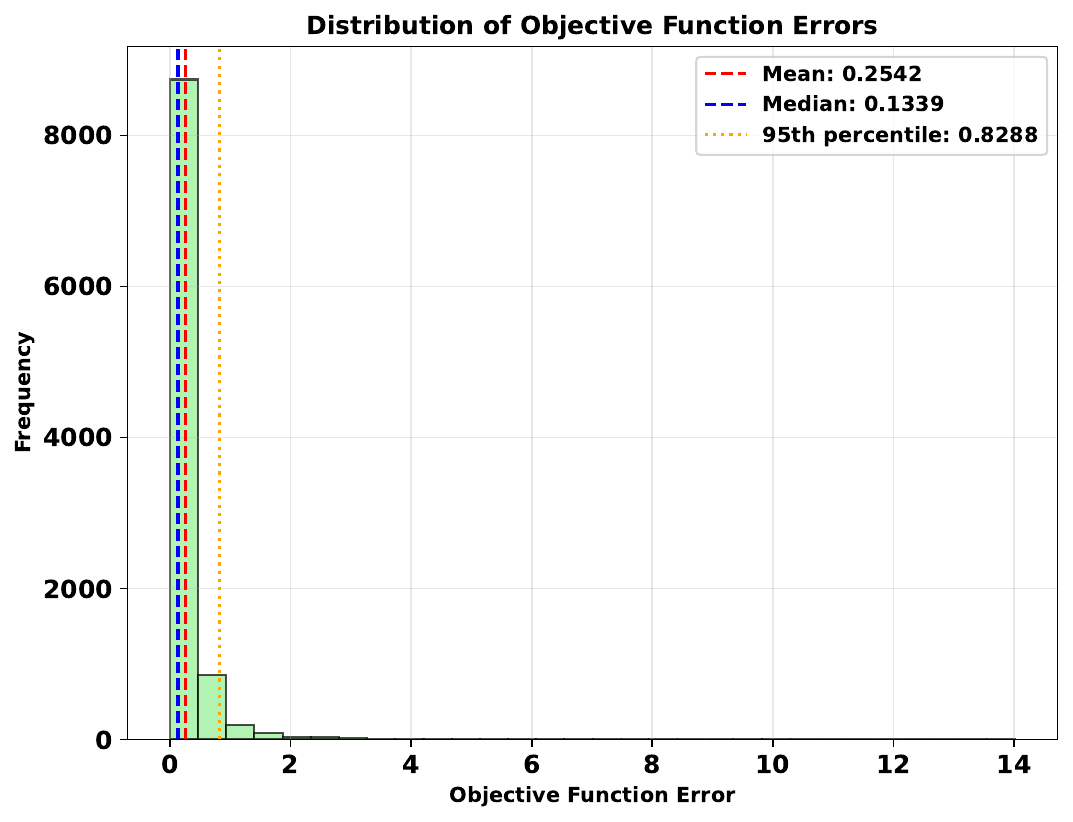}
    \caption{Objective error}
  \end{subfigure}
  \caption{Error distributions for the best LinearTransformer configuration ($n=10, m=9$, layers=8, heads=2).}
  \label{fig:qp_lin_best_distros_n10m9}
\end{figure}

\begin{figure} [H]
  \centering
  \begin{subfigure}{0.32\textwidth}
    \includegraphics[width=\linewidth]{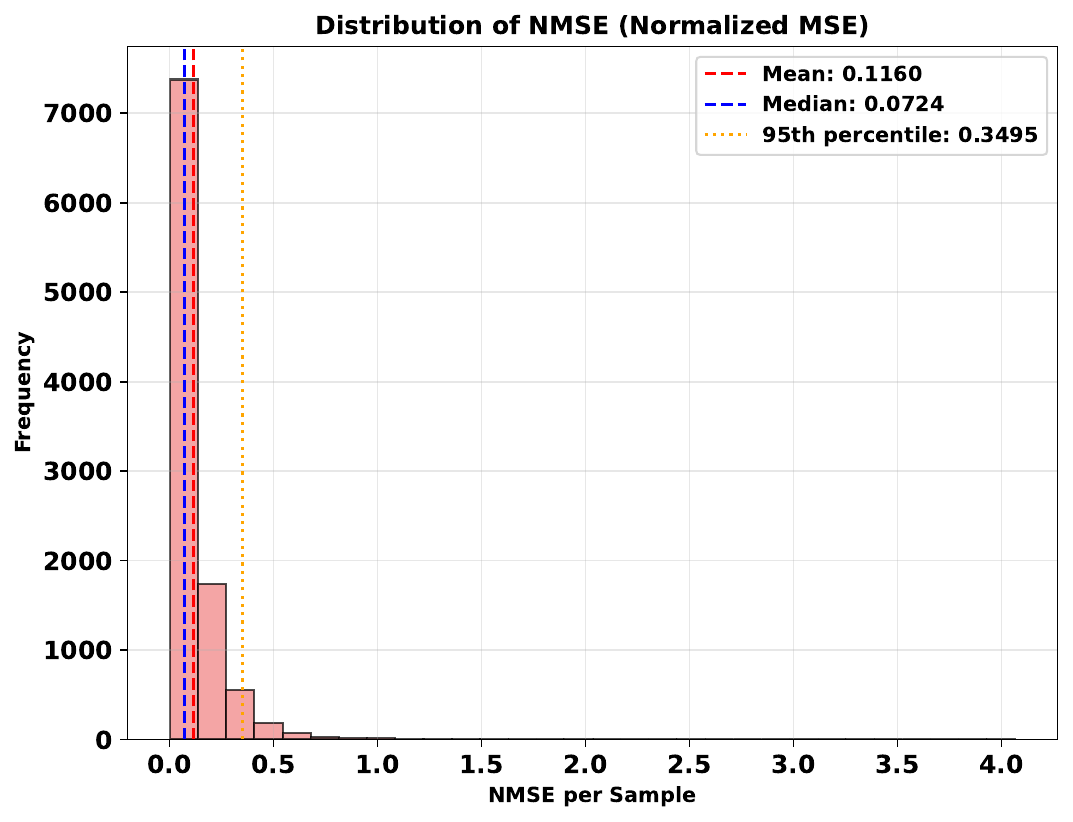}
    \caption{NMSE distribution}
  \end{subfigure}\hfill
  \begin{subfigure}{0.32\textwidth}
    \includegraphics[width=\linewidth]{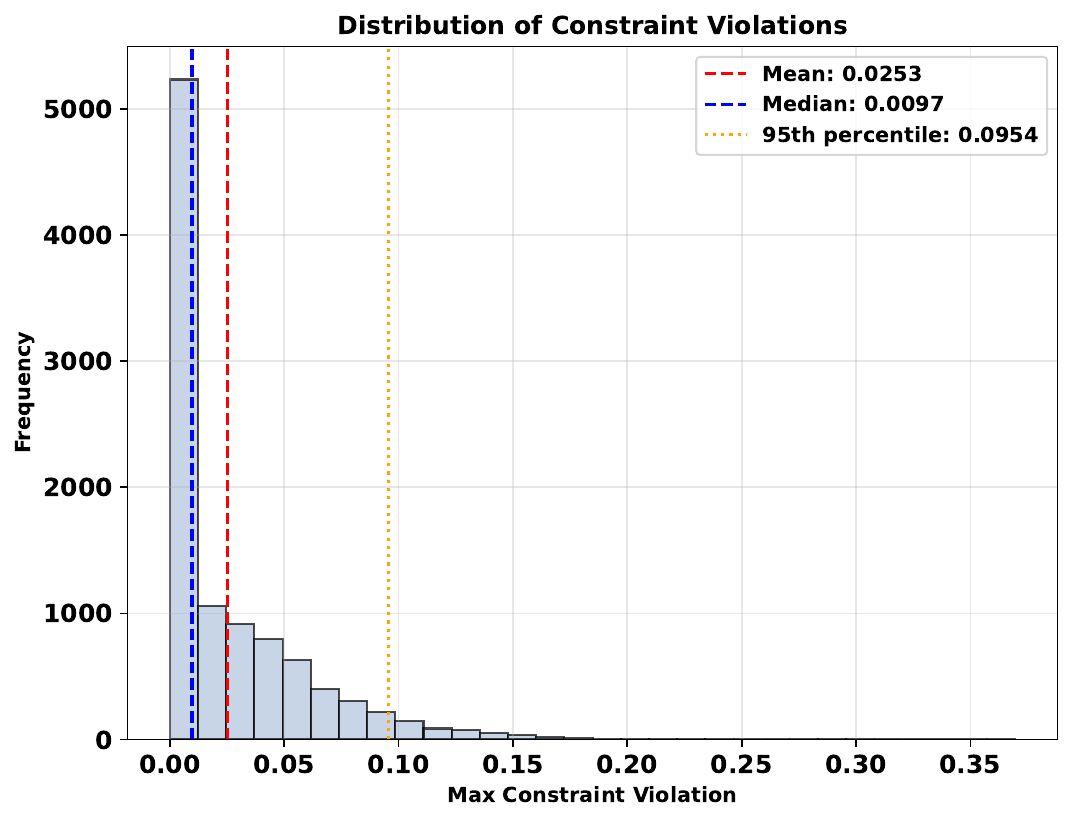}
    \caption{Constraint violation}
  \end{subfigure}\hfill
  \begin{subfigure}{0.32\textwidth}
    \includegraphics[width=\linewidth]{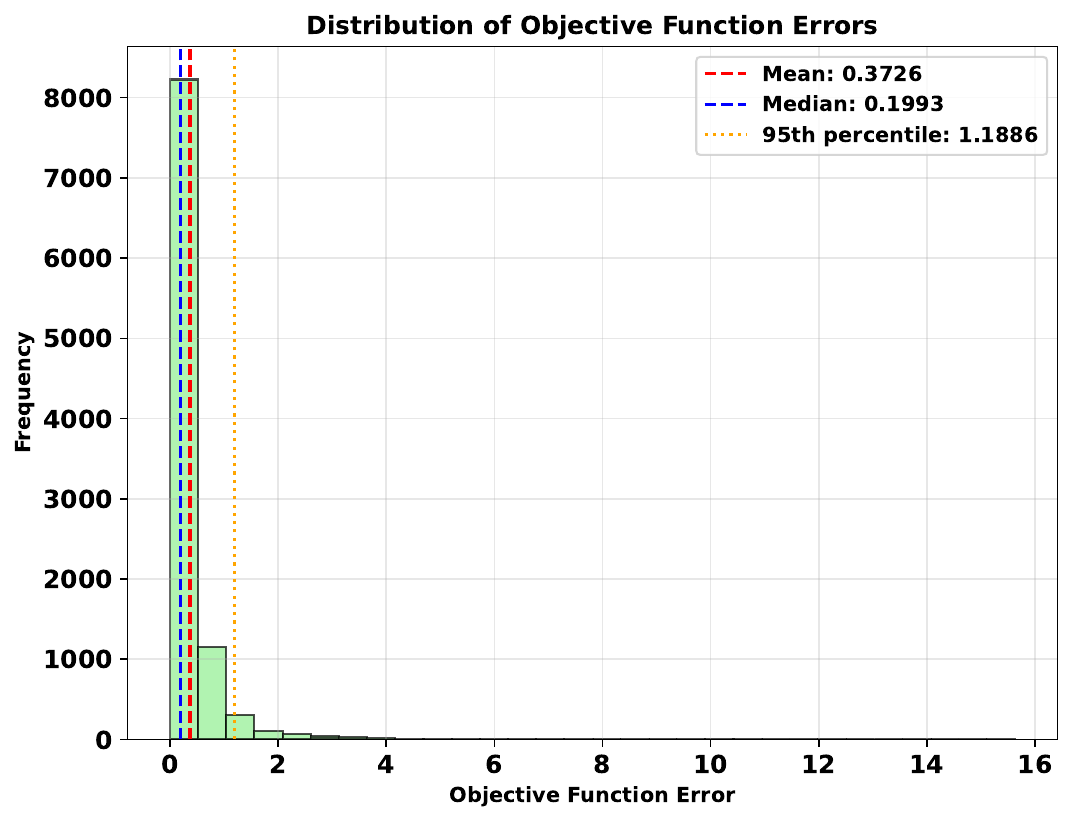}
    \caption{Objective error}
  \end{subfigure}
  \caption{Error distributions for the best SoftmaxTransformer configuration ($n=10, m=9$, layers=16, heads=8).}
  \label{fig:qp_softmax_best_distros_n10m9}
\end{figure}

\clearpage

\paragraph{Robustness to distribution shift.}
\label{sec:distribution_shift}
To investigate the robustness of our models, we conduct experiments where the training and test distributions differ in the condition number of the QP problems. The condition number, $\kappa(A) = \lambda_{\max}(A)/\lambda_{\min}(A)$, controls the numerical difficulty of the optimization, with higher values corresponding to more ill-conditioned problems that are harder to solve. We generate this distribution shift by controlling the eigenvalue spread of the $A$ matrix using a scaling parameter $\kappa$, which creates exponentially more difficult problems for higher $\kappa$ values. To isolate this effect, we normalize the Frobenius norm of $A$ to match the baseline distribution.

\begin{table}[htbp]
  \centering
  \caption{R$^2$ Performance under distribution shift in condition number $\kappa$.}
  \label{tab:qp_distribution_shift}
  \small
  \renewcommand{\arraystretch}{1.15}
  \begin{tabular}{c|cc|cc}
    \noalign{\hrule height 1.1pt}
    \multirow{2}{*}{\textbf{$\kappa$ Range}} 
      & \multicolumn{2}{c|}{\textbf{LinearTransformer}} 
      & \multicolumn{2}{c}{\textbf{SoftmaxTransformer}} \\
    \cline{2-3}\cline{4-5}
      & \textbf{R$^2$} & \textbf{Drop} & \textbf{R$^2$} & \textbf{Drop} \\
    \hline
    \textbf{No $\kappa$} & 0.973832 & --      & 0.935439 & --      \\
    \hline
    \textbf{1.2--2} & 0.971755 & -0.21\% & 0.930962 & -0.48\% \\
    \hline
    \textbf{2--5}   & 0.968110 & -0.59\% & 0.924842 & -1.13\% \\
    \hline
    \textbf{5--10}  & 0.960139 & -1.41\% & 0.909774 & -2.74\% \\
    \hline
    \textbf{10--20} & 0.951526 & -2.29\% & 0.888270 & -5.04\% \\
    \noalign{\hrule height 1.1pt}
  \end{tabular}
\end{table}

The results of this analysis are shown in Table~\ref{tab:qp_distribution_shift}. We trained the best-performing models (LinearTransformer with 8 layers/2 heads and SoftmaxTransformer with 16 layers/2 heads) on the baseline distribution and tested them on shifted distributions with varying $\kappa$ ranges. Both architectures demonstrate considerable robustness, indicating they can generalize reasonably well to more challenging numerical conditions. Again, the LinearTransformer consistently shows superior performance and a more graceful degradation across all $\kappa$ ranges as the problem difficulty increases.

\section{Toy Model: Statistical Suboptimality of Predict-then-Optimize}

\label{sec:app_portfolio}

\paragraph{Problem Setup.}
Consider two assets with random returns $r_1, r_2$ having unknown means $\mu_1, \mu_2$. Let $w \in [0,1]$ be the weight on asset 1. Define the mean gap $\delta := \mu_1 - \mu_2$.
For simplicity, we assume the assets are uncorrelated with equal variance. Consequently, the portfolio variance is proportional to $w^2 + (1-w)^2$, i.e.,
\[
    \mathrm{Var}\!\big(w r_1 + (1-w) r_2\big) \;=\; \sigma^2\big(w^2 + (1-w)^2\big).
\]
Without loss of generality, we absorb the scale factor $\sigma^2$ into the risk-aversion parameter by redefining $\lambda \leftarrow \lambda \sigma^2$; thus $\lambda$ below should be interpreted as the coefficient on the variance term. We adopt a standard mean-variance utility:
\[
    U(w; \delta) \;:=\; \mu_2 + w\delta - \lambda\big(w^2 + (1-w)^2\big), \quad \lambda > 0.
\]
By completing the square for the quadratic penalty, maximizing the utility is equivalent to minimizing the distance between $w$ and a shifted target. Letting $u := w - 1/2$, the objective simplifies to:
\[
    \max_{u \in [-1/2, 1/2]} \quad u\delta - 2\lambda u^2.
\]
Because the toy-model utility is quadratic in $w$ with
curvature $2\lambda$, the utility suboptimality satisfies
$U(w^\star;\delta)-U(w;\delta)=2\lambda\,(w-w^\star)^2$,
so MSE to the oracle allocation is directly proportional
to the utility gap. This justifies MSE-to-oracle as the
evaluation metric.

\paragraph{Assumptions.}
This toy model makes several simplifying assumptions to enable closed-form insights.
First, we consider only two assets with uncorrelated returns and equal variances, so the portfolio variance simplifies to
$\sigma^2\!\left(w^2+(1-w)^2\right)$.
Second, we use a mean--variance objective with fixed risk aversion $\lambda>0$ and a box constraint $w\in[0,1]$, which induces the clipped linear decision rule.
Third, we adopt a Bayesian formulation in which the true mean gap $\delta$ follows a Gaussian prior $\mathcal{N}(0,\tau^2)$ and the estimate $\hat{\delta}$ is a noisy Gaussian signal; this conjugate structure yields the posterior mean $\mathbb{E}[\delta\mid\hat{\delta}]$ and the corresponding linear shrinkage factor in closed form.

\paragraph{Oracle vs. Predict-then-Optimize.}
If the true gap $\delta$ were known, the optimal allocation $w^\star$ (derived by setting the derivative to zero and clipping to the feasible set) would be:
\[
    w^\star(\delta) \;=\; \Pi_{[0,1]}\left( \frac{1}{2} + \frac{\delta}{4\lambda} \right).
\]
In the standard \textbf{Predict-then-Optimize (PtO)} approach, we observe samples, compute the empirical mean gap $\hat{\delta} = \hat{\mu}_1 - \hat{\mu}_2$, and simply plug it into the oracle formula:
\[
    w_{\text{pto}}(\hat{\delta}) \;=\; \Pi_{[0,1]}\left( \frac{1}{2} + \frac{\hat{\delta}}{4\lambda} \right).
\]

\paragraph{End-to-End Learning. (E2E)}
In contrast, an end-to-end learner seeks the action that maximizes expected utility \emph{given} the noisy signal $\hat{\delta}$. This corresponds to maximizing the conditional expectation:
\[
    \mathbb{E}[U(w;\delta)\mid\hat{\delta}]
    \;=\;
    \text{const} + u\,\mathbb{E}[\delta\mid\hat{\delta}] - 2\lambda u^2.
\]
The solution to this problem uses the \emph{posterior mean} of the gap, not the raw estimate:
\[
    w_{\text{e2e}}(\hat{\delta}) \;=\; \Pi_{[0,1]}\left( \frac{1}{2} + \frac{\mathbb{E}[\delta|\hat{\delta}]}{4\lambda} \right).
\]

\paragraph{The Gaussian Case: Explicit Shrinkage.}
To see why the end-to-end policy outperforms PtO, assume a Gaussian prior on the true gap $\delta \sim \mathcal{N}(0, \tau^2)$ and a sampling distribution $\hat{\delta} \sim \mathcal{N}(\delta, s^2)$ (where $s^2 \propto 1/n$). In particular, under the i.i.d.\ Gaussian return model with $n$ samples per asset, we have
\[
    s^2 \;=\; \frac{2\sigma^2}{n}.
\]
The posterior expectation is a linear shrinkage of the observed signal:
\[
    \mathbb{E}[\delta|\hat{\delta}] = \rho \hat{\delta}, \quad \text{where } \rho = \frac{\tau^2}{\tau^2 + s^2} \in (0, 1).
\]
Substituting this into the general end-to-end update yields the explicit decision rule:
\[
    w_{\text{e2e}}(\hat{\delta}) \;=\; \Pi_{[0,1]}\left( \frac{1}{2} + \rho \frac{\hat{\delta}}{4\lambda} \right).
\]
Comparing the equations for $w_{\text{pto}}$ and $w_{\text{e2e}}$ reveals that PtO implicitly assumes $\rho=1$. However, because $\rho < 1$ in finite samples, the end-to-end model outperforms PtO by \textbf{shrinking} the noisy estimate $\hat{\delta}$ toward the prior (zero), thereby reducing the variance of the decision.
Moreover, since $w_{\text{e2e}}(\hat{\delta})$ maximizes the conditional expected utility for each fixed $\hat{\delta}$,
for any measurable policy $w(\hat{\delta})\in[0,1]$ we have
\[
\mathbb{E}\!\left[U\!\left(w_{\text{e2e}}(\hat{\delta});\delta\right)\middle|\hat{\delta}\right]
\;\ge\;
\mathbb{E}\!\left[U\!\left(w(\hat{\delta});\delta\right)\middle|\hat{\delta}\right].
\]
Taking expectations over $\hat{\delta}$ and using the tower property yields
\[
\mathbb{E}\!\left[U\!\left(w_{\text{e2e}}(\hat{\delta});\delta\right)\right]
\;\ge\;
\mathbb{E}\!\left[U\!\left(w(\hat{\delta});\delta\right)\right].
\]
In particular, choosing $w(\hat{\delta})=w_{\text{pto}}(\hat{\delta})$ gives
\[
\mathbb{E}\!\left[U\!\left(w_{\text{e2e}}(\hat{\delta});\delta\right)\right]
\;\ge\;
\mathbb{E}\!\left[U\!\left(w_{\text{pto}}(\hat{\delta});\delta\right)\right].
\]

\begin{figure}[t]
    \centering
    \begin{subfigure}[b]{0.315\textwidth}
        \centering
        \includegraphics[width=\textwidth]{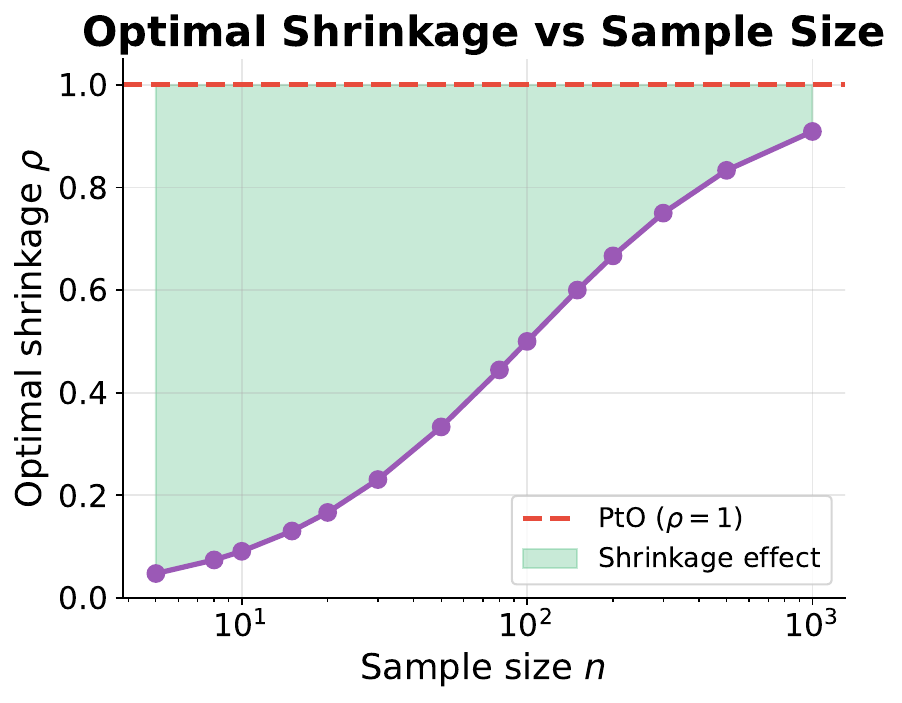}
        \caption{Shrinkage vs. Sample Size}
        \label{fig:shrinkage}
    \end{subfigure}
    \hfill
    \begin{subfigure}[b]{0.315\textwidth}
        \centering
        \includegraphics[width=\textwidth]{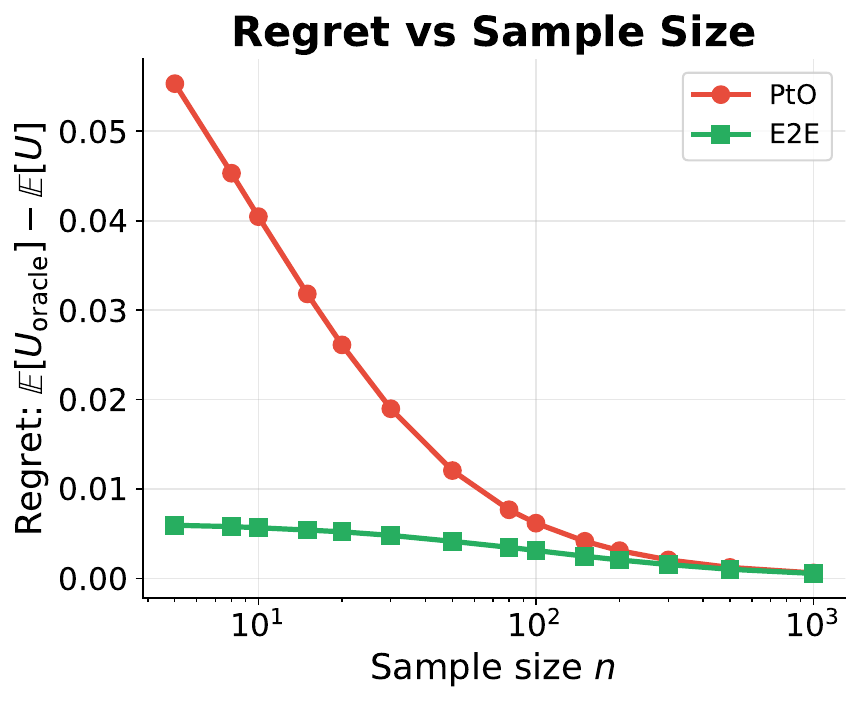}
        \caption{Regret vs. Sample Size}
        \label{fig:regret}
    \end{subfigure}
    \hfill
    \begin{subfigure}[b]{0.315\textwidth}
        \centering
        \includegraphics[width=\textwidth]{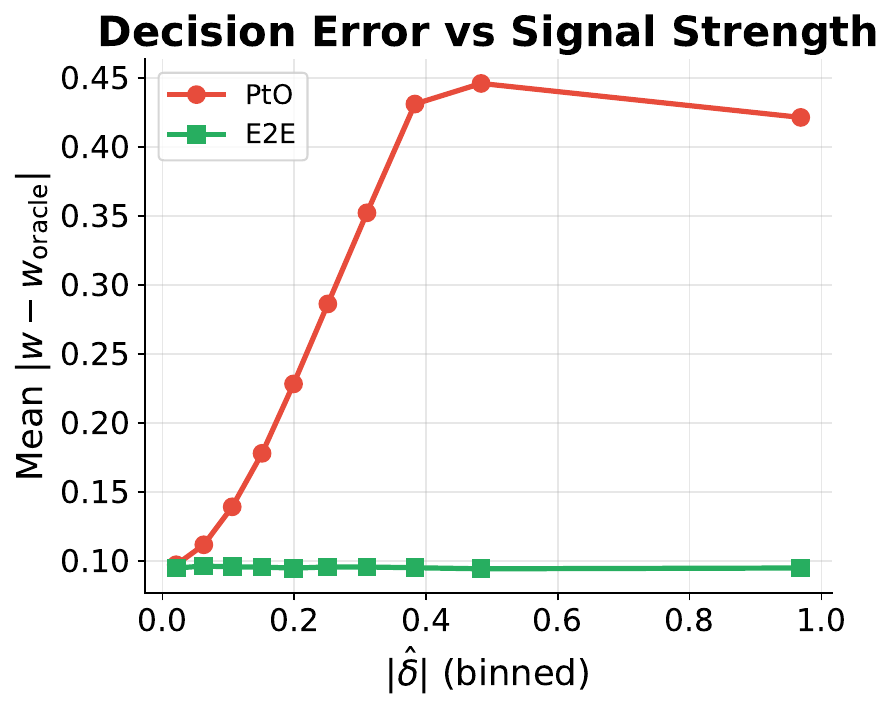}
        \caption{Error vs. Signal Strength}
        \label{fig:error}
    \end{subfigure}
    \caption{Empirical verification of predict-then-optimize (PtO) suboptimality 
    compared to end-to-end learning. 
    \textbf{(a)} The optimal shrinkage factor $\rho = \tau^2/(\tau^2 + s^2)$ increases 
    toward 1 as sample size $n$ grows; the shaded region shows the shrinkage 
    effect that E2E exploits while PtO uses $\rho=1$. 
    \textbf{(b)} Regret for both methods; 
    E2E consistently achieves lower regret than PtO across all sample sizes. 
    \textbf{(c)} Decision error $|w - w_{\mathrm{oracle}}|$ as a function of observed 
    signal strength $|\hat{\delta}|$; PtO error grows linearly with signal strength while E2E remains stable due to shrinkage. 
    Parameters: $\tau^2=0.01$, $\sigma^2=0.5$, $\lambda=0.2$.}
    \label{fig:pto_suboptimality}
\end{figure}

\paragraph{Empirical verification.}
We empirically validate the conclusions of the toy model by Monte Carlo simulation.
In each trial, we sample $\delta\sim\mathcal{N}(0,\tau^2)$ and generate a noisy estimate
$\hat{\delta}\mid\delta\sim\mathcal{N}(\delta,s^2)$ with $s^2=2\sigma^2/n$.
We then compute the PtO and end-to-end decisions
\[
w_{\text{pto}}(\hat{\delta})=\Pi_{[0,1]}\!\left(\frac12+\frac{\hat{\delta}}{4\lambda}\right),
\qquad
w_{\text{e2e}}(\hat{\delta})=\Pi_{[0,1]}\!\left(\frac12+\rho\frac{\hat{\delta}}{4\lambda}\right),
\qquad
\rho=\frac{\tau^2}{\tau^2+s^2},
\]
and evaluate the realized utility 

\[
U(w;\delta)=\mu_2+w\delta-\lambda\big(w^2+(1-w)^2\big),
\]

averaging over $10^5$ trials to estimate expected utilities. 
We also compute the oracle decision $w_{\text{oracle}}(\delta)=\Pi_{[0,1]}\!\left(\frac12+\frac{\delta}{4\lambda}\right)$, 
which has access to the true $\delta$, and define regret as the expected utility loss relative to the oracle:
$\text{Regret}(w) = \mathbb{E}[U(w_{\text{oracle}})] - \mathbb{E}[U(w)]$.

\paragraph{Results.}
Figure~\ref{fig:pto_suboptimality} reports three key findings.
Panel~(a) shows the optimal shrinkage factor $\rho$ as a function of sample size $n$; 
the shaded region represents the shrinkage effect that E2E exploits while PtO implicitly uses $\rho=1$.
Panel~(b) compares the regret of both methods: E2E consistently achieves lower regret than PtO across all sample sizes, 
with the gap largest in the low-data regime and vanishing as $n\to\infty$ (where $\rho\to 1$).
Panel~(c) analyzes decision error $|w - w_{\text{oracle}}|$ as a function of observed signal strength $|\hat{\delta}|$:
PtO error grows linearly with $|\hat{\delta}|$ because it overreacts to noisy estimates, 
while E2E error remains stable due to shrinkage.
These results confirm the theoretical prediction that PtO is statistically suboptimal under finite-sample estimation noise,
and that the suboptimality arises precisely from the failure to shrink predictions toward the prior mean.

\paragraph{Connection to Neural End-to-End Learning.}
The closed-form rule
\[
w_{\text{e2e}}(\hat{\delta})
=
\Pi_{[0,1]}\!\left(\frac{1}{2}+\rho\,\frac{\hat{\delta}}{4\lambda}\right),
\qquad
\rho=\frac{\tau^2}{\tau^2+s^2},
\]
shows that a decision-aware policy should \emph{shrink} noisy estimates before optimizing, with the shrinkage level increasing as the effective estimation noise $s^2$ grows. While the Bayes-optimal $\rho$ depends on unknown problem statistics (e.g., $\tau^2$ and $s^2$), an end-to-end trained model can \emph{implicitly learn} an analogous attenuation from data by directly optimizing a downstream decision objective (e.g., expected utility/regret or supervision from oracle allocations). This viewpoint helps explain why Time2Decide can outperform PtO in practice: PtO effectively corresponds to $\rho=1$, so it can overreact to noisy return signals, whereas Time2Decide can learn to temper its response when the signal-to-noise ratio is low. In particular, providing covariance information via our CAT tokenization supplies explicit second-order statistics about the market environment, which can inform \emph{how much} to trust the return signal and yield more calibrated decisions under uncertainty.

\paragraph{Experiments with Noisy Data.}
In the portforlio optimization case study, we additionally evaluate performance in a \emph{noisy-environment} setting. Concretely, we inject Gaussian noise into the return series (scaling by $0.01$, adding $\mathcal{N}(0,0.01)$, and clipping), and then recompute all downstream quantities using the corrupted data. This includes re-estimating any covariance inputs and recomputing the oracle allocations from the noisy returns. We then report MSE between each strategy's allocations and the corresponding \emph{noisy-data oracle}. As shown in Table~\ref{tab:time2decide_mse_summary_noisy}, Time2Decide dominates most strategies across budgets, while PtO degrades as $\gamma$ increases. This confirms that end-to-end learning effectively mitigates input noise, aligning with the theory.

\begin{table*}[htbp]
\caption{MSE performance in a noisy-environment setting with oracle recomputed from noisy returns across rebalancing budgets $\gamma$. Lower is better; bold indicates the best method for each $\gamma$.}
\centering
\small
\setlength{\tabcolsep}{6pt}
\renewcommand{\arraystretch}{1.15}
\begin{adjustbox}{max width=0.98\textwidth}
\begin{tabular}{c|ccccccc}
\noalign{\hrule height 1.1pt}
\textbf{Strategy} & \textbf{$\gamma=0.5$} & \textbf{$\gamma=0.75$} & \textbf{$\gamma=1.0$} & \textbf{$\gamma=1.25$} & \textbf{$\gamma=1.5$} & \textbf{$\gamma=1.75$} & \textbf{$\gamma=2.0$} \\
\hline
\textbf{Uniform} & 0.0336 & 0.0356 & 0.0403 & 0.0415 & 0.0449 & 0.0506 & 0.0585 \\
\textbf{Pretrained} & 0.0409 & 0.0428 & 0.0485 & 0.0479 & 0.0529 & 0.0582 & 0.0664 \\
\textbf{SFT} & 0.0324 & 0.0347 & 0.0408 & 0.0420 & 0.0434 & 0.0496 & 0.0562 \\
\textbf{Predict-then-Optimize} & \textbf{0.0224} & 0.0300 & 0.0407 & 0.0446 & 0.0522 & 0.0633 & 0.0781 \\
\textbf{Time2Decide} & 0.0232 & \textbf{0.0245} & \textbf{0.0347} & \textbf{0.0332} & \textbf{0.0344} & \textbf{0.0394} & \textbf{0.0476} \\
\noalign{\hrule height 1.1pt}
\end{tabular}
\end{adjustbox}
\label{tab:time2decide_mse_summary_noisy}
\end{table*}

\section{Negativity Case: Parity Barrier for Pure-Data Linear Attention}
\label{sec:neg-parity}
We show a simple obstruction: with only pure data tokens, linear attention cannot synthesize even-degree statistics (e.g., covariance). Intuitively, every attention layer contributes one factor of the data, so any finite stack remains an odd-degree function of the inputs unless we inject constants or explicit second-order tokens.

\begin{proposition}
Let tokens be \(z_i=x_i\in\R^{n}\) for \(i=1,\dots,N\) and let all per-token maps be linear:
\(q_j=x_j W_Q,\ k_i=x_i W_K,\ v_i=x_i W_V\).
With linear attention
\(o_j=\sum_{i=1}^{N} \langle q_j,k_i\rangle\,v_i\),
we have
\[
o_j=\Big[\sum_{i=1}^{N} (Vx_i)x_i^\top\Big] M^\top x_j,
\qquad M:=W_Q W_K^\top,\ \ V:=W_V,
\]
so \(o_j\) is linear in \(x_j\) and quadratic in \(\{x_i\}_{i=1}^{N}\), hence an \emph{odd}-degree function of the dataset. Consequently, such a layer and any finite stack cannot construct even-degree statistics like \(\sum_{i=1}^{N} x_i x_i^\top\) or any map invariant under the global flip \(x_i\mapsto -x_i\).
\end{proposition}

\begin{proof}
We compute
\[
\langle q_j,k_i\rangle
=(x_j W_Q)(x_i W_K)^\top
=x_j^\top W_Q W_K^\top x_i
=x_j^\top M x_i.
\]
Thus
\[
o_j=\sum_{i=1}^{N} (x_j^\top M x_i)\,(Vx_i)
=\Big(\sum_{i=1}^{N} (Vx_i)x_i^\top\Big)M^\top x_j
=: S_V\,M^\top x_j,
\]
where \(S_V:=\sum_{i=1}^{N} (Vx_i)x_i^\top\) is a quadratic sketch of the data. 
Under the global sign flip \(x_i\mapsto -x_i\) for all \(i=1,\dots,N\), we have \(S_V\mapsto S_V\), while \(x_j\mapsto -x_j\), so \(o_j=S_V M^\top x_j\mapsto -o_j\). Thus, the map is odd and cannot equal a flip-invariant even statistic such as \(\sum_{i=1}^{N} x_i x_i^\top\). Each additional linear-attention layer multiplies by one more linear factor in the data, raising the degree by one, while linear/ReLU maps do not introduce even-degree terms. Therefore, no finite-depth stack of such components can yield even-degree statistics without adding constant channels or explicit second-order tokens.
\end{proof}

\paragraph{Empirical validation.} To validate this theoretical limitation, we conduct a simple covariance learning experiment where models must learn to estimate covariance matrices from unseen multivariate time series. We compare two transformer variants: (i) LinearTransformer with linear attention and input projection layers, and (ii) BasicLinearTransformer with linear attention and no input projection layers. Each model processes sequences of length 50 with 4 variables and predicts the empirical covariance matrix using the Frobenius norm loss. The models are trained on 5000 sample series and tested on 1000 sample series.

Tables~\ref{tab:covariance_results_linear} and \ref{tab:covariance_results_basic} summarize the R² performance across different architectures and configurations. The results confirm our theoretical prediction: LinearTransformer achieves R² 0.9935, while BasicLinearTransformer (without input projection) performs significantly worse with R² 0.2827.

\begin{table} [ht]
  \centering
  \begin{tabular}{c|ccc|c}
    \hline
    Model & Layers & Heads & Hidden & R² \\
    \hline
    LinearTransformer & 4 & 2 & 16 & 0.9935 \\
    LinearTransformer & 4 & 4 & 16 & 0.9892 \\
    \hline
  \end{tabular}
  \caption{LinearTransformer covariance learning performance (R²) with input projection layers.}
  \label{tab:covariance_results_linear}
\end{table}

\begin{table} [ht]
  \centering
  \begin{tabular}{c|cc|c}
    \hline
    Model & Layers & Heads & R² \\
    \hline
    BasicLinearTransformer & 4 & 4 & 0.2827 \\
    BasicLinearTransformer & 8 & 4 & 0.2473 \\
    BasicLinearTransformer & 4 & 2 & 0.2234 \\
    BasicLinearTransformer & 8 & 2 & 0.1953 \\
    \hline
  \end{tabular}
  \caption{BasicLinearTransformer covariance learning performance (R²) without input projection layers.}
  \label{tab:covariance_results_basic}
\end{table}

These empirical results support our theoretical analysis: while linear attention can learn covariance patterns when equipped with proper input projections/embeddings, removing these projections creates a fundamental barrier to learning even-degree statistics.

\section{Additional Portfolio Optimization Results}
\label{app:additional_portfolio}

This section presents comprehensive MSE results for portfolio optimization across different rebalancing constraints $\gamma \in \{0.5, 0.75, 1.0, 1.25, 1.5, 1.75, 2.0\}$ and return loss weights $\lambda_r \in \{0, 10, 50, 100, \allowbreak 200, \allowbreak 300,\allowbreak 500, \allowbreak 1000, \allowbreak 3000,\allowbreak 5000,\allowbreak 10000,\allowbreak 20000,\allowbreak 40000,\allowbreak 80000,\allowbreak 160000,\allowbreak 320000\}$.
In most of our experiments, SFT achieves a better performance with a small $\lambda_r$ (focusing primarily on allocation prediction), while Time2Decide requires a larger $\lambda_r$ value to balance both return and allocation predictions effectively.

\begin{table}[H]
\centering
\caption{MSE results for $\gamma = 0.5$ across different return loss weights (Part 1).}
\label{tab:mse_gamma_05_part1}
\scriptsize
\begin{tabular}{l|ccccccccc}
\hline
Strategy &$\lambda_r$=0 &$\lambda_r$=10 &$\lambda_r$=50 &$\lambda_r$=100 &$\lambda_r$=200 &$\lambda_r$=300 &$\lambda_r$=500 &$\lambda_r$=1000 &$\lambda_r$=3000 \\
\hline
Uniform & 0.0425 & 0.0425 & 0.0425 & 0.0425 & 0.0425 & 0.0425 & 0.0425 & 0.0425 & 0.0425 \\
Pretrained & 0.0447 & 0.0447 & 0.0447 & 0.0447 & 0.0447 & 0.0447 & 0.0447 & 0.0447 & 0.0447 \\
SFT & 0.0408 & 0.0462 & 0.0559 & 0.0566 & 0.0538 & 0.0626 & 0.0610 & 0.0626 & 0.0624 \\
Predict-then-Optimize & 0.0245 & 0.0245 & 0.0245 & 0.0245 & 0.0245 & 0.0245 & 0.0245 & 0.0245 & 0.0245 \\
Time2Decide & 0.0635 & 0.0618 & 0.0521 & 0.0526 & 0.0482 & 0.0477 & 0.0439 & 0.0367 & 0.0346 \\
\hline
\end{tabular}
\end{table}

\begin{table}[H]
\centering
\caption{MSE results for $\gamma = 0.5$ across different return loss weights (Part 2).}
\label{tab:mse_gamma_05_part2}
\scriptsize
\begin{tabular}{l|ccccccc}
\hline
Strategy &$\lambda_r$=5000 &$\lambda_r$=10000 &$\lambda_r$=20000 &$\lambda_r$=40000 &$\lambda_r$=80000 &$\lambda_r$=160000 &$\lambda_r$=320000 \\
\hline
Uniform & 0.0425 & 0.0425 & 0.0425 & 0.0425 & 0.0425 & 0.0425 & 0.0425 \\
Pretrained & 0.0447 & 0.0447 & 0.0447 & 0.0447 & 0.0447 & 0.0447 & 0.0447 \\
SFT & 0.0622 & 0.0621 & 0.0627 & 0.0621 & 0.0610 & 0.0618 & 0.0614 \\
Predict-then-Optimize & 0.0245 & 0.0245 & 0.0245 & 0.0245 & 0.0245 & 0.0245 & 0.0245 \\
Time2Decide & 0.0340 & 0.0339 & 0.0326 & 0.0325 & 0.0325 & 0.0323 & 0.0323 \\
\hline
\end{tabular}
\end{table}

\begin{table}[H]
\centering
\caption{MSE results for $\gamma = 0.75$ across different return loss weights (Part 1).}
\label{tab:mse_gamma_075_part1}
\scriptsize
\begin{tabular}{l|ccccccccc}
\hline
Strategy &$\lambda_r$=0 &$\lambda_r$=10 &$\lambda_r$=50 &$\lambda_r$=100 &$\lambda_r$=200 &$\lambda_r$=300 &$\lambda_r$=500 &$\lambda_r$=1000 &$\lambda_r$=3000 \\
\hline
Uniform & 0.0444 & 0.0444 & 0.0444 & 0.0444 & 0.0444 & 0.0444 & 0.0444 & 0.0444 & 0.0444 \\
Pretrained & 0.0458 & 0.0458 & 0.0458 & 0.0458 & 0.0458 & 0.0458 & 0.0458 & 0.0458 & 0.0458 \\
SFT & 0.0411 & 0.0499 & 0.0554 & 0.0548 & 0.0527 & 0.0635 & 0.0643 & 0.0648 & 0.0639 \\
Predict-then-Optimize & 0.0294 & 0.0294 & 0.0294 & 0.0294 & 0.0294 & 0.0294 & 0.0294 & 0.0294 & 0.0294 \\
Time2Decide & 0.0646 & 0.0607 & 0.0501 & 0.0495 & 0.0510 & 0.0476 & 0.0417 & 0.0378 & 0.0359 \\
\hline
\end{tabular}
\end{table}

\begin{table}[H]
\centering
\caption{MSE results for $\gamma = 0.75$ across different return loss weights (Part 2).}
\label{tab:mse_gamma_075_part2}
\scriptsize
\begin{tabular}{l|ccccccc}
\hline
Strategy &$\lambda_r$=5000 &$\lambda_r$=10000 &$\lambda_r$=20000 &$\lambda_r$=40000 &$\lambda_r$=80000 &$\lambda_r$=160000 &$\lambda_r$=320000 \\
\hline
Uniform & 0.0444 & 0.0444 & 0.0444 & 0.0444 & 0.0444 & 0.0444 & 0.0444 \\
Pretrained & 0.0458 & 0.0458 & 0.0458 & 0.0458 & 0.0458 & 0.0458 & 0.0458 \\
SFT & 0.0637 & 0.0632 & 0.0635 & 0.0634 & 0.0633 & 0.0626 & 0.0624 \\
Predict-then-Optimize & 0.0294 & 0.0294 & 0.0294 & 0.0294 & 0.0294 & 0.0294 & 0.0294 \\
Time2Decide & 0.0352 & 0.0351 & 0.0350 & 0.0348 & 0.0348 & 0.0347 & 0.0346 \\
\hline
\end{tabular}
\end{table}

\begin{table}[H]
\centering
\caption{MSE results for $\gamma = 1.0$ across different return loss weights (Part 1).}
\label{tab:mse_gamma_10_part1}
\scriptsize
\begin{tabular}{l|ccccccccc}
\hline
Strategy &$\lambda_r$=0 &$\lambda_r$=10 &$\lambda_r$=50 &$\lambda_r$=100 &$\lambda_r$=200 &$\lambda_r$=300 &$\lambda_r$=500 &$\lambda_r$=1000 &$\lambda_r$=3000 \\
\hline
Uniform & 0.0479 & 0.0479 & 0.0479 & 0.0479 & 0.0479 & 0.0479 & 0.0479 & 0.0479 & 0.0479 \\
Pretrained & 0.0500 & 0.0500 & 0.0500 & 0.0500 & 0.0500 & 0.0500 & 0.0500 & 0.0500 & 0.0500 \\
SFT & 0.0413 & 0.0504 & 0.0542 & 0.0553 & 0.0528 & 0.0681 & 0.0688 & 0.0690 & 0.0680 \\
Predict-then-Optimize & 0.0360 & 0.0360 & 0.0360 & 0.0360 & 0.0360 & 0.0360 & 0.0360 & 0.0360 & 0.0360 \\
Time2Decide & 0.0629 & 0.0561 & 0.0504 & 0.0566 & 0.0503 & 0.0478 & 0.0436 & 0.0414 & 0.0387 \\
\hline
\end{tabular}
\end{table}

\begin{table}[H]
\centering
\caption{MSE results for $\gamma = 1.0$ across different return loss weights (Part 2).}
\label{tab:mse_gamma_10_part2}
\scriptsize
\begin{tabular}{l|ccccccc}
\hline
Strategy &$\lambda_r$=5000 &$\lambda_r$=10000 &$\lambda_r$=20000 &$\lambda_r$=40000 &$\lambda_r$=80000 &$\lambda_r$=160000 &$\lambda_r$=320000 \\
\hline
Uniform & 0.0479 & 0.0479 & 0.0479 & 0.0479 & 0.0479 & 0.0479 & 0.0479 \\
Pretrained & 0.0500 & 0.0500 & 0.0500 & 0.0500 & 0.0500 & 0.0500 & 0.0500 \\
SFT & 0.0683 & 0.0681 & 0.0669 & 0.0680 & 0.0675 & 0.0672 & 0.0673 \\
Predict-then-Optimize & 0.0360 & 0.0360 & 0.0360 & 0.0360 & 0.0360 & 0.0360 & 0.0360 \\
Time2Decide & 0.0386 & 0.0391 & 0.0387 & 0.0387 & 0.0388 & 0.0387 & 0.0387 \\
\hline
\end{tabular}
\end{table}

\begin{table}[H]
\centering
\caption{MSE results for $\gamma = 1.25$ across different return loss weights (Part 1).}
\label{tab:mse_gamma_125_part1}
\scriptsize
\begin{tabular}{l|ccccccccc}
\hline
Strategy &$\lambda_r$=0 &$\lambda_r$=10 &$\lambda_r$=50 &$\lambda_r$=100 &$\lambda_r$=200 &$\lambda_r$=300 &$\lambda_r$=500 &$\lambda_r$=1000 &$\lambda_r$=3000 \\
\hline
Uniform & 0.0485 & 0.0485 & 0.0485 & 0.0485 & 0.0485 & 0.0485 & 0.0485 & 0.0485 & 0.0485 \\
Pretrained & 0.0514 & 0.0514 & 0.0514 & 0.0514 & 0.0514 & 0.0514 & 0.0514 & 0.0514 & 0.0514 \\
SFT & 0.0446 & 0.0560 & 0.0550 & 0.0523 & 0.0523 & 0.0683 & 0.0678 & 0.0685 & 0.0672 \\
Predict-then-Optimize & 0.0376 & 0.0376 & 0.0376 & 0.0376 & 0.0376 & 0.0376 & 0.0376 & 0.0376 & 0.0376 \\
Time2Decide & 0.0639 & 0.0501 & 0.0489 & 0.0498 & 0.0486 & 0.0452 & 0.0416 & 0.0400 & 0.0390 \\
\hline
\end{tabular}
\end{table}

\begin{table}[H]
\centering
\caption{MSE results for $\gamma = 1.25$ across different return loss weights (Part 2).}
\label{tab:mse_gamma_125_part2}
\scriptsize
\begin{tabular}{l|ccccccc}
\hline
Strategy &$\lambda_r$=5000 &$\lambda_r$=10000 &$\lambda_r$=20000 &$\lambda_r$=40000 &$\lambda_r$=80000 &$\lambda_r$=160000 &$\lambda_r$=320000 \\
\hline
Uniform & 0.0485 & 0.0485 & 0.0485 & 0.0485 & 0.0485 & 0.0485 & 0.0485 \\
Pretrained & 0.0514 & 0.0514 & 0.0514 & 0.0514 & 0.0514 & 0.0514 & 0.0514 \\
SFT & 0.0667 & 0.0681 & 0.0678 & 0.0675 & 0.0679 & 0.0678 & 0.0676 \\
Predict-then-Optimize & 0.0376 & 0.0376 & 0.0376 & 0.0376 & 0.0376 & 0.0376 & 0.0376 \\
Time2Decide & 0.0391 & 0.0393 & 0.0394 & 0.0395 & 0.0395 & 0.0395 & 0.0395 \\
\hline
\end{tabular}
\end{table}

\begin{table}[H]
\centering
\caption{MSE results for $\gamma = 1.5$ across different return loss weights (Part 1).}
\label{tab:mse_gamma_15_part1}
\scriptsize
\begin{tabular}{l|ccccccccc}
\hline
Strategy &$\lambda_r$=0 &$\lambda_r$=10 &$\lambda_r$=50 &$\lambda_r$=100 &$\lambda_r$=200 &$\lambda_r$=300 &$\lambda_r$=500 &$\lambda_r$=1000 &$\lambda_r$=3000 \\
\hline
Uniform & 0.0505 & 0.0505 & 0.0505 & 0.0505 & 0.0505 & 0.0505 & 0.0505 & 0.0505 & 0.0505 \\
Pretrained & 0.0523 & 0.0523 & 0.0523 & 0.0523 & 0.0523 & 0.0523 & 0.0523 & 0.0523 & 0.0523 \\
SFT & 0.0445 & 0.0557 & 0.0529 & 0.0532 & 0.0689 & 0.0698 & 0.0693 & 0.0692 & 0.0698 \\
Predict-then-Optimize & 0.0413 & 0.0413 & 0.0413 & 0.0413 & 0.0413 & 0.0413 & 0.0413 & 0.0413 & 0.0413 \\
Time2Decide & 0.0574 & 0.0490 & 0.0472 & 0.0457 & 0.0461 & 0.0435 & 0.0423 & 0.0414 & 0.0410 \\
\hline
\end{tabular}
\end{table}

\begin{table}[H]
\centering
\caption{MSE results for $\gamma = 1.5$ across different return loss weights (Part 2).}
\label{tab:mse_gamma_15_part2}
\scriptsize
\begin{tabular}{l|ccccccc}
\hline
Strategy &$\lambda_r$=5000 &$\lambda_r$=10000 &$\lambda_r$=20000 &$\lambda_r$=40000 &$\lambda_r$=80000 &$\lambda_r$=160000 &$\lambda_r$=320000 \\
\hline
Uniform & 0.0505 & 0.0505 & 0.0505 & 0.0505 & 0.0505 & 0.0505 & 0.0505 \\
Pretrained & 0.0523 & 0.0523 & 0.0523 & 0.0523 & 0.0523 & 0.0523 & 0.0523 \\
SFT & 0.0683 & 0.0686 & 0.0690 & 0.0679 & 0.0677 & 0.0680 & 0.0681 \\
Predict-then-Optimize & 0.0413 & 0.0413 & 0.0413 & 0.0413 & 0.0413 & 0.0413 & 0.0413 \\
Time2Decide & 0.0412 & 0.0415 & 0.0418 & 0.0420 & 0.0420 & 0.0420 & 0.0419 \\
\hline
\end{tabular}
\end{table}

\begin{table}[H]
\centering
\caption{MSE results for $\gamma = 1.75$ across different return loss weights (Part 1).}
\label{tab:mse_gamma_175_part1}
\scriptsize
\begin{tabular}{l|ccccccccc}
\hline
Strategy &$\lambda_r$=0 &$\lambda_r$=10 &$\lambda_r$=50 &$\lambda_r$=100 &$\lambda_r$=200 &$\lambda_r$=300 &$\lambda_r$=500 &$\lambda_r$=1000 &$\lambda_r$=3000 \\
\hline
Uniform & 0.0539 & 0.0539 & 0.0539 & 0.0539 & 0.0539 & 0.0539 & 0.0539 & 0.0539 & 0.0539 \\
Pretrained & 0.0567 & 0.0567 & 0.0567 & 0.0567 & 0.0567 & 0.0567 & 0.0567 & 0.0567 & 0.0567 \\
SFT & 0.0583 & 0.0515 & 0.0544 & 0.0541 & 0.0732 & 0.0724 & 0.0723 & 0.0727 & 0.0721 \\
Predict-then-Optimize & 0.0472 & 0.0472 & 0.0472 & 0.0472 & 0.0472 & 0.0472 & 0.0472 & 0.0472 & 0.0472 \\
Time2Decide & 0.0626 & 0.0538 & 0.0486 & 0.0474 & 0.0472 & 0.0457 & 0.0448 & 0.0445 & 0.0442 \\
\hline
\end{tabular}
\end{table}

\begin{table}[H]
\centering
\caption{MSE results for $\gamma = 1.75$ across different return loss weights (Part 2).}
\label{tab:mse_gamma_175_part2}
\scriptsize
\begin{tabular}{l|ccccccc}
\hline
Strategy &$\lambda_r$=5000 &$\lambda_r$=10000 &$\lambda_r$=20000 &$\lambda_r$=40000 &$\lambda_r$=80000 &$\lambda_r$=160000 &$\lambda_r$=320000 \\
\hline
Uniform & 0.0539 & 0.0539 & 0.0539 & 0.0539 & 0.0539 & 0.0539 & 0.0539 \\
Pretrained & 0.0567 & 0.0567 & 0.0567 & 0.0567 & 0.0567 & 0.0567 & 0.0567 \\
SFT & 0.0711 & 0.0716 & 0.0718 & 0.0720 & 0.0721 & 0.0725 & 0.0726 \\
Predict-then-Optimize & 0.0472 & 0.0472 & 0.0472 & 0.0472 & 0.0472 & 0.0472 & 0.0472 \\
Time2Decide & 0.0444 & 0.0448 & 0.0452 & 0.0453 & 0.0454 & 0.0455 & 0.0455 \\
\hline
\end{tabular}
\end{table}

\begin{table}[H]
\centering
\caption{MSE results for $\gamma = 2.0$ across different return loss weights (Part 1).}
\label{tab:mse_gamma_20_part1}
\scriptsize
\begin{tabular}{l|ccccccccc}
\hline
Strategy &$\lambda_r$=0 &$\lambda_r$=10 &$\lambda_r$=50 &$\lambda_r$=100 &$\lambda_r$=200 &$\lambda_r$=300 &$\lambda_r$=500 &$\lambda_r$=1000 &$\lambda_r$=3000 \\
\hline
Uniform & 0.0586 & 0.0586 & 0.0586 & 0.0586 & 0.0586 & 0.0586 & 0.0586 & 0.0586 & 0.0586 \\
Pretrained & 0.0604 & 0.0604 & 0.0604 & 0.0604 & 0.0604 & 0.0604 & 0.0604 & 0.0604 & 0.0604 \\
SFT & 0.0620 & 0.0620 & 0.0567 & 0.0565 & 0.0769 & 0.0754 & 0.0771 & 0.0770 & 0.0771 \\
Predict-then-Optimize & 0.0553 & 0.0553 & 0.0553 & 0.0553 & 0.0553 & 0.0553 & 0.0553 & 0.0553 & 0.0553 \\
Time2Decide & 0.0618 & 0.0575 & 0.0524 & 0.0526 & 0.0506 & 0.0499 & 0.0490 & 0.0488 & 0.0488 \\
\hline
\end{tabular}
\end{table}

\begin{table}[H]
\centering
\caption{MSE results for $\gamma = 2.0$ across different return loss weights (Part 2).}
\label{tab:mse_gamma_20_part2}
\scriptsize
\begin{tabular}{l|ccccccc}
\hline
Strategy &$\lambda_r$=5000 &$\lambda_r$=10000 &$\lambda_r$=20000 &$\lambda_r$=40000 &$\lambda_r$=80000 &$\lambda_r$=160000 &$\lambda_r$=320000 \\
\hline
Uniform & 0.0586 & 0.0586 & 0.0586 & 0.0586 & 0.0586 & 0.0586 & 0.0586 \\
Pretrained & 0.0604 & 0.0604 & 0.0604 & 0.0604 & 0.0604 & 0.0604 & 0.0604 \\
SFT & 0.0766 & 0.0768 & 0.0769 & 0.0767 & 0.0763 & 0.0772 & 0.0771 \\
Predict-then-Optimize & 0.0553 & 0.0553 & 0.0553 & 0.0553 & 0.0553 & 0.0553 & 0.0553 \\
Time2Decide & 0.0491 & 0.0503 & 0.0502 & 0.0506 & 0.0506 & 0.0505 & 0.0507\\
\hline
\end{tabular}
\end{table}

\section{Additional Results with Noisy Data}
\label{app:noisy-portfolio}

To further support our theory in Appendix \ref{sec:app_portfolio}, we present detailed results for the noisy data experiments. Specifically, we generated the returns by scaling the raw time series by a factor of 0.01, adding Gaussian noise with a standard deviation of 0.01, and clipping the values to the range $[-0.3, 0.8]$. The summary of these results is presented in the main text (Table \ref{tab:time2decide_mse_summary_noisy}), while the detailed breakdown between different weights of return loss is provided below. We can see that with the added noise, the end-to-end \textit{Time2Decide} dominates other strategies in most cases. This aligns with our theory. Note that for Predict-then-Optimize, we use the original TimePFN for prediction. This is because we find that fine-tuning the model on the noisy data leads to worse performance, as presented in Table \ref{tab:compare-timepfn-noise-sft-or-not}.

\begin{table}[H]
\centering
\caption{MSE results for $\gamma = 0.5$ with noisy data across different return loss weights (Part 1).}
\label{tab:mse_gamma_05_noisy_part1}
\scriptsize
\begin{tabular}{l|ccccccccc}
\hline
Strategy &$\lambda_r$=0 &$\lambda_r$=10 &$\lambda_r$=50 &$\lambda_r$=100 &$\lambda_r$=200 &$\lambda_r$=300 &$\lambda_r$=500 &$\lambda_r$=1000 &$\lambda_r$=3000 \\
\hline
Uniform & 0.0336 & 0.0336 & 0.0336 & 0.0336 & 0.0336 & 0.0336 & 0.0336 & 0.0336 & 0.0336 \\
Pretrained & 0.0409 & 0.0409 & 0.0409 & 0.0409 & 0.0409 & 0.0409 & 0.0409 & 0.0409 & 0.0409 \\
SFT & 0.0358 & 0.0329 & 0.0324 & 0.0354 & 0.0373 & 0.0395 & 0.0383 & 0.0431 & 0.0433 \\
Predict-then-Optimize & 0.0224 & 0.0224 & 0.0224 & 0.0224 & 0.0224 & 0.0224 & 0.0224 & 0.0224 & 0.0224 \\
Time2Decide & 0.0423 & 0.0232 & 0.0353 & 0.0377 & 0.0421 & 0.0371 & 0.0596 & 0.0666 & 0.0594 \\
\hline
\end{tabular}
\end{table}

\begin{table}[H]
\centering
\caption{MSE results for $\gamma = 0.5$ with noisy data across different return loss weights (Part 2).}
\label{tab:mse_gamma_05_noisy_part2}
\scriptsize
\begin{tabular}{l|ccccccc}
\hline
Strategy &$\lambda_r$=5000 &$\lambda_r$=10000 &$\lambda_r$=20000 &$\lambda_r$=40000 &$\lambda_r$=80000 &$\lambda_r$=160000 &$\lambda_r$=320000 \\
\hline
Uniform & 0.0336 & 0.0336 & 0.0336 & 0.0336 & 0.0336 & 0.0336 & 0.0336 \\
Pretrained & 0.0409 & 0.0409 & 0.0409 & 0.0409 & 0.0409 & 0.0409 & 0.0409 \\
SFT & 0.0451 & 0.0465 & 0.0499 & 0.0506 & 0.0490 & 0.0471 & 0.0462 \\
Predict-then-Optimize & 0.0224 & 0.0224 & 0.0224 & 0.0224 & 0.0224 & 0.0224 & 0.0224 \\
Time2Decide & 0.0540 & 0.0507 & 0.0446 & 0.0409 & 0.0361 & 0.0332 & 0.0310 \\
\hline
\end{tabular}
\end{table}

\begin{table}[H]
\centering
\caption{MSE results for $\gamma = 0.75$ with noisy data across different return loss weights (Part 1).}
\label{tab:mse_gamma_075_noisy_part1}
\scriptsize
\begin{tabular}{l|ccccccccc}
\hline
Strategy &$\lambda_r$=0 &$\lambda_r$=10 &$\lambda_r$=50 &$\lambda_r$=100 &$\lambda_r$=200 &$\lambda_r$=300 &$\lambda_r$=500 &$\lambda_r$=1000 &$\lambda_r$=3000 \\
\hline
Uniform & 0.0356 & 0.0356 & 0.0356 & 0.0356 & 0.0356 & 0.0356 & 0.0356 & 0.0356 & 0.0356 \\
Pretrained & 0.0428 & 0.0428 & 0.0428 & 0.0428 & 0.0428 & 0.0428 & 0.0428 & 0.0428 & 0.0428 \\
SFT & 0.0349 & 0.0347 & 0.0357 & 0.0393 & 0.0467 & 0.0402 & 0.0449 & 0.0419 & 0.0400 \\
Predict-then-Optimize & 0.0300 & 0.0300 & 0.0300 & 0.0300 & 0.0300 & 0.0300 & 0.0300 & 0.0300 & 0.0300 \\
Time2Decide & 0.0302 & 0.0277 & 0.0245 & 0.0303 & 0.0343 & 0.0388 & 0.0400 & 0.0578 & 0.0549 \\
\hline
\end{tabular}
\end{table}

\begin{table}[H]
\centering
\caption{MSE results for $\gamma = 0.75$ with noisy data across different return loss weights (Part 2).}
\label{tab:mse_gamma_075_noisy_part2}
\scriptsize
\begin{tabular}{l|ccccccc}
\hline
Strategy &$\lambda_r$=5000 &$\lambda_r$=10000 &$\lambda_r$=20000 &$\lambda_r$=40000 &$\lambda_r$=80000 &$\lambda_r$=160000 &$\lambda_r$=320000 \\
\hline
Uniform & 0.0356 & 0.0356 & 0.0356 & 0.0356 & 0.0356 & 0.0356 & 0.0356 \\
Pretrained & 0.0428 & 0.0428 & 0.0428 & 0.0428 & 0.0428 & 0.0428 & 0.0428 \\
SFT & 0.0410 & 0.0429 & 0.0451 & 0.0467 & 0.0461 & 0.0465 & 0.0461 \\
Predict-then-Optimize & 0.0300 & 0.0300 & 0.0300 & 0.0300 & 0.0300 & 0.0300 & 0.0300 \\
Time2Decide & 0.0516 & 0.0461 & 0.0421 & 0.0385 & 0.0347 & 0.0330 & 0.0330 \\
\hline
\end{tabular}
\end{table}

\begin{table}[H]
\centering
\caption{MSE results for $\gamma = 1.0$ with noisy data across different return loss weights (Part 1).}
\label{tab:mse_gamma_10_noisy_part1}
\scriptsize
\begin{tabular}{l|ccccccccc}
\hline
Strategy &$\lambda_r$=0 &$\lambda_r$=10 &$\lambda_r$=50 &$\lambda_r$=100 &$\lambda_r$=200 &$\lambda_r$=300 &$\lambda_r$=500 &$\lambda_r$=1000 &$\lambda_r$=3000 \\
\hline
Uniform & 0.0403 & 0.0403 & 0.0403 & 0.0403 & 0.0403 & 0.0403 & 0.0403 & 0.0403 & 0.0403 \\
Pretrained & 0.0485 & 0.0485 & 0.0485 & 0.0485 & 0.0485 & 0.0485 & 0.0485 & 0.0485 & 0.0485 \\
SFT & 0.0422 & 0.0408 & 0.0426 & 0.0458 & 0.0440 & 0.0500 & 0.0485 & 0.0432 & 0.0427 \\
Predict-then-Optimize & 0.0407 & 0.0407 & 0.0407 & 0.0407 & 0.0407 & 0.0407 & 0.0407 & 0.0407 & 0.0407 \\
Time2Decide & 0.0413 & 0.0452 & 0.0390 & 0.0347 & 0.0371 & 0.0422 & 0.0392 & 0.0460 & 0.0487 \\
\hline
\end{tabular}
\end{table}

\begin{table}[H]
\centering
\caption{MSE results for $\gamma = 1.0$ with noisy data across different return loss weights (Part 2).}
\label{tab:mse_gamma_10_noisy_part2}
\scriptsize
\begin{tabular}{l|ccccccc}
\hline
Strategy &$\lambda_r$=5000 &$\lambda_r$=10000 &$\lambda_r$=20000 &$\lambda_r$=40000 &$\lambda_r$=80000 &$\lambda_r$=160000 &$\lambda_r$=320000 \\
\hline
Uniform & 0.0403 & 0.0403 & 0.0403 & 0.0403 & 0.0403 & 0.0403 & 0.0403 \\
Pretrained & 0.0485 & 0.0485 & 0.0485 & 0.0485 & 0.0485 & 0.0485 & 0.0485 \\
SFT & 0.0419 & 0.0431 & 0.0433 & 0.0455 & 0.0459 & 0.0491 & 0.0492 \\
Predict-then-Optimize & 0.0407 & 0.0407 & 0.0407 & 0.0407 & 0.0407 & 0.0407 & 0.0407 \\
Time2Decide & 0.0485 & 0.0474 & 0.0450 & 0.0420 & 0.0400 & 0.0387 & 0.0373 \\
\hline
\end{tabular}
\end{table}

\begin{table}[H]
\centering
\caption{MSE results for $\gamma = 1.25$ with noisy data across different return loss weights (Part 1).}
\label{tab:mse_gamma_125_noisy_part1}
\scriptsize
\begin{tabular}{l|ccccccccc}
\hline
Strategy &$\lambda_r$=0 &$\lambda_r$=10 &$\lambda_r$=50 &$\lambda_r$=100 &$\lambda_r$=200 &$\lambda_r$=300 &$\lambda_r$=500 &$\lambda_r$=1000 &$\lambda_r$=3000 \\
\hline
Uniform & 0.0415 & 0.0415 & 0.0415 & 0.0415 & 0.0415 & 0.0415 & 0.0415 & 0.0415 & 0.0415 \\
Pretrained & 0.0479 & 0.0479 & 0.0479 & 0.0479 & 0.0479 & 0.0479 & 0.0479 & 0.0479 & 0.0479 \\
SFT & 0.0448 & 0.0432 & 0.0441 & 0.0460 & 0.0451 & 0.0468 & 0.0454 & 0.0428 & 0.0422 \\
Predict-then-Optimize & 0.0446 & 0.0446 & 0.0446 & 0.0446 & 0.0446 & 0.0446 & 0.0446 & 0.0446 & 0.0446 \\
Time2Decide & 0.0420 & 0.0467 & 0.0412 & 0.0377 & 0.0394 & 0.0372 & 0.0394 & 0.0332 & 0.0410 \\
\hline
\end{tabular}
\end{table}

\begin{table}[H]
\centering
\caption{MSE results for $\gamma = 1.25$ with noisy data across different return loss weights (Part 2).}
\label{tab:mse_gamma_125_noisy_part2}
\scriptsize
\begin{tabular}{l|ccccccc}
\hline
Strategy &$\lambda_r$=5000 &$\lambda_r$=10000 &$\lambda_r$=20000 &$\lambda_r$=40000 &$\lambda_r$=80000 &$\lambda_r$=160000 &$\lambda_r$=320000 \\
\hline
Uniform & 0.0415 & 0.0415 & 0.0415 & 0.0415 & 0.0415 & 0.0415 & 0.0415 \\
Pretrained & 0.0479 & 0.0479 & 0.0479 & 0.0479 & 0.0479 & 0.0479 & 0.0479 \\
SFT & 0.0426 & 0.0420 & 0.0437 & 0.0447 & 0.0460 & 0.0476 & 0.0489 \\
Predict-then-Optimize & 0.0446 & 0.0446 & 0.0446 & 0.0446 & 0.0446 & 0.0446 & 0.0446 \\
Time2Decide & 0.0408 & 0.0423 & 0.0410 & 0.0412 & 0.0398 & 0.0391 & 0.0381 \\
\hline
\end{tabular}
\end{table}

\begin{table}[H]
\centering
\caption{MSE results for $\gamma = 1.5$ with noisy data across different return loss weights (Part 1).}
\label{tab:mse_gamma_15_noisy_part1}
\scriptsize
\begin{tabular}{l|ccccccccc}
\hline
Strategy &$\lambda_r$=0 &$\lambda_r$=10 &$\lambda_r$=50 &$\lambda_r$=100 &$\lambda_r$=200 &$\lambda_r$=300 &$\lambda_r$=500 &$\lambda_r$=1000 &$\lambda_r$=3000 \\
\hline
Uniform & 0.0449 & 0.0449 & 0.0449 & 0.0449 & 0.0449 & 0.0449 & 0.0449 & 0.0449 & 0.0449 \\
Pretrained & 0.0529 & 0.0529 & 0.0529 & 0.0529 & 0.0529 & 0.0529 & 0.0529 & 0.0529 & 0.0529 \\
SFT & 0.0454 & 0.0449 & 0.0474 & 0.0480 & 0.0453 & 0.0489 & 0.0471 & 0.0449 & 0.0445 \\
Predict-then-Optimize & 0.0522 & 0.0522 & 0.0522 & 0.0522 & 0.0522 & 0.0522 & 0.0522 & 0.0522 & 0.0522 \\
Time2Decide & 0.0550 & 0.0510 & 0.0421 & 0.0408 & 0.0414 & 0.0382 & 0.0344 & 0.0353 & 0.0384 \\
\hline
\end{tabular}
\end{table}

\begin{table}[H]
\centering
\caption{MSE results for $\gamma = 1.5$ with noisy data across different return loss weights (Part 2).}
\label{tab:mse_gamma_15_noisy_part2}
\scriptsize
\begin{tabular}{l|ccccccc}
\hline
Strategy &$\lambda_r$=5000 &$\lambda_r$=10000 &$\lambda_r$=20000 &$\lambda_r$=40000 &$\lambda_r$=80000 &$\lambda_r$=160000 &$\lambda_r$=320000 \\
\hline
Uniform & 0.0449 & 0.0449 & 0.0449 & 0.0449 & 0.0449 & 0.0449 & 0.0449 \\
Pretrained & 0.0529 & 0.0529 & 0.0529 & 0.0529 & 0.0529 & 0.0529 & 0.0529 \\
SFT & 0.0441 & 0.0434 & 0.0446 & 0.0454 & 0.0482 & 0.0523 & 0.0503 \\
Predict-then-Optimize & 0.0522 & 0.0522 & 0.0522 & 0.0522 & 0.0522 & 0.0522 & 0.0522 \\
Time2Decide & 0.0391 & 0.0400 & 0.0408 & 0.0404 & 0.0409 & 0.0418 & 0.0411 \\
\hline
\end{tabular}
\end{table}

\begin{table}[H]
\centering
\caption{MSE results for $\gamma = 1.75$ with noisy data across different return loss weights (Part 1).}
\label{tab:mse_gamma_175_noisy_part1}
\scriptsize
\begin{tabular}{l|ccccccccc}
\hline
Strategy &$\lambda_r$=0 &$\lambda_r$=10 &$\lambda_r$=50 &$\lambda_r$=100 &$\lambda_r$=200 &$\lambda_r$=300 &$\lambda_r$=500 &$\lambda_r$=1000 &$\lambda_r$=3000 \\
\hline
Uniform & 0.0506 & 0.0506 & 0.0506 & 0.0506 & 0.0506 & 0.0506 & 0.0506 & 0.0506 & 0.0506 \\
Pretrained & 0.0582 & 0.0582 & 0.0582 & 0.0582 & 0.0582 & 0.0582 & 0.0582 & 0.0582 & 0.0582 \\
SFT & 0.0549 & 0.0546 & 0.0518 & 0.0515 & 0.0545 & 0.0541 & 0.0529 & 0.0506 & 0.0503 \\
Predict-then-Optimize & 0.0633 & 0.0633 & 0.0633 & 0.0633 & 0.0633 & 0.0633 & 0.0633 & 0.0633 & 0.0633 \\
Time2Decide & 0.0408 & 0.0436 & 0.0444 & 0.0431 & 0.0394 & 0.0404 & 0.0402 & 0.0396 & 0.0400 \\
\hline
\end{tabular}
\end{table}

\begin{table}[H]
\centering
\caption{MSE results for $\gamma = 1.75$ with noisy data across different return loss weights (Part 2).}
\label{tab:mse_gamma_175_noisy_part2}
\scriptsize
\begin{tabular}{l|ccccccc}
\hline
Strategy &$\lambda_r$=5000 &$\lambda_r$=10000 &$\lambda_r$=20000 &$\lambda_r$=40000 &$\lambda_r$=80000 &$\lambda_r$=160000 &$\lambda_r$=320000 \\
\hline
Uniform & 0.0506 & 0.0506 & 0.0506 & 0.0506 & 0.0506 & 0.0506 & 0.0506 \\
Pretrained & 0.0582 & 0.0582 & 0.0582 & 0.0582 & 0.0582 & 0.0582 & 0.0582 \\
SFT & 0.0496 & 0.0574 & 0.0575 & 0.0574 & 0.0564 & 0.0576 & 0.0580 \\
Predict-then-Optimize & 0.0633 & 0.0633 & 0.0633 & 0.0633 & 0.0633 & 0.0633 & 0.0633 \\
Time2Decide & 0.0414 & 0.0430 & 0.0437 & 0.0437 & 0.0459 & 0.0468 & 0.0473 \\
\hline
\end{tabular}
\end{table}

\begin{table}[H]
\centering
\caption{MSE results for $\gamma = 2.0$ with noisy data across different return loss weights (Part 1).}
\label{tab:mse_gamma_20_noisy_part1}
\scriptsize
\begin{tabular}{l|ccccccccc}
\hline
Strategy &$\lambda_r$=0 &$\lambda_r$=10 &$\lambda_r$=50 &$\lambda_r$=100 &$\lambda_r$=200 &$\lambda_r$=300 &$\lambda_r$=500 &$\lambda_r$=1000 &$\lambda_r$=3000 \\
\hline
Uniform & 0.0585 & 0.0585 & 0.0585 & 0.0585 & 0.0585 & 0.0585 & 0.0585 & 0.0585 & 0.0585 \\
Pretrained & 0.0664 & 0.0664 & 0.0664 & 0.0664 & 0.0664 & 0.0664 & 0.0664 & 0.0664 & 0.0664 \\
SFT & 0.0614 & 0.0614 & 0.0578 & 0.0591 & 0.0624 & 0.0620 & 0.0606 & 0.0603 & 0.0562 \\
Predict-then-Optimize & 0.0781 & 0.0781 & 0.0781 & 0.0781 & 0.0781 & 0.0781 & 0.0781 & 0.0781 & 0.0781 \\
Time2Decide & 0.0484 & 0.0494 & 0.0490 & 0.0506 & 0.0478 & 0.0485 & 0.0487 & 0.0481 & 0.0476 \\
\hline
\end{tabular}
\end{table}

\begin{table}[H]
\centering
\caption{MSE results for $\gamma = 2.0$ with noisy data across different return loss weights (Part 2).}
\label{tab:mse_gamma_20_noisy_part2}
\scriptsize
\begin{tabular}{l|ccccccc}
\hline
Strategy &$\lambda_r$=5000 &$\lambda_r$=10000 &$\lambda_r$=20000 &$\lambda_r$=40000 &$\lambda_r$=80000 &$\lambda_r$=160000 &$\lambda_r$=320000 \\
\hline
Uniform & 0.0585 & 0.0585 & 0.0585 & 0.0585 & 0.0585 & 0.0585 & 0.0585 \\
Pretrained & 0.0664 & 0.0664 & 0.0664 & 0.0664 & 0.0664 & 0.0664 & 0.0664 \\
SFT & 0.0661 & 0.0653 & 0.0643 & 0.0647 & 0.0652 & 0.0652 & 0.0653 \\
Predict-then-Optimize & 0.0781 & 0.0781 & 0.0781 & 0.0781 & 0.0781 & 0.0781 & 0.0781 \\
Time2Decide & 0.0481 & 0.0487 & 0.0503 & 0.0516 & 0.0539 & 0.0557 & 0.0557 \\
\hline
\end{tabular}
\end{table}

\begin{table*}[htbp]
  \caption{Predict-then-Optize MSE Results Using Different TimePFN Checkpoints}
  \label{tab:compare-timepfn-noise-sft-or-not}
  \centering
  \small
  \setlength{\tabcolsep}{6pt}
  \renewcommand{\arraystretch}{1.15}
  \begin{adjustbox}{max width=0.98\textwidth}
  \begin{tabular}{c|ccccccc}
    \noalign{\hrule height 1.1pt}
    \textbf{Checkpoint} & \textbf{$\gamma=0.5$} & \textbf{$\gamma=0.75$} & \textbf{$\gamma=1.0$} & \textbf{$\gamma=1.25$} & \textbf{$\gamma=1.5$} & \textbf{$\gamma=1.75$} & \textbf{$\gamma=2.0$} \\
    \hline
    \textbf{Original Checkpoint} & 0.0224 & 0.0300 & 0.0407 & 0.0446 & 0.0522 & 0.0633 & 0.0781 \\
    \textbf{Fine-tuned on Noisy Data} & 0.0219 & 0.0301 & 0.0412 & 0.0452 & 0.0528 & 0.0638 & 0.0783 \\
    \noalign{\hrule height 1.1pt}
  \end{tabular}
  \end{adjustbox}
\end{table*}

\end{document}